%% file: main.tex
\documentclass[sigplan,screen]{acmart}

\setcopyright{none}
\renewcommand\footnotetextcopyrightpermission[1]{}

\usepackage[noend]{algpseudocode}
\usepackage{algorithm}
\usepackage{xspace}
\usepackage{tcolorbox}
\usepackage{subcaption}
\usepackage{stmaryrd} 
\usepackage{tabularx}
\usepackage{listings}
\usepackage{enumitem}
\usepackage{cleveref}
\AddToHook{cmd/appendix/before}{\crefalias{section}{appendix}}
\usepackage{makecell}
\usepackage{todonotes}
\usepackage{balance}

\input{macros}

\begin{document}

\title{Vectorizer: Vectorizing NumPy Programs with Shape-Guided Rewrite}

\author{Jingqian Liu}
\affiliation{%
  \institution{Simon Fraser University}
  \city{Burnaby}
  \state{British Columbia}
  \country{Canada}}
\email{jingqian\_liu@sfu.ca}

\author{Xiaoyu Liu}
\affiliation{%
  \institution{Simon Fraser University}
  \city{Burnaby}
  \state{British Columbia}
  \country{Canada}}
\email{xla411@sfu.ca}

\author{Yuepeng Wang}
\affiliation{%
  \institution{Simon Fraser University}
  \city{Burnaby}
  \state{British Columbia}
  \country{Canada}}
\email{yuepeng@sfu.ca}

\input{sections/abstract}

\maketitle

% Hide acmart's placeholder conference header while preserving review line numbers.
\makeatletter
\fancyhead[LE]{\ACM@linecountL}
\fancyhead[RO]{\ACM@linecountR}
\makeatother

\input{sections/intro}
\input{sections/overview}
\input{sections/prelim}
\input{sections/corelang}
\input{sections/rewrite}
\input{sections/implementation}
\input{sections/eval}
\input{sections/relatedwork}
\input{sections/conclusion}
\balance
\bibliographystyle{ACM-Reference-Format}
\bibliography{main}
\clearpage
\appendix
\input{sections/semantics}

\input{sections/type}

\input{sections/type-proof}
\input{sections/rewrite-appendix}
\input{sections/rewrite-proof}
\end{document}

%% file: macros.tex
\newcommand{\oplit}{op}
\newcommand{\flit}{f}
\newcommand{\glit}{g}
\newcommand{\iflit}{if}
\newcommand{\thenlit}{then}
\newcommand{\elselit}{else}
\newcommand{\forlit}{for}
\newcommand{\inlit}{in}
\newcommand{\dolit}{do}
\newcommand{\getslit}{\gets}
\newcommand{\deflit}{:=}
\newcommand{\returnlit}{return}

\newcommand{\matmullit}{matmul}
\newcommand{\expandlit}{expand\_dims}
\newcommand{\replicatelit}{replicate}

\newcommand{\oneslit}{ones}
\newcommand{\arangelit}{arange}
\newcommand{\filledlit}{filled}
\newcommand{\getmaskarraylit}{getmaskarray}
\newcommand{\makemaskedarraylit}{make\_masked}
\newcommand{\logicalnotlit}{logical\_not}
\newcommand{\logicalorlit}{logical\_or}
\newcommand{\logicalandlit}{logical\_and}

\definecolor{pink}{RGB}{219, 48, 122}
\definecolor{darkblue}{RGB}{0, 0, 180}
\definecolor{purple}{RGB}{128, 0, 128}
\definecolor{darkgreen}{RGB}{0, 100, 0}
\definecolor{brown}{RGB}{92, 64, 51}

\newcommand{\tool}{\textsc{Vectorizer}\xspace}
\newcommand{\tenspiler}{\textsc{Tenspiler}\xspace}
\newcommand{\metalift}{\textsc{Metalift}\xspace}
\newcommand{\tensorize}{\textsc{Tensorize}\xspace}
\newcommand{\ctotaco}{\textsc{C2TACO}\xspace}
\newcommand{\dexter}{\textsc{Dexter}\xspace}

\newcommand{\numpy}{NumPy\xspace}
\newcommand{\stackoverflow}{Stack Overflow\xspace}
\newcommand{\numba}{Numba\xspace}

\newcommand{\blas}{blas\xspace}
\newcommand{\blend}{blend\xspace}
\newcommand{\darknet}{darknet\xspace}
\newcommand{\dsp}{dsp\xspace}
\newcommand{\dspstone}{dspstone\xspace}
\newcommand{\llama}{llama\xspace}
\newcommand{\makespeare}{makespeare\xspace}
\newcommand{\mathfu}{mathfu\xspace}
\newcommand{\polybench}{polybench\xspace}
\newcommand{\simplarray}{simpl\_array\xspace}
\newcommand{\utdsp}{utdsp\xspace}

\newcommand{\irule}[2]%
   {\mkern-2mu\displaystyle\frac{#1}{\vphantom{,}#2}\mkern-2mu}
\newcommand{\irulelabel}[3]
{
\mkern-2mu
\begin{array}{ll}
\displaystyle\frac{#1}{\vphantom{,}#2} & \!\!\!\! #3
\end{array}
\mkern-2mu
}

\newcommand{\bfpara}[1]{\vspace{5pt} \noindent \textbf{\textit{#1}}}

\newcommand{\cprog}{\mathcal{P}}
\newcommand{\cann}{\mathcal{A}}
\newcommand{\cloop}{\mathcal{L}}

\newcommand{\lp}[0]{\text{\textcolor{pink}{\,(\,}}}
\newcommand{\rp}[0]{\text{\textcolor{pink}{\,)}}}

\newcommand{\mybar}[0]{\textcolor{pink}{\,\vert\,}}

\newcommand{\prog}{P}
\newcommand{\stmt}{S}
\newcommand{\expr}{E}
\newcommand{\name}{x}
\newcommand{\integral}{I}
\newcommand{\update}{U}
\newcommand{\mtarget}{\hat{T}}

\NewDocumentCommand{\codeliteral}{m}
  { \textcolor{darkblue}{\mathtt{#1}} }
\newcommand{\idxbrackl}{[}
\newcommand{\idxbrackr}{]}
\newcommand{\auxicodeliteral}[1]{\textcolor{purple}{\mathtt{#1}}}
\newcommand{\codeparen}[1]{\codeliteral{(}#1\codeliteral{)}}
\newcommand{\codetupleone}[1]{\codeliteral{(}#1\codeliteral{,)}}
\newcommand{\codebrack}[1]{\codeliteral{\idxbrackl}#1\codeliteral{\idxbrackr}}
\newcommand{\denot}[1]{\llbracket #1 \rrbracket}

\ExplSyntaxOn
\NewDocumentCommand{\codetuple}{ m }
 {
  \codeliteral{(}
  \clist_use:nn { #1 } { \codeliteral{,}\; ~ }
  \codeliteral{)}
 }
\NewDocumentCommand{\op}{ m }
 {
  \oplit \codeliteral{(}
  \clist_use:nn { #1 } { \codeliteral{,}\; ~ }
  \codeliteral{)}
 }
\NewDocumentCommand{\f}{ m }
 {
  \flit \codeliteral{(}
  \clist_use:nn { #1 } { \codeliteral{,}\; ~ }
  \codeliteral{)}
 }
 \NewDocumentCommand{\g}{ m }
 {
  \glit \codeliteral{(}
  \clist_use:nn { #1 } { \codeliteral{,}\; ~ }
  \codeliteral{)}
 }
\NewDocumentCommand{\anglebrack}{ m }
 {
  \codeliteral{\langle}
  \clist_use:nn { #1 } { \codeliteral{,}\; ~ }
  \codeliteral{\rangle}
 }
\NewDocumentCommand{\squarebrack}{ m }
 {
  \codeliteral{\idxbrackl}
  \clist_use:nn { #1 } { \codeliteral{,}\; ~ }
  \codeliteral{\idxbrackr}
 }
\NewDocumentCommand{\operators}{m}
  {
    \bool_set_true:N \l_tmpa_bool
    \clist_map_inline:nn { #1 }
      {
        \bool_if:NF \l_tmpa_bool { ,\, } % normal color separator
        \bool_set_false:N \l_tmpa_bool
        \codeliteral{##1}                % blue item
      }
    \ldots
  }
\NewDocumentCommand{\operatorsnotrailingdots}{m}
  {
    \bool_set_true:N \l_tmpa_bool
    \clist_map_inline:nn { #1 }
      {
        \bool_if:NF \l_tmpa_bool { ,\, } % normal color separator
        \bool_set_false:N \l_tmpa_bool
        \codeliteral{##1}                % blue item
      },
  }
\ExplSyntaxOff

\newcommand{\domain}[1]{\textbf{dom}(#1)}
\newcommand{\venv}{\sigma}
\newcommand{\venvs}{\mathsf{\Sigma}}

\newcommand{\integrali}{\mathsf{i}}
\newcommand{\integralj}{\mathsf{j}}
\newcommand{\stmts}{s}
\newcommand{\expre}{e}
\newcommand{\valuev}{v}
\newcommand{\maskmu}{\mu}
\newcommand{\cons}{::}
\newcommand{\peek}[1]{\textsf{Peek(}#1\textsf{)}}
\newcommand{\push}[2]{\textsf{Push(}#1\textsf{, }#2\textsf{)}}
\newcommand{\pop}[1]{\textsf{Pop(}#1\textsf{)}}
 
\newcommand{\shapes}{\varsigma}
\newcommand{\maskm}{m}
\newcommand{\tovenvs}{\searrow}
\newcommand{\toevalresult}{\Downarrow}
\newcommand{\tovalue}{\toevalresult_\valuev}
\newcommand{\toprogramresult}{\toevalresult}
\newcommand{\todynmaskness}{\toevalresult_\maskmu}

\newcommand{\dynarraymasked}{\otimes}
\newcommand{\dynarraynotmasked}{\odot}

\newcommand{\senv}{\Gamma}
\newcommand{\shapetenv}{\senv}
\newcommand{\shapetypeof}{:_\shapes}
\newcommand{\totypeenv}{\hookrightarrow}
\newcommand{\masktenv}{\senv}
\newcommand{\masktypeof}{:_\maskm}
\newcommand{\integraleval}[1]{\denot{#1}_{\shapetenv}}

\newcommand{\masked}{\otimes}
\newcommand{\arraymasked}{\top}
\newcommand{\arraynotmasked}{\bot}

\newcommand{\absfunc}{\alpha}
\newcommand{\shapeabs}[1]{\absfunc_\shapes(#1)}
\newcommand{\maskabs}[1]{\absfunc_\maskm(#1)}
\newcommand{\cmdReturn}[1]{\codeliteral{\returnlit}\; #1}
\newcommand{\cmdProgram}[3]{\codeliteral{\flit}\codetuple{#1}\; #2 \; \cmdReturn{#3}}
\newcommand{\cmdSkip}{\codeliteral{skip}}
\newcommand{\cmdAssign}[2]{#1 \codeliteral{\;\deflit\;} #2}
\newcommand{\cmdITE}[3]{\codeliteral{\iflit}~#1~\codeliteral{\thenlit}~#2~\codeliteral{\elselit}~#3}
\newcommand{\cmdFor}[3]{\codeliteral{\forlit}~#1~\codeliteral{\inlit}~#2~\codeliteral{\dolit}~#3}
\newcommand{\codeget}{\codeliteral{\getslit}}
\newcommand{\codeupdate}[2]{#1 \; \codeget \; #2}
\newcommand{\codeshapeaccess}[2]{\codeliteral{S(}#1 \codeliteral{)}\codebrack{#2}}
\newcommand{\arrindex}[2]{#1 \squarebrack{#2}}

\newcommand{\broadcasttoname}{\textsc{BroadcastTo}}
\newcommand{\shape}[1]{\textsc{Shape}(#1)}
\newcommand{\auxiupdate}[3]{\textsc{Update}(#1, #2, #3)}
\newcommand{\broadcastto}[2]{\broadcasttoname(#1, #2)}
\newcommand{\auxiindex}[1]{\auxicodeliteral{\idxbrackl} #1 \auxicodeliteral{\idxbrackr}}

\newcommand{\matmulname}{\codeliteral{\matmullit}}
\newcommand{\matmul}[2]{\matmulname\codeliteral{(}#1 \codeliteral{,\;} #2\codeliteral{)}}
\newcommand{\expandname}{\codeliteral{\expandlit}}
\newcommand{\expand}[2]{\expandname\codeliteral{(} #1 \codeliteral{,\;} #2 \codeliteral{)}}
\newcommand{\replicatename}{\codeliteral{\replicatelit}}
\newcommand{\replicate}[3]{\replicatename\codeliteral{(} #1 \codeliteral{,\;} #2 \codeliteral{,\;} #3 \codeliteral{)}}
\newcommand{\onesname}{\codeliteral{\oneslit}}
\newcommand{\ones}[1]{\onesname\codeliteral{(}#1\codeliteral{)}}
\newcommand{\arangename}{\codeliteral{\arangelit}}
\newcommand{\arange}[1]{\arangename\codeliteral{(}#1\codeliteral{)}}
\newcommand{\filledname}{\codeliteral{\filledlit}}
\newcommand{\filled}[2]{\filledname\codeliteral{(} #1 \codeliteral{,\;} #2 \codeliteral{)}}
\newcommand{\getmaskarrayname}{\codeliteral{\getmaskarraylit}}
\newcommand{\getmaskarray}[1]{\getmaskarrayname\codeliteral{(}#1\codeliteral{)}}
\newcommand{\codesemicolon}{\codeliteral{;\;}}

\newcommand{\makemaskedarrayname}{\codeliteral{\makemaskedarraylit}}
\newcommand{\makemaskedarray}[2]{\codeliteral{\makemaskedarrayname(}#1 \codeliteral{,\;} #2 \codeliteral{)}}

\newcommand{\logicalnot}[1]{\codeliteral{\logicalnotlit(}#1\codeliteral{)}}

\newcommand{\logicalor}[2]{\codeliteral{\logicalorlit(} #1 \codeliteral{,\;} #2 \codeliteral{)}}

\newcommand{\logicaland}[2]{\codeliteral{\logicalandlit(} #1 \codeliteral{,\;} #2 \codeliteral{)}}

\newcommand{\lhsmakemaskedarray}[2]{\makemaskedarray{#1}{#2}}
\newcommand{\reducelit}{reduce}

\newcommand{\maximumlit}{maximum}
\newcommand{\maximum}{\codeliteral{\maximumlit}}
\newcommand{\minimumlit}{minimum}
\newcommand{\minimum}{\codeliteral{\minimumlit}}
\newcommand{\sumlit}{sum}
\newcommand{\codesum}{\codeliteral{\sumlit}}
\newcommand{\prodlit}{prod}
\newcommand{\codeproduct}{\codeliteral{\prodlit}}
\newcommand{\maxlit}{max}
\newcommand{\codemax}{\codeliteral{\maxlit}}
\newcommand{\minlit}{min}
\newcommand{\codemin}{\codeliteral{\minlit}}
\newcommand{\reduceops}{\reducelit}
\newcommand{\opops}{\oplit}
\newcommand{\fops}{\flit}
\newcommand{\gops}{\glit}

\newcommand{\broadcast}{\textsc{Broadcast}}

\newcommand{\curloopvar}{\textsf{LoopVar}}
\newcommand{\curloopvarreplacer}{\textsf{LoopVarSub}}
\newcommand{\homeloopvar}[1]{\textsf{DefiningLoopVar}(#1)}

\newcommand{\rmmaskedindexer}{\textsc{UnmaskIndexer}}
\newcommand{\isrewritablerdce}{\textsf{IsRewritableRdce}}
\newcommand{\definingscope}{\textsf{DefiningBranch}}
\newcommand{\curscope}[1]{\textsf{CurBranch}(#1)}

\newcommand{\auximax}{\textsf{max}}
\newcommand{\dependsOn}{\textsc{DependsOn}}
\newcommand{\ite}[3]{\textsf{ITE}(#1, #2, #3)}
\newcommand{\onlhs}[1]{\textsf{OnLHS}(#1)}
\newcommand{\notonlhs}[1]{\lnot\onlhs{#1}}
\newcommand{\rewriteto}{\leadsto}
\newcommand{\looprewriteto}{\; \rewriteto \;}
\newcommand{\stmtrewriteto}{\; \curvearrowright \;}

\newcommand{\senvb}{\Gamma_b}
\newcommand{\senva}{\Gamma_a}
\newcommand{\loopvarrep}{\Delta}
\newcommand{\senvba}{\senvb, \senva}

\newcommand{\maskedindexer}{\arraymasked}
\newcommand{\maskedindexers}{(e_n^\maskedindexer, \ldots, e_1^\maskedindexer)}
\newlength{\exampleparindent}
\tcbuselibrary{breakable,skins}
\newtcolorbox[
  auto counter,
  number within=subsection
]{myexample}[1][]{%
  enhanced,
  breakable,
  parbox=false,
  colback=gray!5,
  colframe=black,
  boxrule=0.4pt,
  before upper={%
    \setlength{\parindent}{\exampleparindent}%
    \noindent\textbf{Example~\thetcbcounter.}\ %
  },
  #1
}

\lstdefinelanguage{dsl}{
  morekeywords={for, in, do, if, then, else},
  morekeywords=[2]{arange, replicate},
  sensitive=true,
  morecomment=[l]{\#}
}
\lstdefinestyle{dsl}
{
  language=dsl,
  basicstyle=\ttfamily\footnotesize,
  rulecolor=\color{darkblue},
  breaklines=true,
  frame=single,
  numbers=left,
  numberstyle=\tiny,
  stepnumber=1,
  showstringspaces=false,
  keywordstyle=\bfseries,
  keywordstyle=[2]\color{brown},
}

\newcommand{\parvenvs}{\venvs^*}

%% file: sections/abstract.tex
\begin{abstract}
\numpy is a widely used Python library for numerical scientific computing, known for its declarative APIs and its optimized implementations. 
However, writing efficient \numpy programs, which often entails using vectorized array operations instead of explicit Python loops, may not be straightforward. 
This can be difficult for programmers who are accustomed to imperative array traversal, especially when vectorized API invocations require careful reasoning about shapes, broadcasting, and advanced indexing.
This paper presents a rewrite-based approach for vectorizing \numpy programs with explicit loops over array data. 
Our approach vectorizes loops from the inside out, using array shapes and dataflow analysis to guide a source-to-source transformation that replaces loop bodies with vectorized statements. 
Following a set of rewrite rules that are correct by construction, our approach is consistently fast.
We have implemented the approach as a tool called \tool and evaluated it on 150 benchmarks collected from prior work and \stackoverflow. 
The evaluation shows that \tool vectorizes 142 of the 150 benchmarks directly and 2 more after minor changes to the original benchmarks, with only 0.53 seconds on average to rewrite each one. The resulting programs are, on average, 74.83$\times$ faster than the original loop-based implementations.
\end{abstract}

%% file: sections/intro.tex
\section{Introduction} \label{sec:intro}

\numpy~\cite{harris2020array} offers high-level APIs for efficient computation with multidimensional arrays while hiding optimizations and memory management in its low-level implementation. 
Such APIs are designed around concepts like axes and shapes. 
Built upon such concepts, techniques such as broadcasting enable concise, declarative array programming.
However, correctly aligning axes for broadcasting and advanced indexing can be confusing, even for experienced programmers.
As a result, programmers often write nested loops to traverse arrays, which may incur substantial runtime inefficiency.

Vectorizing loop-based \numpy code is challenging.
First, Python has a rich syntax. 
\numpy exposes many operators with flexible interfaces, which makes it difficult to lift arbitrary \numpy code.
Second, an effective lifter should incorporate techniques like broadcasting and advanced indexing to improve memory and runtime efficiency.
These techniques require properly aligning axes, which can be tricky for code with nested loops and many variables.
Third, practical \numpy programs often contain branches, reductions, and other structures whose vectorized forms do not have a straightforward correspondence to the original code.

Existing work attempts to address similar problems.
For example, \tenspiler~\cite{qiu2024tenspiler} lifts C++ programs to high-level array DSLs by summarizing the input as logical formulas and searching for an equivalent target program. 
However, it requires program-specific configurations, and extending it with new operators requires additional axioms. 
Both tasks demand non-trivial user effort.
\tensorize~\cite{brauckmann2025tensorize} summarizes programs via symbolic execution and searches for an equivalent program with pruning based on symbolic algebra solvers. 
It expects input in the Affine dialect of MLIR~\cite{mlir}, which imposes exacting requirements on the source program.

Upon reviewing \numpy questions on \stackoverflow and \numpy code on GitHub, we observe that many loops can be replaced by vectorized statements through deductive program rewrites.
The key insight is that vectorizing the innermost loop expands its loop variable to an array with an additional axis.
Variables defined in outer scopes can be treated as free variables with fixed shapes.
By carefully rewriting expressions in the loop based on shapes, we can propagate the newly added dimension to vectorize the entire loop.

For example, suppose we want to vectorize the code below.
\begin{lstlisting}[language=Python]
for i in range(A.shape[0]):
  for j in range(A.shape[1]):
    A[i, j] = i + j
\end{lstlisting}
We can first vectorize the innermost loop, treating \lstinline|i| as a free variable.
Replacing \lstinline|j| with \lstinline|np.arange(A.shape[1])|, the 1-D array of all values taken by \lstinline|j|, yields the following code.
\begin{lstlisting}[language=Python]
for i in range(A.shape[0]):
  js_ = np.arange(A.shape[1])
  A[i, js_] = i + js_
\end{lstlisting}
With broadcasting, we expand both sides of Line 3 to 1-D arrays, updating all the elements of each row of \lstinline|A| at once.

In this paper, we present a novel method based on this insight for rewriting Python loops over array data into vectorized \numpy code.
The high-level workflow of our method is shown in Figure~\ref{fig:block-diagram-no-code}.
Rather than attempting to reason about the full \numpy, we introduce a domain-specific language (DSL) that captures the common operations needed for practical loop vectorization while keeping the analysis tractable.
Within this DSL, we design a set of rewrite rules for vectorizing code with explicit loops over data. The rules are correct by construction, which allows the transformation procedure to avoid enumerative search or a separate equivalence-checking phase.
The rewrite process is guided by dataflow analysis and inferred type information about array shape and maskedness.
Our rewrite procedure works inside out, repeatedly vectorizing the innermost loop until there is no more loop in the program.
After rewriting, we apply postprocessing optimizations, such as indexing simplifications, sum-after-multiplication-to-tensordot replacements, common subexpression elimination and unused variable elimination, to produce the final runnable \numpy program.

\begin{figure}[!t]
    \centering
    \includegraphics[width=1\linewidth]{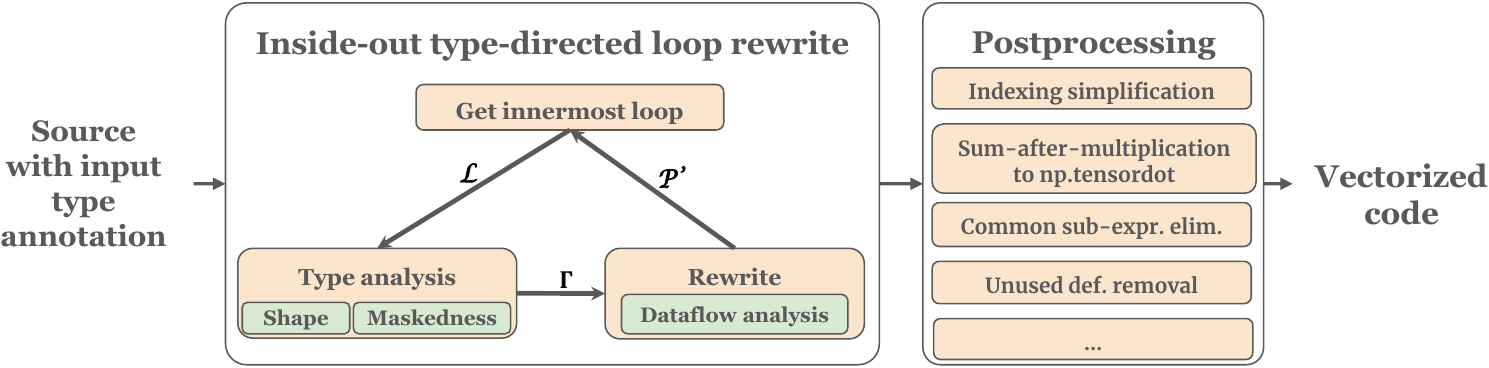}
    \caption{An overview of the workflow.}
    \label{fig:block-diagram-no-code}
    \vspace{-10pt}
\end{figure}

We have implemented our method in a tool called \tool and evaluated it on 150 benchmarks.
The results show that our technique offers significantly better applicability compared to tools from previous work.
Specifically, \tool can successfully vectorize 142 of the 150 benchmarks directly, and 2 more benchmarks after minor adaptations. On average, \tool can rewrite a \numpy function in only 0.53 seconds, and the resulting programs are, on average, 74.83$\times$ faster than the original implementations.

\bfpara{Contributions.}
We make the following main contributions:
\begin{itemize}[leftmargin=*]
\item We designed a DSL capturing the core structures and APIs of real-world \numpy programs. We formalize its semantics and type inference rules for shapes and maskedness.
\item We propose an inside-out loop vectorization algorithm guided by shapes, maskedness, and dataflow analysis.
\item We formulate the core vectorization procedure as a set of correct-by-construction rewrite rules over the DSL.
\item We implement the approach as a tool called \tool and evaluate it on 150 benchmarks. The results show that it is effective and efficient, and the obtained programs are on average $74.83\times$ faster than original implementations.
\end{itemize}

%% file: sections/overview.tex
\section{Overview} \label{sec:overview}

Let us consider the task to compute the increasing powers of each element in a given 1-D array. 
Specifically, the task requires implementing a function that takes an input vector $x$ and returns a 2-D array $X$ such that each entry $X_{i,j} = x_i^j$.
A simple loop-based implementation is shown below.
\begin{lstlisting}[language=Python]
import numpy as np
def looped_power(x):
  # x: shape (n,).
  X = np.zeros((x.shape[0], x.shape[0]))
  for i in range(x.shape[0]):
    for j in range(x.shape[0]):
      X[i, j] = np.pow(x[i], j)
  return X
\end{lstlisting}

Given that \lstinline|np.pow| supports element-wise power on arrays, to vectorize the loops, we can replicate \lstinline|x| as columns and repeat numbers from 0 to $n - 1$ as rows. 
Then we only need to call \lstinline|np.pow| once on the two arrays to get the result.
The explicit loops are delegated to the optimized \numpy implementation.
This idea is demonstrated graphically in Figure~\ref{fig:ele-pow1}.

\begin{figure}[!t]
    \begin{subfigure}{0.45\linewidth}
        \includegraphics[width=\linewidth]{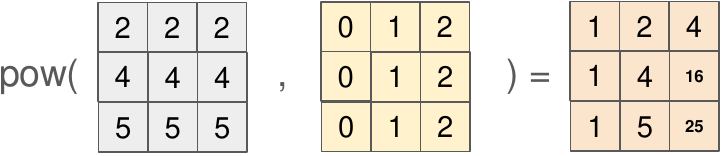}
        \caption{Element-wise power.}
        \label{fig:ele-pow1}
    \end{subfigure}
    \quad
    \begin{subfigure}{0.45\linewidth}
        \includegraphics[width=\linewidth]{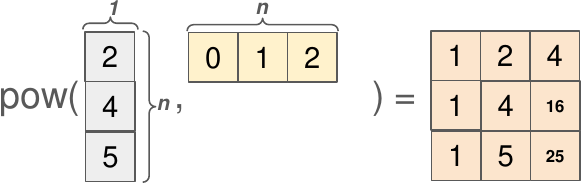}
        \caption{Element-wise power with broadcasting.}
        \label{fig:ele-pow2}
    \end{subfigure}
    \vspace{-5pt}
    \caption{Graphical illustration of the power computations.}
    \vspace{-10pt}
\end{figure}

In fact, we can further simplify by taking advantage of broadcasting, which implicitly replicates data on axes with dimensionalities of 1 to avoid boilerplate code and unnecessary memory usage.
The function \lstinline|looped_power_vectorized| follows this idea, which is shown below.
\begin{lstlisting}[language=Python]
def looped_power_vectorized(x):
  exponents = np.arange(x.shape[0])
  return np.pow(np.expand_dims(x,1), exponents)
\end{lstlisting}

Instead of replicating arrays explicitly, this function expands the \lstinline|x| array with one more axis, changing its shape from $(n,)$ to $(n, 1)$.
Here, the shape of an array is a tuple of natural numbers specifying the size of each dimension of the array.
In Line 2, \lstinline|np.arange| is called to get all the exponents in an array of shape $(n,)$.
Because of broadcasting, we can get the same result.
The idea is illustrated in Figure~\ref{fig:ele-pow2}.

From this example, we see that one of the challenges in vectorizing such functions lies in correctly aligning computations on target axes.
To automate such transformations, we can first obtain a partially vectorized version of the function.
\begin{lstlisting}[language=Python]
def looped_power_one_loop(x):
  X = np.zeros((x.shape[0], x.shape[0]))
  for i in range(x.shape[0]):
    js_ = np.arange(x.shape[0])
    X[i, js_] = np.pow(x[i], js_)
  return X
\end{lstlisting}

Here, the innermost loop is vectorized by replacing \lstinline|j| with \lstinline|js_|, the array of values it ranges over.
All other variables keep their original shapes, and broadcasting performs row-wise computations and updates.
Note that \lstinline|X[i, js_]| broadcasts \lstinline|i| to get all the indices of row \lstinline|i|, equivalent to \lstinline|X[i, 0:x.shape[0]]|. 
Next, we can eliminate the remaining loop and get the following code.
\begin{lstlisting}[language=Python]
def looped_power_vectorized_v2(x):
  X = np.zeros((x.shape[0], x.shape[0]))
  is_ = np.arange(x.shape[0])
  is_e= np.expand_dims(is_, 1)
  js_ = np.arange(x.shape[0])
  x_e = np.expand_dims(x[is_], 1)
  X[is_e, js_] = np.pow(x_e, js_)
  return X
\end{lstlisting}

As before, we want to replace all occurrences of \lstinline|i| with an array, which adds a dimension to expressions that use \lstinline|i|. 
Then \lstinline|x[i]| becomes \lstinline|x[is_]|, changing its shape from $()$ to $(n,)$.
However, a simple replacement will put the newly expanded dimension on the same axis as the one created when we vectorize the previous loop. 
To properly broadcast the operands, we shift the newly created dimension to the left by expanding a new dimension on the right-most axis, changing its shape from $(n,)$ to $(n, 1)$. 
Now, broadcasting can perform element-wise computation for operands with shapes $(n, 1)$ and $(n,)$, and the left-hand side is handled similarly for selections from \lstinline|X|.
The resulting program is fully vectorized and can be simplified to \lstinline|looped_power_vectorized| by postprocessing.

This example suggests a general insight: when vectorizing the innermost loop, variables from the outer scope can be treated as free variables with fixed shapes.
Replacing the loop variable with an array adds an extra axis to the expressions that originally contained it.
By properly arranging array axes, this new axis can propagate through the program and vectorize computations previously performed by loops.

\bfpara{Branches.} 
Consider this function having a branch, which makes the technique discussed so far not directly applicable.
\begin{lstlisting}[language=Python]
def half_burn(base):
  # base: shape (M,).
  out = np.zeros((base.shape[0],))
  for i in range(base.shape[0]):
    if i < 3:
      out[i] = base[i] / 2
  return out
\end{lstlisting}

Specifically, we cannot directly replace the variable \lstinline|i| in the branch with \lstinline|np.arange|, as the replacement would include values that are not taken in some iterations.
We must exclude such values while maintaining the dimensional structure, so the technique discussed above can still be applicable.

Towards this goal, we leverage the \emph{masked array} in \numpy.
Each element in a masked array has an associated boolean mask indicating its validity, and invalid ones are omitted from computations.
However, masked arrays have limited out-of-the-box usability and some of their behaviors are not clearly specified or documented.
To address these issues, our DSL has more specific semantics for masked arrays and extends \numpy with masked update statements.
Unlike normal updates with advanced indexing, whose left-hand side is a variable indexed by arrays, masked updates also accept boolean mask indices.
An element is updated only when it is unmasked, the corresponding mask index is true, and the corresponding right-hand-side value is unmasked.

When rewriting code in a branch, we replace variables restricted by the branch with masked arrays.
Such variables are identified by a dataflow analysis.
As \numpy's documentation does not explicitly specify the behavior of using masked arrays as indices, our DSL also disallows it.
During the rewrite, if an indexing array becomes masked, we fill it with \lstinline|0| and apply the mask of the indexing array to the indexing result.

\begin{figure}[!t]
    \begin{center}
        \includegraphics[width=0.78\linewidth]{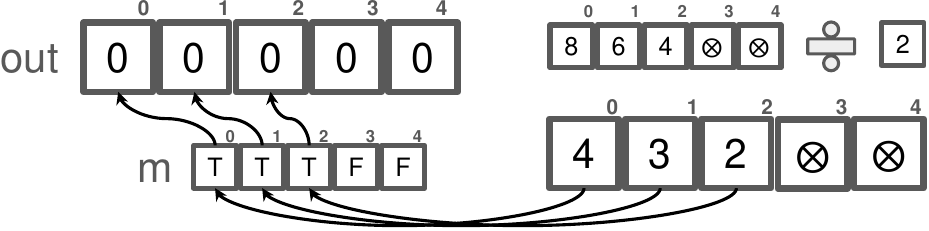}
        \caption{Graphical illustration of masked updates. $\masked$ represents masked entries.}
    \label{fig:masked-update}
    \end{center}
    \vspace{-10pt}
\end{figure}

This idea is explained graphically in Figure~\ref{fig:masked-update}, and the corresponding code is shown in the \lstinline|half_burn_vectorized| function.
Here, \lstinline|make_masked| is an operator in the DSL that takes a data array and a mask array to build a masked array.
When used on the left-hand side, it turns a statement into a masked update, with its second argument being a mask index.
In the original branch, \lstinline|i| is restricted by the branch condition, so the vectorized index \lstinline|is_| must be masked by \lstinline|m|.
Because our DSL does not allow masked index, we fill it with \lstinline|0| to unmask it and mask the indexing result instead.

\begin{lstlisting}[language=Python]
# Proof of concept; not runnable Python code.
def half_burn_vectorized(base):
  out = np.zeros((base.shape[0],))
  is_ = np.arange(base.shape[0])
  m = is_ < 3
  idx = np.ma.filled(make_masked(is_, m), 0)
  half = make_masked(base[idx], m) / 2
  make_masked(out[idx], m) = half
  return out
\end{lstlisting}

To summarize, we formalize our approach as a set of rewrite rules conditioned on the shape and maskedness of the targets.
The rules operate on our DSL, which captures many structures found in practical \numpy programs.
After rewriting a program, we apply postprocessing to simplify the resulting code and emit a runnable Python program.

%% file: sections/prelim.tex
\section{Preliminaries}

We first introduce some concepts used in the paper.
\numpy organizes data in multidimensional arrays. 
We consider two aspects from which arrays are typed: shape and maskedness. 

\subsection{Shapes and Broadcasting}

The shape of an array is a tuple of numbers recording the size of the array along each dimension.
For example, the shape of \lstinline|np.array([[1,2,3],[4,5,6]])| is $(2, 3)$.
We refer to positions in a shape as \emph{axes} and to the corresponding sizes as \emph{dimensionalities}.
Scalar values have the shape $()$.

Some \numpy operators accept operands whose shapes are broadcasting-compatible. 
% The rule for broadcasting two shapes is given in Figure~\ref{fig:broadcast}. 
Broadcasting an array to a target shape can be understood as duplicating the array along axes of size 1 and any missing leading axes. 
Thus, it allows us to reuse elements in an array to match the required shape in computations.
As shown below, binary operators implicitly broadcast operands to the same shape for computation.
\begin{lstlisting}[language=Python]
t = np.array([[1], [0]]) # shape (2, 1)
s = np.array([2, 0])     # shape    (2,)
r = s + t    # [[3, 1], [2, 0]], shape (2, 2)
\end{lstlisting}
Here, \lstinline|t| and \lstinline|s| are broadcast to \lstinline|[[1, 1], [0, 0]]| and \lstinline|[[2, 0], [2, 0]]|. The addition is element-wise on the two arrays.

\subsection{Advanced Indexing}

\numpy allows arrays to be indexed by integer arrays, which is a feature known as \emph{advanced indexing}.
We refer to the indexing arrays as \emph{indexers}, and to the indexed array as the \emph{base}, or \emph{indexee}.
When several arrays are used together as indexers, their shapes must be broadcasting-compatible.
\numpy first broadcasts these indexers to a common shape and then uses them to make element-wise selections from the indexee.
Therefore, the result has the common broadcast shape of the indexers.
In the example below, \lstinline|i2| has shape $(2,)$ and when used as an indexer, it is broadcast to the shape of \lstinline|i1|, which is $(2, 2)$.
Replicated on the implicitly added left-most axis, the effective indexing array is \lstinline|[[0, 2], [0, 2]]|.
\begin{lstlisting}[language=Python]
i1 = np.array([[1, 1], [1, 0]]) # shape: (2, 2)
i2 = np.array([0, 2]) # shape: (2,)
w = z[i1, i2] #[[z[1,0],z[1,2]],[z[1,0],z[0,2]]]
\end{lstlisting}

\subsection{Masked Arrays}

The \lstinline|ma| submodule of \numpy provides the \lstinline|MaskedArray| class.
In addition to the array data, \lstinline|MaskedArray| also has an associated boolean array of the same shape called \emph{mask}.
An element whose corresponding mask entry is true is considered \emph{masked} and is omitted from computations.
The reduce operators also skip masked elements. If every element being reduced is masked, the reduction result is masked as well.
\lstinline|np.ma.filled| can convert a masked array to a normal array by replacing masked entries with a specified value.
\begin{lstlisting}[language=Python,escapechar=\%]
o = np.ma.MaskedArray([1,2],[False,True]) #[1,%$\masked$%]
n = o + np.array([100, 101])         # [101, %$\masked$%]
m = np.ma.filled(n, 99)              # [101, 99]
l = np.mean(n, axis=0)               # 101
k = np.mean(np.ma.MaskedArray(n,[True,False]))#%$\masked$%
\end{lstlisting}

%% file: sections/corelang.tex
\section{Core Language} \label{sec:corelang}

This section presents the DSL used by our vectorization procedure, which captures core features of \numpy.

\subsection{Syntax and Semantics}

Figure~\ref{fig:dsl} shows the syntax of our DSL, which is high-level and imperative.
Its formal semantics is shown in Appendix~\ref{sec:sem}.

Here, a program is a function consisting of statements, which are built from expressions.
The variable binding statement $\cmdAssign{\name}{\expre}$ binds the expression $\expre$ to the name $\name$. 
The normal update statement $\codeupdate{\arrindex{\name}{\overline{\expre_1}}}{\expre_2}$ modifies elements of $\name$ selected by the array indices $\overline{\expre_1}$ to the value of $\expre_2$, following \numpy's semantics of updates with advanced indexing.
We require the right-hand side of such updates to be unmasked.
The masked update statement $\codeupdate{\makemaskedarray{\arrindex{\name}{\overline{\expre_1}}}{\expre_2}}{\expre_3}$ uses the mask index $\expre_2$.
The left-hand side may use nested calls to $\makemaskedarrayname$ to incorporate multiple mask indices.
As discussed in Section~\ref{sec:overview}, a selected element is updated only when it is unmasked, the corresponding element in the conjunction of the mask indices is true, and the corresponding element on the right-hand side is unmasked.

Like many imperative languages, our DSL also has the standard $\cmdSkip$ statement and sequential composition $\stmts_1 \codesemicolon \stmts_2$.
Branches like $\cmdITE{\expre}{\stmts_1}{\stmts_2}$ expect an unmasked scalar as the branch condition $\expre$.
Since we do not consider specific data types in this DSL, zeros are considered false and non-zero values are considered true.
Loops in the form of $\cmdFor{\name}{\codeliteral{0 \ldots} \integrali}{\stmts}$ must have an integer literal or a shape access expression as their bound $\integrali$, and the loop variable $\name$ ranges from zero to the loop bound minus one.
To make control flows easy to analyze, each variable name can be bound by at most one statement, and names bound inside a branch or loop may not be used outside that branch or loop.

The semantics of most expressions have a straightforward correspondence to their counterparts in \numpy.
The shape access expression $\codeshapeaccess{\expre}{c}$ returns the size of axis $c$ of $\expre$.
The array indexing operator, $\arrindex{\expre}{\expre_1, \ldots, \expre_n}$, is equivalent to integer array indexing in \numpy, but $\expre_1, \ldots, \expre_n$ cannot be masked arrays.
We also require $n$ to equal the number of axes of $\expre$.
The expression $\makemaskedarray{\expre_1}{\expre_2}$ constructs a masked array from the data array $\expre_1$ and the mask array $\expre_2$.
Unlike the constructor of \lstinline|MaskedArray| in \numpy, $\makemaskedarrayname$ invalidates entries whose masks are falsy.
When used on the right-hand side, it broadcasts operands as any other binary operator does.
When used on the left-hand side of a masked update statement, only the mask array may be broadcast to the shape of the data array.
The $\replicatename$ operator extends $\expandname$ by additionally accepting a sequence of integral expressions, which specifies the sizes of the inserted axes.
Operators are assumed to return arrays that do not overlap with their operands in memory.

\input{figures/dsl}
\subsection{Typing Shapes and Maskedness} \label{subsec:types}

\input{figures/types/type-primary}

As shown by the formal semantics in Appendix~\ref{sec:sem}, expressions evaluate to both values and runtime maskedness.
Given type annotations for the input arrays, we can statically infer the type of every expression in a well-typed program.
These annotations specify the maskedness and dimensionalities of each argument, using integers for axes with fixed length and symbolic variables for axes with dynamic length.

Due to space limit, the complete set of type inference rules for expressions is provided in Appendix~\ref{sec:type-complete}.
Here, we present a subset in Figure~\ref{fig:shape-expre-shown}. 
The rules for type analysis for statements are in Figure~\ref{fig:shape-stmt-shown}.
In our formalization, the type environment $\senv$ is a map from variables to pairs of static shapes and maskedness.
Judgments of the forms $\senv \vdash \expre \shapetypeof \shapes$ and $\senv \vdash \expre \masktypeof \maskm$ denote that $\expre$ is inferred to have static shape $\shapes$ and maskedness $\maskm$.
We use $\arraymasked$ to denote masked arrays and use $\arraynotmasked$ to denote unmasked arrays.
For example, the (S-Idx) rule infers the shape of an indexing expression by first inferring the shapes of the indexers and then broadcasting them to a common shape.
The (M-Idx) rule infers the maskedness of an indexing expression to be the maskedness of the indexee, provided that none of the indexers is masked.
In Figure~\ref{fig:shape-stmt-shown}, judgments of the form $\senv \vdash \stmts \totypeenv \senv'$ mean that the statement $\stmts$ is well-typed under the environment $\senv$, and analyzing $\stmts$ updates $\senv$ to $\senv'$.
For example, the (T-Bind) rule states that analyzing a variable binding statement updates $\senv$ to map the variable to the inferred type of the right-hand side, given that the variable has not been bound before.
The following theorem formalizes the soundness property of our type system.
\begin{theorem}\label{thm:type-soundness-shown}
Let $\cprog$ be a function with body $\stmts$, $\cann$ and $\venvs$ be the type annotations and an evaluation environment for the arguments. 
If every variable that can be evaluated under $\venvs$ is well-typed under $\cann$, $\cprog$ can be executed under $\venvs$, and $\cann \vdash \stmts \totypeenv \senv$, then the value returned is well-typed under $\senv$.
\end{theorem}
\begin{proof}
The proof is available in \Cref{sec:type-proof}.
\end{proof}

%% file: figures/dsl.tex
\begin{figure}[!t]
\centering
\small
\[
\begin{array}{l c l}
\text{Program } \prog & ::= & \cmdProgram{\overline{x}}{\stmt}{y} \\
\text{Statement }    \stmt & ::= & \cmdSkip \mybar \cmdAssign{\name}{\expr} \mybar \update \mybar \cmdITE{\expr}{\stmt}{\stmt}  \\
    &  & \mybar \cmdFor{\name}{\codeliteral{0 \ldots} \integral}{\stmt}\mybar \stmt \codesemicolon \stmt \\
\text{Integral } \integral & ::= & \codeshapeaccess{\expr}{c} \mybar c \\
\text{Expr }    \expr & ::= & \name \mybar \integral \mybar \arrindex{\expr}{\overline{\expr}} \mybar \f{\overline{\expr}} \mybar \g{\expr, c} \mybar \filled{E}{E} \\
&  & \mybar \ones{\codeparen{\overline{\integral}}} \mybar \arange{\integral} \mybar \matmul{E}{E} \\
&  & \mybar \expand{\expr}{\codeparen{\overline{c}}} \\
&  & \mybar \replicate{\expr}{\codeparen{\overline{c}}}{\codeparen{\overline{\integral}}} \\
\text{Masked LHS } \mtarget & ::= & \lhsmakemaskedarray{ \lp  \arrindex{\name}{\overline{\expr}} \mybar \mtarget \rp }{ \expr} \\
\text{Update } \update & ::= & \codeupdate{\lp \arrindex{\name}{\overline{\expr}} \mybar \mtarget \rp }{ \lp \expr \rp }\\

\end{array}
\]
\[
\name, y \in \textbf{Variables} \quad c \in \textbf{Integer literals} \quad \fops \in \opops \quad \gops \in \reduceops
\]
\[
\begin{array}{l l}
        \text{Unary $\opops$:} & \operators{abs, exp, log, negative,  \logicalnotlit} \\
        \text{Binary $\opops$:} & \operatorsnotrailingdots{+, -, *, /, pow, maximum, minimum, \logicalorlit,} \\
        & \operators{\logicalandlit, ==, <, >, \makemaskedarraylit}\\
        \text{$\reduceops$:} & \operators{\sumlit, \prodlit, \maxlit, \minlit, all, any, mean} \\
\end{array}
\]
\caption{DSL syntax. Meta symbols and code constructs are in \textcolor{pink}{pink} and $\codeliteral{blue}$. $\overline{\text{Bars}}$ denote comma-separated sequences. }
\label{fig:dsl}
\vspace{-10pt}
\end{figure}

%% file: figures/types/type-primary.tex
\begin{figure}[!t]
\[
\scriptsize
\begin{array}{c}
\irulelabel{
    \begin{array}{c}
        \textbf{Var } x \\
    \shapetenv(x) = (\shapes, m)
    \end{array}
}{
    \shapetenv \vdash x \shapetypeof \shapes
}{
    \textrm{(S-Var)}
}

\quad

\irulelabel{
    \begin{array}{c}
        0 \leq i \leq n \quad
        \shapetenv \vdash \expre_i \shapetypeof \shapes_i \\
        \broadcast(\shapes_0, \ldots, \shapes_n) = \shapes
    \end{array}
}{
    \shapetenv \vdash \arrindex{\expre}{\expre_0, \ldots, \expre_n} \shapetypeof \shapes
}{
    \textrm{(S-Idx)}
}

\\ \ \\

\irulelabel{
    \begin{array}{c}
        1 \leq i \leq n \quad \shapetenv \vdash \expre_i \shapetypeof \shapes_i \\
        \broadcast(\shapes_1, \ldots, \shapes_n) = \shapes \\
        \fops \in \text{Unary } \opops \cup \text{Binary } \opops
    \end{array}
}{
    \shapetenv \vdash \f{\expre_1, \ldots, \expre_n} \shapetypeof \shapes
}{
    \textrm{(S-Op)}
}

\irulelabel{
    \begin{array}{c}
        \textbf{Var } x \\ \masktenv(x) = (\shapes, \maskm)
    \end{array}
}{
    \masktenv \vdash x \masktypeof \maskm
}{
    \textrm{(M-Var)}
}

\\ \ \\

\irulelabel{
    \begin{array}{c}
        \fops \in \text{Unary } \opops \cup \text{Binary } \opops \\
        \fops = \makemaskedarrayname \lor (\exists k. 1 \leq k \leq n \land \masktenv \vdash \expre_k \masktypeof \arraymasked)
    \end{array}
}{
    \masktenv \vdash \f{\expre_1, \ldots, \expre_n} \masktypeof \arraymasked
}{
    \textrm{(M-Op1)}
}

\\ \ \\

\irulelabel{
    \begin{array}{c}
        \fops \in \text{Unary } \opops \cup \text{Binary } \opops \\
        \fops \neq \makemaskedarrayname \\
        \masktenv \vdash \expre_k \masktypeof \arraynotmasked \quad
        1 \leq k \leq n
    \end{array}
}{
    \masktenv \vdash \f{\expre_1, \ldots, \expre_n} \masktypeof \arraynotmasked
}{
    \textrm{(M-Op2)}
}

\!

\irulelabel{
    \begin{array}{c}
        \masktenv \vdash \expre \masktypeof \maskm \\ 
        0 \leq k \leq n \\ 
        \masktenv \vdash \expre_k \masktypeof \arraynotmasked
    \end{array}
}{
    \masktenv \vdash \arrindex{\expre}{\expre_0, \ldots, \expre_n} \masktypeof \maskm
}{
    \textrm{(M-Idx)}
}
\end{array}
\]
\vspace{-10pt}
\caption{Sample type inference rules for expressions.}
\label{fig:shape-expre-shown}
\vspace{-5pt}
\end{figure}

\begin{figure}[!t]
\[
\scriptsize
\begin{array}{c}
\irulelabel{
}{
    \shapetenv \vdash \cmdSkip \totypeenv \shapetenv 
}{
    \textrm{(T-Skp)}
}
\quad

\irulelabel{
    \begin{array}{c}
        \shapetenv \vdash \expre \shapetypeof \shapes \\
        \masktenv \vdash \expre \masktypeof m \quad
        x \notin \domain{\shapetenv}
    \end{array}
}{
    \shapetenv \vdash \cmdAssign{x}{\expre} \totypeenv \shapetenv[x \mapsto (\shapes, m)]
}{
    \textrm{(T-Bind)}
}
\\ \ \\

\irulelabel{
    \begin{array}{c}
        \textbf{Update } u
    \end{array}
}{
    \shapetenv \vdash u \totypeenv \shapetenv 
}{
    \textrm{(T-Upd)}
}
\quad

\irulelabel{
    \begin{array}{c}
        \shapetenv \vdash \expre \shapetypeof () \quad
        \masktenv \vdash \expre \masktypeof \arraynotmasked \\
        \shapetenv \vdash \stmts_1 \totypeenv \shapetenv' \quad
        \shapetenv' \vdash \stmts_2 \totypeenv \shapetenv''
    \end{array}
}{
    \shapetenv \vdash \cmdITE{\expre}{\stmts_1}{\stmts_2} \totypeenv \shapetenv''
}{
    \textrm{(T-If)}
}
\\ \ \\

\irulelabel{
    \begin{array}{c}
        \shapetenv[x \mapsto ((), \arraynotmasked)] \vdash \stmts \totypeenv \shapetenv' 
    \end{array}
}{
    \shapetenv \vdash \cmdFor{x}{\codeliteral{0 \ldots} c}{\stmts} \totypeenv \shapetenv'
}{
    \textrm{(T-For)}
}
\quad

\irulelabel{
    \begin{array}{c}
        \shapetenv \vdash \stmts_1 \totypeenv \shapetenv' \\ 
        \shapetenv' \vdash \stmts_2 \totypeenv \shapetenv''
    \end{array}
}{
    \shapetenv \vdash \stmts_1 \codesemicolon \stmts_2 \totypeenv \shapetenv'' 
}{
    \textrm{(T-Seq)}
}

\end{array}
\]
\vspace{-10pt}
\caption{Rules for type analysis for statements.}
\label{fig:shape-stmt-shown}
\vspace{-5pt}
\end{figure}

%% file: sections/rewrite.tex
\section{Vectorizing with Type-Directed Rewrite} \label{sec:rewrite}

\subsection{Problem Statement}

We first formally state the vectorization problem:
Given a program $\cprog$ and its input type annotations, our goal is to find a loop-free program $\cprog'$ that is observationally equivalent to $\cprog$.
That is, for any evaluation environment $\venvs$ under which $\cprog$ returns a value $\valuev$ with maskedness $\maskmu$, $\cprog'$ under $\venvs$ returns a value $\valuev'$ with maskedness $\maskmu'$ such that $\valuev = \valuev'$ and $\maskmu = \maskmu'$.

Some programs in the DSL cannot be vectorized, such as the ones with strongly coupled loop-carried dependence.
Consequently, our rewrite approach is not complete for the full set of programs in the DSL.
We will circle back to define \textit{rewritable} programs more precisely in \Cref{sec:vectorizability}.

\subsection{Top-Level Algorithm}

\input{figures/rewrite/algo-vectorize}

The top-level algorithm for vectorization is shown in \Cref{algo:vectorize}.
Given a program $\cprog$ with input type annotations $\cann$, \textsc{Vectorize} works inside out, repeatedly rewriting the innermost loop of $\cprog$ (Lines 2--6).
In each iteration, \textsc{AnalyzeShape} infers the variables' types (Line 3), as described in \Cref{subsec:types}.
\textsc{Vectorize} then selects the innermost loop (Line 4), invokes \textsc{Rewrite} to rewrite the loop, guided by the inferred types (Line 5), and substitutes the rewritten statements back (Line 6).
The procedure repeats until no vectorizable loops remain, and returns the resulting program.

\subsection{Type-Directed Rewrite}

\input{figures/rewrite/rewrite-stmt-primary}

\bfpara{Rewriting statements.}
The \textsc{Rewrite} procedure rewrites statements.
The complete set of rewrite rules is in \Cref{sec:rewrite-complete} and a subset of rules for statements is shown in Figure~\ref{fig:rewrite-stmt-shown}. 
We use judgments of the form $\senvba, \loopvarrep \vdash \stmts \stmtrewriteto \stmts', \senva'$ to represent that rewriting the statement $\stmts$ under environments $\senvba$ and $\loopvarrep$ produces a new statement $\stmts'$ and an updated type environment $\senva'$.
$\senvb$ is the type environment before the current invocation of \textsc{Rewrite}, and $\senva$ is the updated environment. $\loopvarrep$ maps the current loop variable to its substitute.
$\curloopvar(\loopvarrep)$ and $\curloopvarreplacer(\loopvarrep)$ return the current loop variable and the current loop variable substitute stored in $\loopvarrep$.

The key idea is to eliminate a loop by replacing its loop variable with an array containing all the values taken by that variable, and propagating the newly added axis through the loop body.
In the rest of the paper, we say that an expression is \emph{lifted} when it is rewritten to an expression that evaluates to an array of values produced by the original one across loop iterations.
We can take advantage of broadcasting to avoid unnecessary replication.
Rule (R-For) in Figure~\ref{fig:rewrite-stmt-shown} initializes the process: before rewriting the loop body, it records in $\loopvarrep$ that the loop variable should be replaced by $\arangename(\integrali)$ where $\integrali$ is the loop bound.
Rule (R-Seq) rewrites sequences compositionally, using the type environment produced by the first rewrite in the second.
Bind and update statements are rewritten mainly by rewriting their expressions.
In particular, Rule (R-Bind1) handles the case where the right-hand side is lifted, and it updates the type environment accordingly.

\bfpara{Rewriting expressions.}
Statement-level rewrites involve rewriting expressions. \Cref{fig:rewrite-expr-shown} shows some expression-level rules.
Judgments of the form $\senvba, \loopvarrep \vdash \expre \looprewriteto \expre'$ mean that the expression $\expre$ is rewritten to $\expre'$ under environments $\senvba$ and $\loopvarrep$.
They use the same environments as the statement-level rewrites, but only return the new expression.

\input{figures/rewrite/rewrite-expr-primary}

As shown by rule (R-LVar) in Figure~\ref{fig:rewrite-expr-shown}, every occurrence of the loop variable is replaced by its substitute stored in $\loopvarrep$.
The substitute is a 1-D array, so each expression involving the loop variable gains a new dimension.
The remaining expression rules propagate this new axis recursively.
Unary operators do not need to change as they are applied element-wise to their sole argument.
The $reduce$ operators, $\expandname$, $\replicatename$, and the shape access operator take an array and the specified axes as arguments.
As shown in the (R-Shape) and (R-Exp) rules, if the array argument is lifted, we shift the target axes one position to the right to ensure the operators are still applied to the same axes as before.

For element-wise operators, rewrites are guided by shapes.
The (R-Biop) rule shows how to rewrite binary operations.
The shapes before the rewrite indicate whether broadcasting occurred originally.
If a lifted operand was broadcast to have additional axes before the rewrite, we insert axes of size $1$ after the newly added axis.
The number of inserted axes equals the number of axes added by the broadcast.
Doing so keeps the positions of the axes added by the broadcast while keeping the newly added axis as the left-most one.

The (R-Indx) rule rewrites indexing expressions.
As indexers are broadcast to a common shape, the rule first rewrites each indexer and inserts new axes when needed, similar to rewriting binary operations.
It then rewrites the base when the expression is not on the left-hand side of an update.
If the base is lifted, we add an indexer for the new left-most axis.
This indexer is the loop variable substitute stored in $\loopvarrep$, which is an array containing numbers of iterations where the expression is evaluated.
The new indexer needs to be expanded to ensure proper broadcasting and propagate the new axis.
Finally, the rule applies \textsc{UnmaskIndexer} (discussed in Section~\ref{subsec:branches}), to ensure that there is no masked indexer.

\begin{example}
Consider the code below.
\begin{lstlisting}[style=dsl,escapechar=\%]
# a: shape (L, M, N), not masked
for x in 0...S(a)[0] do
  ys %\deflit% %\arangelit%(S(a)[1]); # (M,)
  t %\deflit% a[x, ys, 0]; # (M,)
  ns %\deflit% %\arangelit%(S(t)[0]); #(M,)
  a[x, ys, 0] %$\getslit$% t[ns] + x;
\end{lstlisting}
To rewrite the for-loop, we set $\arange{\codeshapeaccess{\mathtt{a}}{0}}$ as the loop variable substitute in $\loopvarrep$.
The first statement in the loop is unchanged by the rewrite.
The next statement defines $\mathtt{t}$ using an expression involving $\mathtt{x}$.
For this statement, we replace $\mathtt{x}$ with the loop variable substitute.
Since $\mathtt{ys}$ has shape $(M,)$, the indexers were broadcast originally.
To preserve the broadcast axis, we expand a new right-most axis for the first indexer.
As $\mathtt{a}$ is not lifted, no new indexer is needed.
The returned type environment shows that $\mathtt{t}$ is lifted to the new shape $(L, M)$.
Thus, we increment the axis argument in the shape access expression in the next statement.
For the last update statement, we rewrite its left-hand side in the same way as before.
On the right-hand side, the base $\mathtt{t}$ is lifted, so we add an indexer for its new axis.
Because $\mathtt{ns}$ was 1-D before, the new indexer is expanded for proper broadcasting.
Here, $\mathtt{x}$ is also replaced by the loop variable substitute and then expanded.
The final result is shown below.
\begin{lstlisting}[style=dsl,escapechar=\%]
  xs %\deflit% %\arangelit%(S(a)[0]);
  xse %\deflit% %\expandlit%(xs, (1,));
  ys %\deflit% %\arangelit%(S(a)[1]);
  t %\deflit% a[xse, ys, 0]; # (L, M)
  ns %\deflit% %\arangelit%(S(t)[1]);
  a[xse, ys, 0] %$\getslit$% t[xse, ns] + xse;
\end{lstlisting}
\end{example}

\subsection{Rewriting Branches} \label{subsec:branches}

Recall from Section~\ref{sec:overview} that rewriting branches requires excluding values from invalid iterations.
For such rewrites, we first rewrite the condition.
If it is not lifted, we rewrite the two branch bodies and keep the branch structure.
If lifted, we flatten the branches by composing the rewritten branch bodies sequentially, following the (R-Brch1) rule in Figure~\ref{fig:rewrite-stmt-shown}.

Before rewriting each branch body, the rule updates the loop variable substitute stored in $\loopvarrep$ by masking the original substitute with the lifted condition using a $\makemaskedarrayname$ call. 
The substitute then has all the entries corresponding to invalid iterations masked.
In other words, the then-branch receives the iteration numbers for which the condition is true, with all others masked, and the else-branch uses the negated condition.
Nested branches work in the same way, because the effective mask of nested $\makemaskedarrayname$ calls is the conjunction of the masks from each call.
The rewrite also masks loop-local variables defined outside the flattened branch, ensuring invalid values not leaked into the branch.

Because the DSL disallows masked indexers, we use \textsc{UnmaskIndexer} to check if any indexer is masked under the updated type environment, and if so, wrap them in $\filledname$ with 0 as the fill value.
We also wrap the entire indexing expression in $\makemaskedarrayname$, masking it with the conjunction of the indexers' masks.
Such rewrites ensure that the entries selected by the supposedly masked indices remain masked in the result.
Normal updates may be turned into masked updates, if the left-hand side gets masked after such rewrites.

\begin{example}
Consider the code below as an example.
\begin{lstlisting}[style=dsl,escapechar=\%]
# a, b: shape (M,), not masked
for i in 0...S(a)[0] do
  if i > 2 then 
    a[i] %$\getslit$% b[i];
\end{lstlisting}
The code block below shows the rewrite result.
\begin{lstlisting}[style=dsl,escapechar=\%]
is_ %\deflit% %\arangelit%(S(a)[0]);
mis_t %\deflit% %\makemaskedarraylit%(is_, is_ > 2);
bis_ %\deflit% %\makemaskedarraylit%(b[%\filledlit%(mis_t,0)], is_ > 2);
%\makemaskedarraylit%(a[%\filledlit%(mis_t,0)], is_ > 2) %$\getslit$% bis_;
\end{lstlisting}
Here, replacing $\mathtt{i}$ with $\mathtt{is\_}$ lifts the branch condition to the array $\mathtt{is\_ > 2}$.
Before rewriting the then-branch, the rewrite records $\mathtt{mis\_t}$ as the loop variable substitute, masking out iterations where the condition is false.
All occurrences of $\mathtt{i}$ in the then-branch are replaced by this masked substitute.
Because masked arrays cannot be indexers, the rewrite fills the masked indices with $0$ and moves the same mask to the indexing result.
The last statement is a masked update, so it updates exactly the elements of $\mathtt{a}$ corresponding to iterations that execute the then-branch.
Were there an else-branch, it would be handled analogously with the negated condition.
\end{example}

\subsection{Dataflow Analysis for the Depends-On Relation}

There are two code patterns that need special handling.
The first is \emph{false dependence} \cite{pugh1992eliminating}: a loop-local variable defined by an expression not lifted is updated by a lifted expression. In this case, the defining expression must be lifted explicitly to fit the updates.
The second pattern is \emph{branch-restrained variables defined outside of the loop}.
If the flattened branch puts conditions on other loop variables, only masking loop-local variables may incorrectly remove those restrictions.
To address these issues, we first give the following definition.
\begin{definition}[Depends-on]
A non-variable expression $\expre$ depends on a loop variable $y$ if evaluating $\expre$ directly involves a variable depending on $y$.
A variable $x$ defined inside the loop of $y$ depends on $y$ if its defining or updating expressions depends on $y$, or if $x$ is updated inside a branch whose condition depends on $y$.
Loop variables are self-dependent.
\end{definition}
Before rewriting, a dataflow analysis maps each variable to a set of loop variables on which it depends.
% The analysis is standard, so we omit the detailed rules.
The analysis result allows us to identify variables incurring false dependence and use $\replicatename$ to explicitly expand them.
It also allows us to find variables escaping the restrictions when flattening branches and thus mask them.
\begin{example}
This example shows both uses of the analysis.
\begin{lstlisting}[style=dsl,escapechar=\%]
# a: shape (M, N), not masked; assume M > N
for i in 0...S(a)[0] do
  for j in 0...S(a)[1] do
    b %\deflit% %\oneslit%((1,));
    if i < j then
      b[0] %$\getslit$% a[j, i];
\end{lstlisting}
After rewriting the inner loop, we obtain the code below.
\begin{lstlisting}[style=dsl,escapechar=\%]
for i in 0...S(a)[0] do
  js %\deflit% %\arangelit%(S(a)[1]);
  b %\deflit% %\replicatelit%(%\oneslit%((1,)),(0,),(S(a)[1],));
  m %\deflit% i < js; # shape (N,)
  ir %\deflit% %\makemaskedarraylit%(i, m); # shape (N,)
  jsm %\deflit% %\makemaskedarraylit%(js, m);
  jsmf %\deflit% %\filledlit%(jsm, 0);
  irf %\deflit% %\filledlit%(ir,0);
  %\makemaskedarraylit%(b[jsmf, 0],m) %$\getslit$% \
    %\makemaskedarraylit%(a[jsmf, irf], %\logicalandlit%(m, m));
\end{lstlisting}
The dataflow analysis reports that $\mathtt{b}$ depends on $\mathtt{j}$.
Therefore, the rewrite wraps its definition in $\replicatename$, creating a copy for each $\mathtt{j}$.
Since $\mathtt{b}$ is now lifted, indexing into $\mathtt{b}$ needs a new indexer.
The analysis also shows that $\mathtt{i}$ is restricted by the branch.
When the branch is flattened, the $\mathtt{i}$ in $\mathtt{a[j, i]}$ is masked by $\mathtt{i < js}$.
The masked indexers are then filled as before, leaving $\mathtt{jsmf}$ and $\mathtt{irf}$ as the final indexers.
As the potentially out-of-bound values are masked before indexing into $\mathtt{a}$, the rewritten expression remains well-defined.
\end{example}

\subsection{Vectorizability} \label{sec:vectorizability}

In general, our rewrite rules cannot handle unvectorizable loops.
Loops' vectorizability is determined by the presence of loop-carried dependence~\cite{allen1987automatic}.
As discussed in prior works~\cite{bernstein1966, kennedy2001optimizing}, if two statements access the same memory location, and at least one of them writes, the later one is said to be \emph{dependent} on the earlier one.
In most cases, cycles or backward dependence in dependence graphs prevent direct vectorization~\cite{kennedy2001optimizing, mendis2024advanced}.
Thus, our technique cannot vectorize programs with loop-carried dependence, except for one special case.

Since loops are commonly used to implement reduction operators such as sum, we add a rule that specifically rewrite such statements. 
First, we introduce the following definition.
\begin{definition}[Rewritable reduce statements]
An update statement in a loop is a \emph{rewritable reduce statement} if it meets the following conditions:
(1) the top-level operator on the right-hand side is one of $\codeliteral{+}$, $\codeliteral{*}$, $\codeliteral{maximum}$, or $\codeliteral{minimum}$;
(2) the first operand of the operator is the same as the update target;
(3) the updated variable is not defined in the same loop; and
(4) the left-hand side depends neither on the loop variable of the enclosing loop nor any restricted loop variable. 
\end{definition}

For loops with one such statement, we can replace it with the corresponding reduction operator.
If the reduced expression is not lifted, we explicitly replicate it.
If the statement is in a branch, the reduction is done selectively via masked arrays.
To avoid reductions on masked arrays returning a masked value, we fill the reduced array with the reduction's identity value.

\begin{example}
Consider the code snippet below.
The first loop adds up a sequence of 10s, which can be replaced by a call to $\codesum$.
The second finds the smallest $\mathtt{j}$ such that $\mathtt{j} \geq \mathtt{a[0]}$, which is a $\codemin$ reduction on a subset of the values taken by $\mathtt{j}$.
Both loops contain rewritable reduce statements.
\begin{lstlisting}[style=dsl,escapechar=\%]
# a: shape (M,), not masked; assume M < 100.
acc %\deflit% %\oneslit%((2,)) * 100;
for i in 0...S(a)[0] do
  acc[0] %$\getslit$% acc[0] + 10;
for j in 0...S(a)[0] do
  if j >= a[0] then
    acc[1] %$\getslit$% %\minimumlit%(acc[1], j);
\end{lstlisting}
After the write we get this code:
\begin{lstlisting}[style=dsl,escapechar=\%]
acc %\deflit% %\oneslit%((2,)) * 100;
rep_10 %\deflit% %\replicatelit%(10, (0,), (S(a)[0],));
acc[0] %$\getslit$% acc[0] + %\sumlit%(rep_10, 0);
js %\deflit% %\arangelit%(S(a)[0]);
jsm %\deflit% %\makemaskedarraylit%(js, js >= a[0]);
jsmf %\deflit% %\filledlit%(jsm, inf);
acc[1] %$\getslit$% %\minimumlit%(acc[1],%\minlit%(jsmf, 0));
\end{lstlisting}

In the first loop, because 10 is not lifted, we replicate it explicitly for reduction.
Adding the reduction expression to the accumulator is equivalent to performing the additions iteratively.
In the second loop, the reduced array is $\mathtt{jsm}$, which is lifted from $\mathtt{j}$ and no replication is needed.
If $\mathtt{a[0]}$ is greater than every value of $\mathtt{j}$, $\mathtt{jsm}$ is fully masked, and reducing it would result in a masked value.
To avoid this, we fill $\mathtt{jsm}$ with $\mathtt{inf}$, which denotes infinity, the identity value for $\codemin$.
\end{example}

Next, we define rewritability more precisely in the context of our DSL and state the main correctness theorem.
\begin{definition}[Rewritability]
A loop is rewritable if (1) it has no loop-carried dependence, or (2) it contains only one rewritable reduce statement and the only loop-carried dependence is the rewritable statement on itself.
A program is rewritable if all of its loops are rewritable.
\end{definition}

\begin{theorem}[Correctness of the \textsc{Vectorize} Routine] \label{thm:soundness-vectorize-shown}
Let $\cprog$ be a rewritable program, and $\cann$ be the type annotation for its arguments.
If \textsc{Vectorize}($\cprog, \cann$) returns $\cprog'$, then, for any evaluation environment $\venvs$ under which the execution of $\cprog$ returns a value, executing $\cprog'$ under $\venvs$ returns the same value.
\end{theorem}
\begin{proof}
The proof is provided in Appendix~\ref{sec:rewrite-proof}.
\end{proof}

%% file: figures/rewrite/algo-vectorize.tex
\begin{figure}[!t]
\begin{algorithm}[H]
\caption{Top-level algorithm for vectorizing programs.}
\label{algo:vectorize}
\begin{algorithmic}[1]
\Procedure{\textsc{Vectorize}}{$\cprog, \cann$}
\Statex \textbf{Input:} A program $\cprog$ and the input type annotation $\cann$.
\Statex \textbf{Output:} The vectorized program.

\While{$\textsf{HasVectorizableLoops}(\cprog)$}
    \State $\senv \gets \textsc{AnalyzeShape}(\cprog, \cann)$
    \State $\cloop \gets \textsf{InnerMostLoop}(\cprog)$
    \State $\cloop', \_ \gets \textsc{Rewrite}(\cloop, \senv, \senv, \emptyset)$ 
    \State $\cprog \gets \cprog[\cloop'/\cloop]$
\EndWhile
\State \Return $\cprog$

\EndProcedure
\end{algorithmic}
\end{algorithm}
\vspace{-20pt}
\end{figure}

%% file: figures/rewrite/rewrite-stmt-primary.tex
\begin{figure}[t]
\[
\scriptsize
\begin{array}{c}
\irulelabel{
    \begin{array}{c}
    \loopvarrep' = \{x \mapsto \arange{\integrali}\} \quad
    \senvb, \senva, \loopvarrep' \vdash \stmts \stmtrewriteto \stmts', \senva'
    \end{array}
}{
    \senvb, \senva, \loopvarrep \vdash \cmdFor{x}{ \codeliteral{0 \ldots} \integrali}{\stmts} \stmtrewriteto \stmts', \senva'
}{
    \textrm{(R-For)}
}
\\ \ \\

\irulelabel{
    \begin{array}{c}
        \senvba, \loopvarrep \vdash \stmts_1 \stmtrewriteto \stmts'_1, \senva' \quad
        \senvb, \senva', \loopvarrep \vdash \stmts_2 \stmtrewriteto \stmts'_2, \senva'' 
    \end{array}
}{
    \senvba, \loopvarrep \vdash \stmts_1 \codesemicolon \stmts_2 \stmtrewriteto \stmts'_1 \codesemicolon \stmts'_2, \senva''
}{\textrm{(R-Seq)}} 
\\ \ \\

\irulelabel{
    \begin{array}{c}
        \senvb \vdash \expre \shapetypeof \shapes \quad
        \senva \vdash \expre' \shapetypeof \shapes' \quad
        \senva \vdash \expre' \masktypeof \maskm' \\ 
        \senvba, \loopvarrep \vdash \expre \looprewriteto \expre' \quad
        \shapes \neq \shapes' \quad
        \senva' =  \senva[\name \mapsto (\shapes', \maskm')]
    \end{array}
}{ \senvba, \loopvarrep \vdash \cmdAssign{x}{\expre} \stmtrewriteto \cmdAssign{x}{\expre'}, \senva'}
{
    \textrm{(R-Bind1)}
}
\\ \ \\

\irulelabel{
    \begin{array}{c}
        \senvba, \loopvarrep \vdash \expre_{c} \looprewriteto \expre'_{c} \quad 
        \senvb \vdash \expre_c \shapetypeof \shapes_c \quad
        \senva \vdash \expre'_c \shapetypeof \shapes'_c \quad
        \shapes'_c \neq \shapes_c \\
        \curloopvarreplacer(\loopvarrep) = \expre_s \;
        \curloopvar(\loopvarrep) = y \enspace
        \expre''_{c} = \logicalnot{\expre'_{c}} \\
        \senvba, \loopvarrep[y \mapsto \makemaskedarray{\expre_s}{\expre'_{c}}] \vdash \stmts_{t} \stmtrewriteto \stmts'_{t}, \senva' \\
        \senvb, \senva', \loopvarrep[y \mapsto \makemaskedarray{\expre_s}{\expre''_{c}}] \vdash \stmts_{e} \stmtrewriteto \stmts'_{e}, \senva''
    \end{array}
}{ 
    \begin{array}{c}
            \senvba, \loopvarrep \vdash \cmdITE{\expre_{c}}{\stmts_{t}}{\stmts_{e}} \stmtrewriteto \stmts'_t \codesemicolon \stmts'_e, \senva'' 
    \end{array}
}{
    \textrm{(R-Brch1)}
}

\end{array}
\]
\vspace{-10pt}
\caption{Sample statement rewrite rules. }
\label{fig:rewrite-stmt-shown}
\vspace{-10pt}
\end{figure}

%% file: figures/rewrite/rewrite-expr-primary.tex
\begin{figure*}[!t]
\[
\scriptsize
\hspace{-10pt}
\begin{array}{c}
\irulelabel{ 
    \begin{array}{c}
        \textbf{Var } x \quad
        \curloopvarreplacer(\loopvarrep) = \expre_s \\
        \curloopvar(\loopvarrep) = x
    \end{array}
}{
    \senvba,\loopvarrep \vdash x \looprewriteto \expre_s
}{
    \textrm{(R-LVar)}
} 
\quad

\irulelabel{
    \begin{array}{c}
        d = c + 1 \quad
        \senvb \vdash \expre_t \shapetypeof \shapes \\
        \senvba, \loopvarrep \vdash \expre_t \looprewriteto \expre'_t \quad
        \senva \vdash \expre'_t \shapetypeof \shapes' \\
        \expre' = \ite{\shapes \neq \shapes'}{\codeshapeaccess{\expre'_t}{d}}{\codeshapeaccess{\expre'_t}{c}}
    \end{array}
}{ 
    \senvba, \loopvarrep \vdash \codeshapeaccess{\expre_t}{c} \looprewriteto e'
}{
    \textrm{(R-Shape)}
}
\quad

\irulelabel{
    \begin{array}{c}
        \expre = \expand{\expre_t}{\codetuple{c_0, \ldots, c_m}}\\
        \senvb \vdash \expre_t \shapetypeof \shapes \quad 
        \senvba, \loopvarrep \vdash \expre_t \looprewriteto \expre'_t \\
        \senva \vdash \expre'_t \shapetypeof \shapes' \quad
        0 \leq j \leq m \\ 
        d_j = \ite{\shapes' \neq \shapes}{c_j+1}{c_j} \\
        \expre' = \expand{\expre'_t}{\codetuple{d_0, \ldots, d_m}}
    \end{array}
}{ 
    \senvba, \loopvarrep \vdash \expre \looprewriteto \expre'
}{
    \textrm{(R-Exp)}
}
\\ \ \\

\irulelabel{
    \begin{array}{c}
        \fops \in \text{Binary } \opops \quad
        \senvb \vdash \expre_1 \shapetypeof (a_{m+n}, \ldots, a_0) \\
        \senvb \vdash \expre_2 \shapetypeof \shapes_2 \quad
        \shapes_2 = (b_{m}, \ldots, b_0) \\
        \senvba, \loopvarrep \vdash \expre_1 \looprewriteto \expre'_1 \quad
        \senvba, \loopvarrep \vdash \expre_2 \looprewriteto \expre'_2 \\
        \senva \vdash \expre'_2 \shapetypeof \shapes'_2 \quad
        \expre''_2 = \expand{\expre'_2}{\codetuple{\codeliteral{1}, \ldots, n}}\\
        \expre'''_2 = \ite{\shapes'_2 \neq \shapes_2}{\expre''_2}{\expre'_2}
    \end{array}
}{
    \senvba, \loopvarrep \vdash \f{\expre_1, \expre_2} \looprewriteto \f{\expre'_1, \expre'''_2}
}{
    \textrm{(R-Biop)}
}
\quad

\irulelabel{
    \begin{array}{c}
        \expre = \arrindex{\expre_t}{\expre_m, \ldots, \expre_1} \quad
        \senvb \vdash \expre_t \shapetypeof \shapes \quad
        \senvba, \loopvarrep \vdash \expre_t \looprewriteto \expre'_t \quad 
        \expre''_t = \ite{\notonlhs{\expre}}{\expre'_t}{\expre_t} \\
        \senva \vdash \expre''_t \shapetypeof \shapes'' \quad
        \senvba,\loopvarrep \vdash \squarebrack{\expre_m, \ldots, \expre_1} \looprewriteto \squarebrack{ \expre'_m, \ldots, \expre'_1 } \quad
        \curloopvarreplacer(\loopvarrep) = \expre_s\\
        1 \leq j \leq m \quad 
        \senvb \vdash \expre_j \shapetypeof (a_{n_j}^j, \ldots, a_0^j) \quad 
        \auximax(n_1, \ldots, n_m) = n' \\
        \expre' = \ite{\shapes'' = \shapes}{\arrindex{\expre''_t}{\expre'_m, \ldots, \expre'_1}}{\arrindex{\expre''_t}{\expand{\expre_s}{\codetuple{\codeliteral{1}, \ldots, n'}}, \expre'_m, \ldots, \expre'_1}} \\
    \end{array}
}{
    \senvba, \loopvarrep \vdash \expre \looprewriteto \rmmaskedindexer(\expre', \senva)
}{ \textrm{(R-Indx)}}

\end{array}
\]
\vspace{-10pt}
\caption{Sample rewrite rules for expressions. $\onlhs{\expre}$ checks if $\expre$ is on the left-hand side of an update. }
\label{fig:rewrite-expr-shown}
% \vspace{-10pt}
\end{figure*}

%% file: sections/implementation.tex
\section{Implementation}

We have implemented the proposed technique in a tool called \tool based on standard \numpy and the \lstinline|ast| module of Python, making \tool lightweight and easy to run.
Furthermore, many constructs in practical \numpy programs map directly to our DSL, which allows \tool to process source programs with only minor adaptations, such as renaming variables to satisfy the requirements.

\bfpara{Postprocessing the rewrite results.}
Code generated by the rewrite rules may incur unnecessary runtime overhead or contain masked updates that cannot be expressed directly in Python because function calls are not l-values.
To address these issues, \tool applies a pipeline of postprocessing passes that transforms rewritten programs to executable Python while eliminating common inefficiencies.
For example, the pipeline simplifies boolean mask manipulations, replaces masked updates with \lstinline|np.where| and advanced indexing with slices and transposes, collapses sum-after-multiplications into \lstinline|np.tensordot| calls, and removes unnecessary materializations, copies, calls to $\filledname$, and masked array constructions. 
It also performs standard optimizations such as common-subexpression elimination (CSE) and unused-variable elimination.
Because the main rewrite procedure emits mostly straight-line code, the resulting programs are amenable to dataflow analyses such as reaching definitions, making these postprocessing passes effective and straightforward.

%% file: sections/eval.tex
\section{Evaluation}

We evaluated \tool on \numpy programs that use explicit loops over arrays to answer the following questions.
\vspace{-14pt}
\begin{itemize}[leftmargin=30pt]
\item[\textbf{RQ1.}] Is \tool effective at vectorizing programs?
\item[\textbf{RQ2.}] How much time is needed to vectorize programs?
\item[\textbf{RQ3.}] How much performance improvement does the program vectorized by \tool achieve?
\end{itemize}
\vspace{-6pt}

\subsection{Evaluation Set-up}

\vspace{-5pt}
\bfpara{Benchmarks.}
We collected 150 benchmarks from 12 datasets, which are \blas \cite{blackford2002updated}, \blend \cite{blend}, \darknet \cite{darknet13}, \dsp \cite{dsp}, \dspstone \cite{zivojnovic1994dspstone}, \llama \cite{llamacpp}, \makespeare \cite{rosin2019stepping}, \mathfu \cite{mathfu}, \polybench \cite{polybench}, \simplarray \cite{so2017synthesizing}, \utdsp \cite{saghir1998application}, and \stackoverflow.

Among these datasets, \blas, \darknet, \dsp, \dspstone, \makespeare, \mathfu, \simplarray, and \utdsp were collected by \citet{c2taco} and are available as C source code.
The \blend and \llama datasets were collected by \citet{blend} and \citet{qiu2024tenspiler}, available as C++ code.
The 15 \polybench benchmarks were selected by \citet{brauckmann2025tensorize} from the \polybench suite~\cite{polybench}.
For these benchmarks, we translated the original C or C++ source code into Python programs using \numpy while preserving their logic and structure.

In addition, we also created 51 benchmarks based on questions from \stackoverflow.
Specifically, we searched for code snippets and functionality descriptions in questions related to \numpy and vectorization, and created benchmarks based on the problem descriptions or code snippets in the questions.
Compared to the benchmarks translated from C or C++, some of these \stackoverflow benchmarks already use high-level \numpy features and are therefore more representative of practical \numpy programs.

\vspace{-3pt}
\bfpara{Baselines.}
We compared \tool with two state-of-the-art tools for vectorizing array programs: \tensorize\cite{brauckmann2025tensorize} and \tenspiler\cite{qiu2024tenspiler}.
Specifically, \tensorize uses symbolic execution and a symbolic algebraic solver to search for a program using high-level APIs that is equivalent to the original program.
It expects input in the Affine dialect of MLIR, though we were unable to generate such code that \tensorize accepts.
Hence, we compared with it only on the benchmarks whose MLIR code is available in its repository, which covers all datasets except \stackoverflow. 
\tenspiler uses verified lifting techniques based on program synthesis to search for a vectorized program equivalent to the original.
To compare with \tenspiler, we manually translated the benchmarks not used in its evaluation into C++ as input.
Since \tenspiler requires users to provide driver scripts to configure synthesis for each benchmark, we reused the driver scripts from its repository for benchmarks included in its evaluation.
For other benchmarks, we wrote automatic driver scripts with which \tenspiler infers the synthesis configurations.

\vspace{-3pt}
\bfpara{Configurations.}
We set a 30-minute time limit for synthesis per benchmark for each tool.
All experiments were conducted on a machine with an Intel Core i5-13600K CPU and 64 GB of memory, running Linux Mint 22.

\subsection{Effectiveness and Efficiency}

\begin{figure}[!t]
    \includegraphics[width=0.85\linewidth]{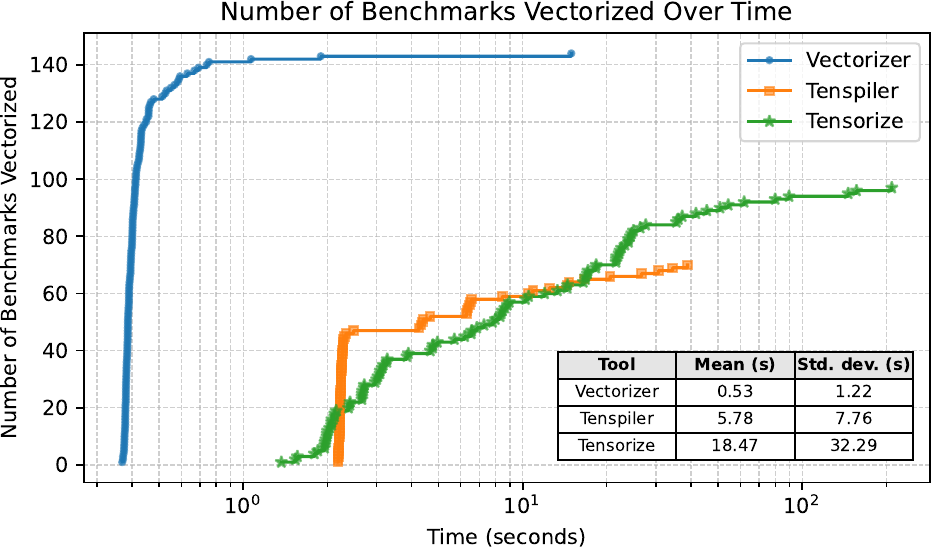}
    \vspace{-12pt}
    \caption{Running time. The x-axis is on a log scale. The table reports the mean and standard deviation of the per-benchmark vectorization time for each tool.}
    \label{fig:synthesis-time}
    \vspace{-10pt}
\end{figure}

\vspace{-5pt}
\bfpara{Effectiveness.}
Across the 150 benchmarks, \tool vectorizes 142, corresponding to an approximately 95\% success rate.
Upon further inspection of the 8 failed benchmarks, we found that 2 can be solved after minor changes to the code.
Specifically, one needs to replace a branch guarding an update with the $\codeliteral{\maximumlit}$ operator.
The other requires adding more axes to an argument to eliminate false dependence, which our technique cannot handle automatically.
The remaining 6 failures are primarily due to complex loop-carried dependence that our technique cannot handle.

By contrast, \tenspiler solves 70 of the 150 benchmarks.
Upon inspection, we found that \tenspiler can efficiently find solutions with driver scripts that have manually specified configurations, but it is less effective with automatic driver scripts and imposes several syntactic restrictions on inputs.
For example, with automatic drivers, \tenspiler can only handle benchmarks with a single top-level loop that cannot be more than two levels deep.

Among the 99 benchmarks lowered to MLIR in the \tensorize repository, \tensorize generates output for 97.
However, many outputs are not directly executable because they contain undefined symbols or incorrect arguments to \numpy APIs.
Therefore, we considered these outputs to be intermediate representations and used coding agents based on large language models (LLMs) to analyze and normalize them while preserving their structure, operator sequence, and logic as closely as possible.
LLM-assisted inspection of the normalized code suggested that 15 of the 97 outputs are semantically different from their source programs.

\bfpara{Efficiency.}
Figure~\ref{fig:synthesis-time} presents the running time of different tools.
On average, \tool takes 0.53 seconds to vectorize a program, which is about 10.91$\times$ faster than \tenspiler (5.78 seconds) and 34.85$\times$ faster than \tensorize (18.47 seconds).
Because \tool does not rely on search or symbolic reasoning, it is fast even when input programs contain semantically complex operators or control-flow structures.

\bfpara{Answer to RQ1 and RQ2.}
The results show that \tool is both effective and efficient.
It correctly vectorizes substantially more benchmarks than other tools and frees users from per-benchmark configuration and normalizing the emitted code.

\subsection{Performance Improvement}

To measure performance improvement, we compared each vectorized program against its original loop-based version and computed the speedup of the vectorized program.
To measure the execution time, we generated random input data whose sizes correspond to 25,000,000 iterations of the inner-most loops. 
The workload is selected to minimize the variance caused by overly simple executions while preventing the experiments from being prohibitively slow or running out of memory on the host machine.
We applied this setup to all benchmarks except 3 from the \stackoverflow dataset, whose program structures require fixed, benchmark-specific dimensionalities.
We ran each program 10 times and computed the speedups using the median measurements.

\input{figures/evaluations/speedup-comparison}

\begin{figure*}[!t]
\centering
  \includegraphics[width=\linewidth]{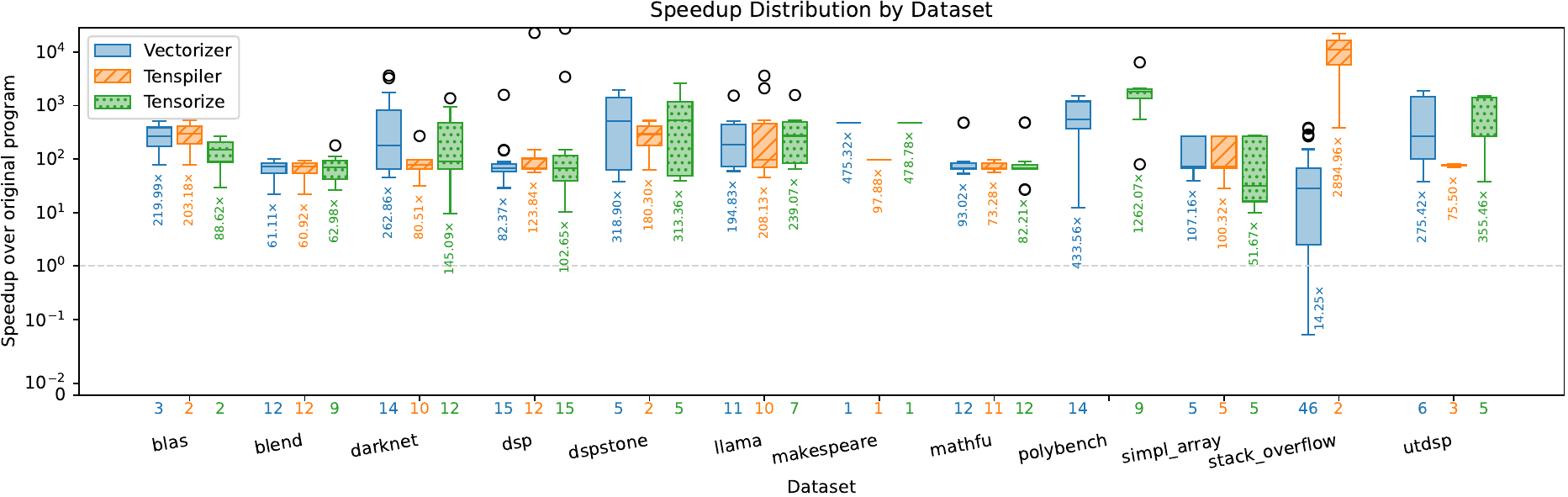}
  \vspace{-15pt}
  \caption{Speedup distribution by tool and dataset. Numbers on the x-axis show the number of benchmarks included in the box. Labels below boxes show geometric mean speedups of the benchmarks counted. The y-axis is on a log scale.}
  \label{fig:speedup}
  % \vspace{-5pt}
\end{figure*}

\bfpara{Results.}
Table~\ref{tab:speedup} summarizes the geometric mean speedups, and Figure~\ref{fig:speedup} shows their distributions.
Overall, \tool's emissions gain modestly higher speedups than those emitted by other tools. 
Nevertheless, \tool slows down 7 benchmarks from \stackoverflow, and our inspection reveals two primary causes.
First, vectorization can eagerly materialize large intermediate arrays, whose memory management overhead may outweigh the computational benefits of vectorization. 
In such cases, loops that operate on small, cache-friendly array chunks can be more efficient.
Second, for some benchmarks involving branches or data filtering, vectorization may replace control flow with predicated execution or reduction, which eliminates the branch and memory pruning available to the original loop-based implementation.
In a few extremely complex benchmarks, the current postprocessing pipeline may not be sufficient to simplify mask manipulations.
One way to address these issues is to design a workload-aware cost model for deciding whether vectorization is desirable.
We leave it as future work.

On 2 benchmarks, \tensorize and \tenspiler produce shallow copies of the inputs instead of the intended deep copies, yielding apparent speedups exceeding $22,000\times$.
\tensorize assumes real number arithmetics and uses real division for a few benchmarks instead of integer division, which is slower.

Compared to \tenspiler, \tool picks up more sum-after-multiplication snippets and replaces them with fast \lstinline|np.tensordot| calls.
The CSE pass also eliminates some redundancy kept by \tenspiler.
\tensorize, aided by the symbolic algebra solver, finds more compact computation routes for some benchmarks.
However, it often uses \lstinline|np.full| to explicitly materialize large arrays of constants, 
while \tool can take advantage of implicit broadcasting.

\bfpara{Comparing with \numba.}
All comparisons so far focus on source-to-source translations.
Next, let us consider whether \tool's vectorization also complements just-in-time (JIT) compilation.
To this end, we evaluate \numba~\cite{lam2015numba}, a popular JIT compiler for Python and NumPy programs, on both the original benchmarks and \tool's emissions.
\numba only supports a limited subset of \numpy. 
Adding the \lstinline|@njit| decorator allows \numba to compile 139 of the 150 original benchmarks, producing a geometric mean speedup of $77.56\times$.
However, some operations in \tool's emissions, such as \lstinline|np.tensordot|, fall outside the subset of \numpy supported by \numba.
We thus apply an additional postprocessing pass that replaces common unsupported operations with \numba-compatible equivalents.
After postprocessing, \numba compiles 122 of \tool's emissions, which achieve a geometric mean speedup of $111.91\times$ over the original, non-compiled benchmarks.
For a direct comparison, we consider the 121 benchmarks for which \numba compiles both the original program and \tool's emission.
On this common subset, applying \numba to \tool's emissions yields a geometric mean speedup of $1.45\times$ over applying \numba directly to the original programs.

\bfpara{Answer to RQ3.}
\tool substantially improves loop-based programs with a geometric mean speedup of 74.83$\times$.
Such vectorization also benefits \numba JIT compilations, bringing a geometric mean speedup of 1.45$\times$.

%% file: figures/evaluations/speedup-comparison.tex
\begin{table}[t]
\centering
\footnotesize
\caption{Geometric mean of speedups.
The two benchmarks solved after minor changes are included.
Benchmarks solved by \tool are a superset of others.
The speedup in row $x$ and column $y$ denotes the speedup of tool $x$ on the benchmarks that can be solved by tool $y$.}
\vspace{-8pt}
\label{tab:speedup}
\begin{tabular}{c c c c c}
\toprule
\textbf{Vectorized by} & \makecell{\tool \\ (144)} & \makecell{\tenspiler \\ (70)} & \makecell{\tensorize \\ (82)} & \makecell{All \\ (59)} \\
\midrule
\tool                  & 74.83$\times$             & 107.16$\times$                & 163.25$\times$                & 106.71$\times$ \\
\tenspiler             & -                         & 110.27$\times$                & -                             & 93.89$\times$  \\
\tensorize             & -                         & -                             & 157.67$\times$                & 101.54$\times$ \\
\bottomrule
\end{tabular}
\vspace{-12pt}
\end{table}

%% file: sections/relatedwork.tex
\section{Related Work} \label{sec:related}

\vspace{-5pt}
\bfpara{DSLs for programming multidimensional arrays.}
Many specialized libraries, DSLs, and IRs support array programming. 
\numpy~\cite{harris2020array} is widely used, underpins libraries such as SciPy~\cite{2020SciPy} and OpenCV-Python~\cite{opencv-library}, and has influenced machine-learning frameworks with similar APIs~\cite{tensorflow2015-whitepaper,jax2018github,chen2015mxnet,paszke2019pytorch}.
Other systems adopt specialized models: Halide~\cite{ragan2013halide} separates image-processing algorithms from schedules, TACO~\cite{kjolstad2017tensor} compiles symbolic index notation into efficient tensor kernels, and StableHLO~\cite{hlo} provides a high-level IR for compiler stacks such as MLIR~\cite{mlir}. 
We target \numpy because its API and programming model are broadly adopted and have influenced many subsequent array-programming systems.

\bfpara{Program synthesis for multidimensional array DSLs.}
Differences in programming models and interfaces make array DSLs difficult to adopt and complicate the migration of legacy code, motivating automated synthesis.
TF-Coder~\cite{shi2022tf} uses enumerative search to synthesize TensorFlow programs from input-output examples, while \ctotaco~\cite{c2taco} lifts array-manipulating C code to TACO. 
\dexter~\cite{blend} combines pattern matching with enumerative search to lift C++ image-processing code to Halide, and \tensorize~\cite{brauckmann2025tensorize} uses symbolic algebra to guide lifting from MLIR Affine to \numpy and MLIR HLO. 
Cobbler~\cite{lubin2024equivalence} uses syntactic canonicalization to synthesize equivalent rewrites of 1-D \numpy programs. 
Built on \metalift~\cite{bhatia2023building}, \tenspiler~\cite{qiu2024tenspiler} searches for logical program summaries to lift C++ code to high-level DSLs. 
Other approaches predict PyTorch API usage with machine learning~\cite{nam2022predictive} or employ LLMs for lifting~\cite{2026joseaccelerating,2025liguided,de2025guess}. 
In contrast, our type-directed rewrites lift loop-based array programs without costly search and are more deterministic and cost-effective than learning-based approaches.

\bfpara{Program rewrites.}
Rewriting is a core compiler optimization technique. 
The Glasgow Haskell Compiler lets programmers express domain-specific optimizations as rewrite rules~\cite{jones2001playing}, while modular rewriting systems and DSLs~\cite{visser1998building,visser2001stratego} separate transformation rules from strategies governing their application. 
This approach has been applied to instruction selection, program interpretation, and constant propagation~\cite{bravenboer2002rewriting,DOLSTRA200257,OLMOS2002156}. 
In high-performance computing, Steuwer et al.~\cite{steuwer2015generating} reshape high-level programs to generate OpenCL~\cite{munshi2009opencl} code, Panyala et al.~\cite{panyala2012use} use term rewriting to migrate applications across platforms, and Elevate~\cite{hagedorn2020achieving} optimizes functional programs for parallel architectures. 
We also use source-to-source rewrites, but for vectorizing \numpy programs.

\bfpara{Predicated execution.}
Predicated expressions have long been used to transform control flow into guarded computation.
Allen et al.~\cite{allen1983conversion} introduce logical guards for this purpose, and Park and Schlansker~\cite{park1991predicated} use predicates to flatten branched loops. 
Subsequent work applies predication to architectures supporting instruction-level parallelism~\cite{mahlke1995comparison,august1997framework,carter1999predicated}. 
Modern compilers similarly eliminate branches through selection and if-conversion: LLVM provides \texttt{llvm.vp.select} for vector selection~\cite{lattner2004llvm}, while GCC performs tree-level if-conversion for vectorization~\cite{stallman2003using}. 
Predicated execution is also supported at the hardware level.
For example, NVIDIA PTX provides predicated instructions~\cite{compute2010ptx}, and divergent warp paths execute serially with inactive threads disabled. 
\tool similarly represents conditional computation using masked arrays. While the idea is inspired by predicated execution, \tool focuses on source-to-source transformation of high-level \numpy programs.

%% file: sections/conclusion.tex
\section{Conclusion}
We presented \tool, a correct-by-construction approach for vectorizing loop-based \numpy programs using inside-out rewrites guided by array types and dataflow analysis.
\tool does not rely on search or symbolic reasoning, so it is consistently fast in practice.
\tool vectorizes 142 of the 150 benchmarks directly, and 2 more benchmarks after minor adaptations.
The average time needed for vectorizing a benchmark is 0.53 seconds, and the resulting programs achieved a geometric mean speedup of 74.83$\times$ over the original implementations.

%% file: sections/semantics.tex
\section{Formal Semantics of the Proposed DSL} \label{sec:sem}
The formal semantics of the proposed DSL is given in Figure~\ref{fig:expr-runtime-maskedness},~\ref{fig:stmt-semantics},~\ref{fig:prog-semantics},~\ref{fig:expr-semantics1}, and~\ref{fig:expr-semantics2}.
Figure~\ref{fig:expr-runtime-maskedness} specifies how to determine the maskedness of an expression at runtime under an evaluation environment.
Figure~\ref{fig:stmt-semantics} specifies how statements in a program are executed.
Figure~\ref{fig:prog-semantics} specifies the rule for evaluating a program in the DSL. 
Finally, Figure~\ref{fig:expr-semantics1} and~\ref{fig:expr-semantics2} specify how expressions in the DSL are evaluated to values under an evaluation environment.

Following the format of the paper, we use $\codeliteral{blue}$ texts for literal code in the proposed DSL.
In our formalism, the state of each variable at runtime is a pair $(\valuev, \maskmu)$. 
Here, $\valuev$ is the value of the variable and $\maskmu$ is the maskedness of the variable.
We use $\venvs$ to represent a program state, which is a stack of maps from variables to such pairs.
We refer to each map in the stack as a scope.
We use $\venv$ to represent a scope.
$\peek{\venvs}, \pop{\venvs},$ and $\push{\venv}{\venvs}$ return the scope at the top of the stack, the stack with the top scope removed, and the stack with a newly pushed-in scope $\venv$.
We use $\venvs \vdash \stmts \tovenvs \venvs'$ to denote that executing the statement $\stmts$ under the environment $\venvs$ results in a new state environment $\venvs'$.
We use $\venvs \vdash \expre \tovalue \valuev$ to denote an expression node $\expre$ is evaluated to a value $\valuev$ under the state $\venvs$.
Similarly, $\venvs \vdash \expre \todynmaskness \maskmu$ denotes that the maskedness of $\expre$ at runtime is $\maskmu$ when evaluated under $\venvs$.
We use $\denot{\fops}(\valuev_1, \ldots, \valuev_n)$ to denote the returned value of a routine $\fops$ with arguments $\valuev_1, \ldots, \valuev_n$.
We use $\dynarraymasked$ to denote arrays that are masked at runtime and $\dynarraynotmasked$ to denote arrays that are unmasked at runtime. 
We use the notation $\venvs \vdash \cmdProgram{x_0, \ldots, x_n}{\stmts}{y} \toprogramresult (\valuev, \maskmu)$ to denote that a program taking arguments $x_0, \ldots, x_n$ returns a value $\valuev$ whose runtime maskedness is $\maskmu$ under the initial environments $\venvs$.

To conveniently define the semantics of our DSL, we also included a few auxiliary operators and defined their semantics. 
The auxiliary operators are technically not part of our DSL.
The rules specifying the semantics of these auxiliary operators have names starting with ``Aux.''
Auxiliary operators' names are in the \textsc{small caps font}. 
The auxiliary indexing operator $\valuev\auxiindex{\integrali}$, which means picking out a top-level element from an array based on the given index, is in \textcolor{purple}{purple}, so it is differentiated from the indexing operator in the DSL.

At a high-level, executing statements under an evaluation environment results in an updated evaluation statement.
Variable binding statements create a new mapping from the variable name to a runtime value-maskedness pair in the top scope (SEM-Bind).
The execution of a branch entails the evaluation of the branch condition first (SEM-If). 
The corresponding branch is run under the evaluation environment with an empty scope pushed to the top.
After the execution of the branch, the top-level scope is removed.
Loops' execution is defined recursively, which is shown by (SEM-For0) and (SEM-Fori).
When the loop bound evaluates to zero, no update to the execution environment is made.
When the evaluation of the loop bound is non-zero, we first try to execute the loop with a decremented loop bound and then execute the loop under the current loop bound. 
Similar to the execution of branches, a new scope with only the state of the loop variable is pushed to the evaluation environment before each iteration's evaluation.
The top scope is removed after each iteration's evaluation.
$\auxiupdate{\venvs}{x}{\valuev}$ is an auxiliary operator that digs into $\venvs$ to find the scope in which $x$ is defined and updates the value bound to $x$ to $\valuev$. 
The semantics of this auxiliary operator is shown by the rules (SEM-Auxupd1) and (SEM-Auxupd2).
The semantics of normal update statements (SEM-Upd) and masked update statements (SEM-MskUpd) are defined in terms of this auxiliary operator.

The rules for evaluating expressions' value and maskedness are pretty straightforward.
For the value evaluations, we use the $\shape{\valuev}$ operator to specify the runtime shapes of the evaluation results.
The semantics of this auxiliary operator is shown by the rules (RTV-AuxSBase) and (RTV-AuxSInd).
We use the auxiliary indexing operator to specify how each individual element in the evaluation result is related to the elements in the operands.
The semantics of this auxiliary indexing operator is shown by the rule (RTV-AuxIdx).
As the denotational semantics for most of the operators in our DSL are well-known, we omit them here and use $\denot{\fops}$ to represent the result of calling these operators.

The rules for maskedness evaluations are also straightforward. 
In general, if one argument of an operator is masked, the evaluation result is also masked.
There are some exceptions.
Masked arrays cannot be indexers of indexing operators, and cannot be arguments to the $\matmulname$ operator. 
Also, the second argument of the $\filledname$ should not be masked.

Finally, the semantics of executing a program in the DSL is simply executing the program body under an evaluation environment that can evaluate all the arguments to the program.
The returned value and its maskedness are then fetched from the updated evaluation environment after the execution of the loop body.
\input{figures/semantics/runtime-maskness}

\input{figures/semantics/semantics}

%% file: figures/semantics/runtime-maskness.tex
\begin{figure*}[!t]
\scriptsize
\[
\begin{array}{c}

\irulelabel{
    \textbf{Intgral } \integrali
}{
    \venvs \vdash \integrali \todynmaskness \dynarraynotmasked
}{
    \textrm{(RTM-Lit)}
}
\quad

\irulelabel{
    \begin{array}{c}
        \textbf{Var } x \quad 
        \venv = \pop{\venvs} \\ 
        x \in \domain{\venv } \quad 
        \venv(x) = (\valuev, \maskmu)
    \end{array}
}{
    \venvs \vdash x \todynmaskness \maskmu
}{
    \textrm{(RTM-Var1)}
}
\quad

\irulelabel{
    \begin{array}{c}
        \textbf{Var } x \quad 
        \venv = \peek{\venvs} \\
        x \notin \domain{\venv} \quad 
        \venvs' = \pop{\venv} \quad
        \venvs' \vdash x \todynmaskness \maskmu
    \end{array}
}{
    \venvs \vdash x \todynmaskness \maskmu
}{
    \textrm{(RTM-Var2)}
}
\\ \ \\

\irulelabel{
    \begin{array}{c}
        \venvs \vdash \expre \todynmaskness \maskmu \quad 
        0 \leq k \leq n \\ 
        \venvs \vdash \expre_k \todynmaskness \dynarraynotmasked
    \end{array}
}{
    \venvs \vdash \arrindex{\expre}{\expre_0, \ldots, \expre_n} \todynmaskness  \maskmu
}{
    \textrm{(RTM-Index)}
}
\quad

\irulelabel{
    \begin{array}{c}
        \fops \in \text{Unary } \opops \cup \text{Binary } \opops \\
        \fops = \makemaskedarrayname \lor (\exists k. 1 \leq k \leq n \land \venvs \vdash \expre_k \todynmaskness \dynarraymasked)
    \end{array}
}{
    \venvs \vdash \op{\expre_1, \ldots, \expre_n} \todynmaskness \dynarraymasked
}{
    \textrm{(RTM-Op1)}
}

\irulelabel{
    \begin{array}{c}
        \fops \in \text{Unary } \opops \cup \text{Binary } \opops \\
        \fops \neq \makemaskedarrayname \\
        1 \leq k \leq n \quad 
        \venvs \vdash \expre_k \todynmaskness \dynarraynotmasked 
    \end{array}
}{
    \venvs \vdash \op{\expre_1, \ldots, \expre_n} \todynmaskness \dynarraynotmasked
}{
    \textrm{(RTM-Op2)}
}
\\ \ \\

\irulelabel{
    \begin{array}{c}
        \venvs \vdash \expre \todynmaskness  \maskmu
    \end{array}
}{
    \venvs \vdash \expand{\expre}{\codetuple{\overline{c}}} \todynmaskness  \maskmu
}{
    \textrm{(RTM-exp)}
}
\quad

\irulelabel{
    \begin{array}{c}
        \venvs \vdash \expre \todynmaskness  \maskmu
    \end{array}
}{
    \venvs \vdash \replicate{\expre}{\codetuple{\overline{c}}}{\codetuple{\overline{\integrali}}} \todynmaskness  \maskmu
}{
    \textrm{(RTM-Rep)}
}
\quad

\irulelabel{
    \venvs \vdash \expre \todynmaskness  \maskmu \quad 
    \fops \in \reduceops
}{
    \venvs \vdash \f{\expre,c} \todynmaskness  \maskmu
}{
    \textrm{(RTM-Rdce)}
}
\\ \ \\

\irulelabel{
    \begin{array}{c}
        \venvs \vdash \expre_1 \todynmaskness \dynarraynotmasked \quad 
        \venvs \vdash \expre_2 \todynmaskness \dynarraynotmasked
    \end{array}
}{
    \venvs \vdash \matmul{\expre_1}{\expre_2} \todynmaskness \dynarraynotmasked
}{
    \textrm{(RTM-Matmul)}
}
\quad

\irulelabel{
    \begin{array}{c}
        \venvs \vdash \expre_1 \todynmaskness  \maskmu \quad \venvs \vdash \expre_2 \todynmaskness \dynarraynotmasked
    \end{array}
}{
    \venvs \vdash \filled{\expre_1}{\expre_2} \todynmaskness \dynarraynotmasked
}{
    \textrm{(RTM-Fill)}
}

\irulelabel{
    \begin{array}{c}
        \expre = \ones{\codetuple{\overline{\integrali}}}
    \end{array}
}{
    \venvs \vdash \expre \todynmaskness \dynarraynotmasked
}{
    \textrm{(RTM-Ones)}
}

\irulelabel{
    \begin{array}{c}
        \expre = \arange{\integrali}
    \end{array}
}{
    \venvs \vdash \expre \todynmaskness \dynarraynotmasked
}{
    \textrm{(RTM-Aran)}
}
\end{array}
\]
\caption{The rules for evaluating maskedness of expressions at runtime.}
\label{fig:expr-runtime-maskedness}
\end{figure*}

%% file: figures/semantics/semantics.tex
\begin{figure*}[t!]
\scriptsize
\[
\begin{array}{c}

\irulelabel{
    
}{
    \venvs \vdash \cmdSkip \tovenvs \venvs
}{
    \textrm{(SEM-Skp)}
}
\quad

\irulelabel{
    \begin{array}{c}
        \venv = \peek{\venvs} \quad 
        \venvs \vdash \expre \tovalue \valuev \quad
        \venvs \vdash \expre \todynmaskness \maskmu \\ 
        \venv' = \venv[x \mapsto (\valuev, \maskmu)] \quad
        \venvs' = \push{\venv'}{\pop{\venvs}}
    \end{array}
}{
    \venvs \vdash \cmdAssign{x}{\expre} \tovenvs \venvs'
}{
    \textrm{(SEM-Bind)}
}
\quad

\irulelabel{
    \begin{array}{c}
        \venvs \vdash \expre \tovalue \valuev \quad 
        \venvs \vdash \expre \todynmaskness \dynarraynotmasked \quad
        \venvs' = \push{\emptyset}{\venvs} \\
        \venvs' \vdash \stmts_t \tovenvs \venvs''_t \quad 
        \venvs' \vdash \stmts_e \tovenvs \venvs''_e \\
        \venvs'' = \ite{v = 0}{\pop{\venvs''_e}}{\pop{\venvs''_t}}
    \end{array}
}{
    \venvs \vdash \cmdITE{\expre}{\stmts_t}{\stmts_e} \tovenvs \venvs''
}{
    \textrm{(SEM-If)}
}
\\ \ \\

\irulelabel{
    \begin{array}{c}
        \venvs \vdash \integrali \tovalue 0 
    \end{array}
}{
    \venvs \vdash \cmdFor{\name}{\codeliteral{0 \ldots} \integrali}{\stmts} \tovenvs \venvs
}{
    \textrm{(SEM-For0)}
}
\quad

\irulelabel{
    \begin{array}{c}
        \venvs \vdash \integrali \tovalue d \quad
        d \geq 1 \quad
        \venvs \vdash \cmdFor{\name}{\codeliteral{0 \ldots} d-1}{\stmts} \tovenvs \venvs' \\
        \venvs'' = \push{\{x \mapsto (d-1, \dynarraynotmasked) \}}{\venvs'} \quad
        \venvs'' \vdash \stmts \tovenvs \venvs''' 
    \end{array}
}{
    \venvs \vdash \cmdFor{\name}{\codeliteral{0 \ldots} \integrali}{\stmts} \tovenvs \pop{\venvs'''}
}{
    \textrm{(SEM-Fori)}
}
\quad

\irulelabel{
    \begin{array}{c}
        \venvs \vdash \stmts_1 \tovenvs \venvs' \\
         \venvs' \vdash \stmts_2 \tovenvs \venvs''
    \end{array}
}{
    \venvs \vdash \stmts_1 \codesemicolon \stmts_2  \tovenvs \venvs''
}{
    \textrm{(SEM-Seq)}
}
\\ \ \\

\irulelabel{
    \begin{array}{c}
        \venv = \peek{\venvs} \quad
        x \in \domain{\venv} \quad 
        \venv(x) = (\valuev', \maskmu) \\
        \venvs'' = \push{\venv[x \mapsto (\valuev,\maskmu)]}{\pop{\venvs}}
    \end{array}
}{
    \auxiupdate{\venvs}{x}{\valuev} = \venvs''
}{
    \textrm{(SEM-Auxupd1)}
}
\quad

\irulelabel{
    \begin{array}{c}
        \venv = \peek{\venvs} \quad
        x \notin \domain{\venv} \\ 
        \auxiupdate{\pop{\venvs}}{x}{\valuev} = \venvs'
    \end{array}
}{
    \auxiupdate{\venvs}{x}{\valuev} = \push{\venv}{\venvs'}
}{
    \textrm{(SEM-Auxupd2)}
}
\\ \ \\

\irulelabel{
    \begin{array}{c}
        0 \leq t \leq m+1 \quad
        \venvs \vdash x \tovalue \valuev \quad 
        \venvs \vdash \expre_t \tovalue v_t \quad
        \broadcast(\shape{\valuev_0}, \ldots, \shape{\valuev_{m+1}}) = (a_n, \ldots, a_0) \quad
        \shape{\valuev} = (b_m, \ldots, b_0) \\ 
        \shape{\valuev'} = (b_m, \ldots, b_0) \quad
        \broadcastto{\valuev_t}{(a_n, \ldots, a_0)} = v'_t \quad

        \forall i_n. \ldots \forall i_0. \bigwedge_{k=0}^n 0 \leq i_k \leq a_k \rightarrow \bigwedge_{j=0}^{m+1} \valuev'_{j}\auxiindex{i_n}\ldots\auxiindex{i_0} \neq \masked \vspace{0.5em}\\ 
        \begin{array}{l}
            \forall i_m. \ldots \forall i_0. \bigwedge^m_{j=0} 0 \leq i_j < b_j \rightarrow \\
            (
                (
                    v\auxiindex{i_m}\ldots\auxiindex{i_0} = \masked \lor (\forall k_n. \ldots k_0. \bigwedge_{l=0}^n 0 \leq k_l < a_l \rightarrow \bigvee^m_{p=0} \valuev'_p\auxiindex{k_n}\ldots\auxiindex{k_0} \neq i_p )
                ) \rightarrow 
                    \valuev'\auxiindex{i_m}\ldots\auxiindex{i_0} = \valuev\auxiindex{i_m}\ldots\auxiindex{i_0}
            ) \land \\
            (
                (
                    \valuev\auxiindex{i_m}\ldots\auxiindex{i_0} \neq \masked \land (\exists k_n. \ldots \exists k_0. \bigwedge_{l=0}^n 0 \leq k_l < a_l \bigwedge^m_{p=0} \valuev'_p\auxiindex{k_n}\ldots\auxiindex{k_0} = i_p )
                ) \rightarrow \\
                \valuev'\auxiindex{i_m}\ldots\auxiindex{i_0} = \valuev'_{m+1}\auxiindex{q_n}\ldots\auxiindex{q_0} \bigwedge_{l=0}^n 0 \leq q_l < a_l \bigwedge^m_{p=0} \valuev'_p\auxiindex{q_n}\ldots\auxiindex{q_0} = i_p 
            )
        \vspace{4pt}
        \end{array}
    \end{array}
}{
    \venvs \vdash \codeupdate{\arrindex{x}{e_m, \ldots, e_0}}{e_{m+1}} \tovenvs \auxiupdate{\venvs}{x}{\valuev'}
}{\textrm{(SEM-Upd)}}

\\ \ \\

\irulelabel{
    \begin{array}{c}
        0 \leq t \leq m + 1 \quad
        \venvs \vdash x \tovalue v \quad 
        \venvs \vdash \expre_t \tovalue \valuev_t \quad
        \broadcast(\shape{\valuev_0}, \ldots, \shape{\valuev_{m+1}}) = (a_n, \ldots, a_0) \quad

        \shape{\valuev} = (b_m, \ldots, b_0) \\  
        \shape{\valuev'} = (b_m, \ldots, b_0) \quad 

        \broadcastto{\valuev_t}{(a_n, \ldots,a_0)} = \valuev'_t \quad

        \forall i_n. \ldots \forall i_0. \bigwedge_{k=0}^n 0 \leq i_k \leq a_k \rightarrow \bigwedge_{j=0}^{m} \valuev'_{j}\auxiindex{i_n}\ldots\auxiindex{i_0}  \neq \masked \\

        \venvs \vdash \logicalor{\hat{\expre}_0}{\logicalor{\hat{\expre}_1}{...\logicalor{\hat{\expre}_{r-1}}{\hat{\expre}_r}}} \tovalue \hat{\valuev} \quad
        \broadcastto{\hat{\valuev}}{(a_n, \ldots, a_0)} = \hat{\valuev}' \vspace{0.5em}\\ 

        \begin{array}{l}
            \venvs \vdash \forall i_m. \ldots \forall i_0. \bigwedge^m_{j=0} 0 \leq i_j < b_j \rightarrow \\
            (
                (
                    \valuev \auxiindex{i_m}\ldots\auxiindex{i_0} = \masked \lor (\forall k_n. \ldots k_0. \bigwedge_{l=0}^n 0 \leq k_l < a_l \rightarrow 
                    \bigvee^m_{p=0} \valuev'_p\auxiindex{k_n}\ldots\auxiindex{k_0} \neq i_p \lor 
                    \hat{\valuev}'\auxiindex{k_n}\ldots\auxiindex{k_0} \in \{0, \masked \} \lor 
                    \valuev'_{m+1}\auxiindex{k_n}\ldots\auxiindex{k_0} = \masked)
                ) \rightarrow \\ 
                \valuev'\auxiindex{i_m}\ldots\auxiindex{i_0} = \valuev \auxiindex{i_m}\ldots\auxiindex{i_0}
            ) \land \\
            (
                (
                    \valuev \auxiindex{i_m}\ldots\auxiindex{i_0} \neq \masked \land (\exists k_n. \ldots \exists k_0. \bigwedge_{l=0}^n 0 \leq k_l < a_l 
                    \bigwedge^m_{p=0} \valuev'_p\auxiindex{k_n}\ldots\auxiindex{k_0} = i_p \land 
                    \hat{\valuev}'\auxiindex{k_n}\ldots\auxiindex{k_0} \notin \{0, \masked\} \land 
                    \valuev'_{m+1}\auxiindex{k_n}\ldots\auxiindex{k_0} \neq \masked)
                ) \rightarrow \\
            \valuev'\auxiindex{i_m}\ldots\auxiindex{i_0} = \valuev'_{m+1}\auxiindex{q_n}\ldots\auxiindex{q_0} 

            \bigwedge_{l=0}^n 0 \leq q_l < a_l \bigwedge^m_{p=0} \valuev'_p\auxiindex{q_n}\ldots\auxiindex{q_0} = i_p \land 
            \hat{\valuev}'\auxiindex{q_n} \ldots \auxiindex{q_0} \notin \{0, \masked\} \land 
            \valuev'_{m+1}\auxiindex{q_n}\ldots\auxiindex{q_0} \neq \masked) 
        \vspace{4pt}
        \end{array} 
    \end{array}
}{
    \venvs \vdash \codeupdate{\lhsmakemaskedarray{\ldots\lhsmakemaskedarray{\arrindex{x}{e_m, \ldots, e_0}}{\hat{e}_0}\ldots}{\hat{e}_r}}{ e_{m+1} } \tovenvs \auxiupdate{\venvs}{x}{\valuev'}
}{\textrm{(SEM-MskUpd)}}
\end{array}
\]
\caption{The semantics of the statements.}
\label{fig:stmt-semantics}
\end{figure*}

\input{figures/semantics/program-sem}

\begin{figure*}[]
\scriptsize
\[
\begin{array}{c}
\irulelabel{
    \begin{array}{c}
        \textbf{Integer literal } c
    \end{array}
}{
    \venvs \vdash c \tovalue c
}{
    \textrm{(RTV-Lit)}
}

\quad 

\irulelabel{
    \begin{array}{c}
        \textbf{Variable } x \quad 
        \venv = \peek{\venvs} \\
        x \in \domain{\venv} \quad 
        \venv(x) = (\valuev, \maskmu)
    \end{array}
}{
    \venvs \vdash x \tovalue \valuev
}{
    \textrm{(RTV-Var1)}
}
\quad

\irulelabel{
    \begin{array}{c}
        \textbf{Variable } x \quad 
        \venv = \peek{\venvs} \quad
        x \notin \domain{\venv} \\ 
        \venvs' = \pop{\venvs} \quad 
        \venvs' \vdash x \tovalue \valuev
    \end{array}
}{
    \venvs \vdash x \tovalue \valuev
}{
    \textrm{(RTV-Var2)}
}
\\ \ \\

\irulelabel{
    \begin{array}{c}
        \valuev \in \mathbb{R} \cup \{\masked\}
    \end{array}
}{
    \shape{\valuev} = ()
}{
    \textrm{(RTV-AuxSBase)}
}
\quad

\irulelabel{
    \begin{array}{c}
        \valuev = (\valuev_1, \ldots, \valuev_m) \quad 
        1 \leq j \leq m \\
        \shape{v_j} = (a_n,\ldots, a_0)
    \end{array}
}{
    \shape{\valuev} = (m, a_n, \ldots, a_0)
}{
    \textrm{(RTV-AuxSInd)}
}
\quad

\irulelabel{
    \begin{array}{c}
        \venvs \vdash \expre \tovalue \valuev \quad
        \shape{\valuev} = (a_0, \ldots, a_n) \\ 
        0 \leq c \leq n
    \end{array}
}{
    \venvs \vdash \codeshapeaccess{\expre}{c} \tovalue a_c
}{
    \textrm{(RTV-SAcc)}
}
\\ \ \\

\irulelabel{
    \begin{array}{c}
        \valuev = (\valuev_0, \ldots, \valuev_n) \quad
        n \geq 0 \\ 
        0 \leq c \leq n \quad 
        c \in \mathbb{Z}
    \end{array}
}{
    \valuev \auxiindex{c} = v_{c}
}{
    \textrm{(RTV-AuxIdx)}
}
\quad

\irulelabel{
    \begin{array}{c}
        \shape{\valuev} = (a_m, \ldots, a_0) \quad
        \broadcast((a_m, \ldots, a_0), (b_n, \ldots, b_0)) = (b_n, \ldots, b_0) \\
        \shape{\valuev'} = (b_n, \ldots, b_0) \quad 
        0 \leq j \leq m \quad
        c_j = \ite{a_j=1}{0}{1} \\
        \forall i_n. \ldots \forall i_0. 
                \bigwedge_{j=0}^n 0 \leq i_j < b_j
                \rightarrow 
                \valuev' \auxiindex{i_n} \ldots \auxiindex{i_0} = \valuev \auxiindex{c_m * i_m} \ldots \auxiindex{c_0 * i_0}
         \\
    \end{array}
}{
    \begin{array}{c}
         \broadcastto{\valuev}{(b_n, \ldots, b_0)} = \valuev'  \\
    \end{array}
}{
    \textrm{(RTV-AuxBT)}
}
\\ \ \\

\irulelabel{
    \begin{array}{c}
        0 \leq l \leq m \quad 
        \venvs \vdash \expre_l \tovalue \valuev_l \quad
        \broadcast(\shape{\valuev_0}, \ldots, \shape{\valuev_m}) = (a_n, \ldots, a_0) \quad
        \venvs \vdash \expre \tovalue \valuev \quad
        \shape{\valuev} = (b_m, \ldots, b_0) \\
        \shape{\valuev'} = (a_n, \ldots, a_0) \quad

        \broadcastto{\valuev_l}{(a_n, \ldots, a_0)} = \valuev'_l \quad

        \forall i_n. \ldots \forall i_0. \bigwedge_{k=0}^n 0 \leq i_k < a_k \rightarrow \bigwedge_{j=0}^m \valuev'_{j}\auxiindex{i_n}\ldots\auxiindex{i_0} \in \mathbb{N} \cup \{0\}\\

        \forall i_n. \ldots \forall i_0. \bigwedge_{j=0}^n 0 \leq i_j < a_j \rightarrow \valuev' \auxiindex{i_n} \ldots \auxiindex{i_0} = \valuev \auxiindex{\valuev'_m\auxiindex{i_n} \ldots \auxiindex{i_0}} \ldots \auxiindex{\valuev'_0\auxiindex{i_n} \ldots \auxiindex{i_0}}
    \end{array}
}{
    \venvs \vdash \arrindex{\expre}{\expre_m, \ldots, \expre_0} \tovalue \valuev'
}{\textrm{(RTV-Index)}}
\\  \ \\

\irulelabel{
    \begin{array}{c}
        \fops \in \text{Unary } \opops \cup \text{Binary } \opops \quad
        1 \leq l \leq m \quad 
        \venvs \vdash \expre_l \tovalue \valuev_l \quad 
        \broadcast(\shape{\valuev_1}, \ldots, \shape{\valuev_m}) = (a_n, \ldots, a_0) \quad
        \broadcastto{\valuev_l}{(a_n, \ldots, a_0)} = \valuev'_l \\ 
        \shape{v} = (a_n, \ldots, a_0) \quad
        \forall i_n. \ldots \forall i_0. 
                \bigwedge_{j=0}^n 0 \leq i_j < a_j \rightarrow
                    \valuev \auxiindex{i_n} \ldots \auxiindex{i_0} = \denot{\fops} (\valuev'_{1}\auxiindex{i_n} \ldots \auxiindex{i_0}, \ldots, \valuev'_{m}\auxiindex{i_n} \ldots \auxiindex{i_0})
        \\
    \end{array}
}{
    \venvs \vdash \f{\expre_1, ..., \expre_m} \tovalue \valuev
}{\textrm{(RTV-Op)}}
\\  \ \\

\irulelabel{
    \begin{array}{c}
        \fops \in \reduceops \quad 
        \venvs \vdash \expre \tovalue \valuev \quad 
        \shape{\valuev} = (a_0, \ldots, a_n) \quad
        0 \leq c \leq n \quad 
        \shape{\valuev'} = (a_0, \ldots, a_{c-1}, a_{c+1}, \ldots a_n) \vspace{0.5em} \\
        \begin{array}{l}
            \forall i_0. \ldots \forall i_{c-1}. \forall i_{c+1}. \ldots \forall i_n. 
                \bigwedge_{j\in(0\ldots c-1, c+1, \ldots, n)} 0 \leq i_j < a_j \rightarrow \\
                \valuev' \auxiindex{i_0}\ldots\auxiindex{i_{c-1}}\auxiindex{i_{c+1}}\ldots\auxiindex{i_n} = 
            \denot{\fops}(\valuev \auxiindex{i_0} \ldots \auxiindex{i_{c-1}}\auxiindex{0}\auxiindex{i_{c+1}} \ldots \auxiindex{i_n}, \ldots, \auxiindex{i_0} \ldots\auxiindex{i_{c-1}}\auxiindex{a_c-1}\auxiindex{i_{c+1}} \ldots \auxiindex{i_n} )
        \end{array}
    \end{array}
}{
    \venvs \vdash \f{\expre,c} \tovalue \valuev'
}{ \textrm{(RTV-Rdce)} }
\\ \ \\

\irulelabel{
    \begin{array}{c}
        \venvs \vdash \expre \tovalue \valuev \quad 
        \shape{\valuev} = (a_0, \ldots a_n) \quad
        \shape{\valuev'} = (b_0, \ldots, b_{n+1}) \quad
        0 \leq c \leq n + 1 \quad
        0 \leq i \leq n + 1 \quad
        b_i = \ite{i = c}{1}{\ite{i < c}{a_i}{a_{i-1}}} \vspace{4pt} \\
        \begin{array}{l}
            \forall k_0. \ldots \forall k_{n+1}. \bigwedge^{n+1}_{m=0} 0 \leq k_m < b_m \rightarrow 
            \forall l. j_l = \ite{l < c}{k_l}{k_{l+1}}\land \valuev'\auxiindex{k_0} \ldots \auxiindex{k_{n+1}} = \valuev\auxiindex{j_0} \ldots \auxiindex{j_n}   
        \end{array}
        
    \end{array}
}{
    \venvs \vdash \expand{\expre}{\codetupleone{c}} \tovalue \valuev'
}{
    \textrm{(RTV-ExpBase)}
}
\\ \ \\

\irulelabel{
    \begin{array}{c}
        \forall k. 0 < k \leq m \rightarrow c_k > c_{k-1} \\
        \venvs \vdash \expand{\expand{\expre}{\codetuple{c_0, \ldots, c_{m-1}}}}{\codetupleone{c_m}}\tovalue \valuev
    \end{array}
}{
    \venvs \vdash \expand{\expre}{\codetuple{c_0, \ldots, c_m}} \tovalue \valuev
}{
    \textrm{(RTV-ExpInd)}
}
\quad

\irulelabel{
    \begin{array}{c}
        0 \leq j \leq n \quad 
        \venvs \vdash \integrali_j \tovalue c_j \\
        \venvs \vdash \shape{\valuev} \tovalue (c_0, \ldots, c_n) \\
        \forall k_0. \ldots \forall k_n. \bigwedge_{l=0}^n 0 \leq k_l \leq c_l \rightarrow \valuev\auxiindex{k_0} \ldots \auxiindex{k_n} = 1
    \end{array}
}{
   \venvs \vdash \ones{\codetuple{\integrali_0, \ldots, \integrali_n}} \tovalue \valuev
}{
    \textrm{(RTV-Ones)}
}

\\ \ \\

\irulelabel{
    \begin{array}{c}
        \venvs \vdash \integrali \tovalue c \quad 
        \shape{v} = (c,) \quad 
        \forall j. 0 \leq j < c \rightarrow \valuev\auxiindex{j} = j 
    \end{array}
}{
    \venvs \vdash \arange{\integrali} \tovalue \valuev
}{
    \textrm{(RTV-Aran)}
}
\end{array}
\]
\caption{The rules for evaluating values of expressions at runtime.}
\label{fig:expr-semantics1}

\end{figure*}

\begin{figure*}[]
\scriptsize
\[
\begin{array}{c}
\irulelabel{
    \begin{array}{c}
        \venvs \vdash \expre \tovalue \valuev \quad
        0 \leq c \leq n + 1 \quad 
        \shape{\valuev} = (a_0, \ldots a_n) \quad
        \shape{\valuev'} = (b_0, \ldots, b_{n+1}) \\
        \venvs \vdash \integralj \tovalue d \quad
        0 \leq i \leq n + 1 \quad
        b_i = \ite{i = c}{d}{\ite{i < c}{a_i}{a_{i-1}}} \\
        \begin{array}{l}
            \forall k_0. \ldots \forall k_{n+1}. \bigwedge^{n+1}_{m=0} 0 \leq k_m < b_m \rightarrow 
            \forall l. j_l = \ite{l < c}{k_l}{k_{l+1}}\land \valuev'\auxiindex{k_0} \ldots \auxiindex{k_{n+1}} = \valuev\auxiindex{j_0} \ldots \auxiindex{j_n}   
        \end{array} 
    \end{array}
}{
    \venvs \vdash \replicate{e}{\codetupleone{c}}{\codetupleone{\integralj}} \tovalue \valuev'
}{
    \textrm{(RTV-RepBase)}
}
\\ \ \\

\irulelabel{
    \begin{array}{c}
        \forall k. 0 < k \leq m \rightarrow c_k > c_{k-1} \quad
        \venvs \vdash \replicate{\replicate{\expre}{\codetuple{c_0, \ldots, c_{m-1}}}{\codetuple{\integralj_0, \ldots, \integralj_{m-1}}}}{\codetupleone{c_m}}{\codetupleone{\integralj_m}} \tovalue \valuev
    \end{array}
}{
    \venvs \vdash \replicate{\expre}{\codetuple{c_0, \ldots, c_m}}{\codetuple{\integralj_0, \ldots, \integralj_m}}\tovalue \valuev
}{
    \textrm{(RTV-RepInd)}
}
\\  \ \\

\irulelabel{
    \begin{array}{c}
        \venvs \vdash \expre_1 \tovalue \valuev_1 \quad 
        \venvs \vdash \expre_2 \tovalue \valuev_2 \quad 
        \shape{\valuev_1} = (a_m, \ldots, a_0) \quad 
        \shape{\valuev_2} = (b_n, \ldots,b_2, a_0, b_0) \quad
        m \geq 1 \quad 
        n \geq 1 \quad
        \shape{\valuev} = (c_l, \ldots, c_0) \\
        c_1 = a_1 \quad
        c_0 = b_0 \quad
        \broadcast((a_m, \ldots, a_2), (b_n,\ldots, b_2)) = (c_l, \ldots c_2) \quad
        \broadcastto{\valuev_1}{(c_l, \ldots, c_2, a_1, a_0)} = \valuev'_1 \\ 
        \broadcastto{\valuev_2}{(c_l, \ldots, c_2, a_0, b_0)} = \valuev'_2 \quad
        \forall i_l. \ldots \forall i_0. \bigwedge_{j=0}^l 0 \leq i_j < c_j \rightarrow \valuev\auxiindex{i_l}\ldots\auxiindex{i_0} = \sum^{a_0 - 1}_{k=0} \valuev'_1\auxiindex{i_l}\ldots\auxiindex{i_1}\auxiindex{k} * \valuev'_2\auxiindex{i_l}\ldots\auxiindex{i_2}\auxiindex{k}\auxiindex{i_0}
    \end{array}
}{
    \venvs \vdash \matmul{\expre_1}{\expre_2} \tovalue \valuev
}{\textrm{(RTV-Matmul)}}
\\  \ \\

\irulelabel{
    \begin{array}{c}
        \venvs \vdash \expre_1 \tovalue \valuev_1 \quad
        \shape{\valuev_1} = (a_m, \ldots, a_0) \quad 
        \shape{\valuev} = (a_m, \ldots, a_0) \quad
         \venvs \vdash e_2 \tovalue c \quad 
        \shape{c} = () \quad c \neq \masked \vspace{0.5em}\\
         \begin{array}{l}
              \forall i_m. \ldots \forall i_0. \bigwedge_{j=0}^m 0 \leq i_j < a_j  \rightarrow ((\valuev_1\auxiindex{i_m}\ldots\auxiindex{i_0} = \masked \rightarrow \valuev\auxiindex{i_m}\ldots\auxiindex{i_0} = c) \land 
              (\valuev_1\auxiindex{i_m}\ldots\auxiindex{i_0} \neq \masked \rightarrow \valuev\auxiindex{i_m}\ldots\auxiindex{i_0} = \valuev_1\auxiindex{i_m}\ldots\auxiindex{i_0}))
         \end{array}
    \end{array}
}{
    \venvs \vdash \filled{\expre_1}{\expre_2} \tovalue \valuev
}{
    \textrm{(RTV-Fill)}
}
\end{array}
\]
\caption{The rules for evaluating values of expressions at runtime, cont.}
\label{fig:expr-semantics2}

\end{figure*}

%% file: figures/semantics/program-sem.tex
\begin{figure}[]
\scriptsize
\[
\begin{array}{c}
    \irulelabel{
    \begin{array}{c}
        0 \leq k \leq n \quad 
        \venvs \vdash x_k \tovalue \valuev_k \quad
        \venvs \vdash x_k \todynmaskness \maskmu_k \\ 
        \venvs \vdash \stmts \tovenvs \venvs' \quad 
        \venvs' \vdash y \tovalue \valuev \quad 
        \venvs' \vdash y \todynmaskness \maskmu
    \end{array}
}{
    \venvs \vdash \cmdProgram{x_0, \ldots, x_n}{\stmts}{y} \toprogramresult (\valuev, \maskmu)
}
{
    \textrm{(SEM-Prog)}
}
\end{array}
\]
\caption{The semantics of programs.}
\label{fig:prog-semantics}
\end{figure}

%% file: sections/type.tex
\section{Type Inference Rules}\label{sec:type-complete}
Figure~\ref{fig:shape-expr} shows the complete set of rules for static shape inference.
Figure~\ref{fig:mask-expr} shows the complete set of rules for static maskedness inference.

The type environment $\senv$ is a map from variable names to pairs of shape and maskedness.
Here, $\senv \vdash \expre \shapetypeof \shapes$ and $\senv \vdash \expre \masktypeof \maskm$ denote that $\expre$ is inferred to have a static shape $\shapes$ and a static maskedness $\maskm$.
We use the $\arraymasked$ symbol to denote masked arrays and use the $\arraynotmasked$ symbol to denote unmasked arrays.
We only support typing arrays with a fixed number of axes, though the dimensionality of each axis may be dynamic.
Therefore, static shapes are presented by tuples of symbols and integers.
We use the $\integraleval{\integrali}$ notation to represent evaluating an integral expression statically, which may return an integer or a symbol.

\input{figures/types/type-complete}

%% file: figures/types/type-complete.tex
\begin{figure*}[t!]
    \scriptsize
\[
\begin{array}{c}

\irulelabel{
    \textbf{Integral } \integrali
}{
    \shapetenv \vdash \integrali \shapetypeof ()
}{
    \textrm{(S-Lit)}
}

\irulelabel{
    \textbf{Var } x \quad \shapetenv(x) = (\shapes, m)
}{
    \shapetenv \vdash x \shapetypeof \shapes
}{
    \textrm{(S-Var)}
}

\irulelabel{
    \begin{array}{c}
        0 \leq i \leq n \quad \shapetenv \vdash \expre_i \shapetypeof \shapes_i \\
        \broadcast(\shapes_0, \ldots, \shapes_n) = \shapes
    \end{array}
}{
    \shapetenv \vdash \arrindex{\expre}{\expre_0, \ldots, \expre_n} \shapetypeof \shapes
}{
    \textrm{(S-Idx)}
}

\irulelabel{
    \begin{array}{c}
        1 \leq i \leq n \quad \shapetenv \vdash \expre_i \shapetypeof \shapes_i \\
        \broadcast(\shapes_1, \ldots, \shapes_n) = \shapes \quad
        \fops \in \text{Unary } \opops \cup \text{Binary } \opops
    \end{array}
}{
    \shapetenv \vdash \f{\expre_1, \ldots, \expre_n} \shapetypeof \shapes
}{
    \textrm{(S-Op)}
}
\\ \ \\

\irulelabel{
    \begin{array}{c}
        \shapetenv \vdash \expre \shapetypeof (a_0, \ldots, a_n) \quad 0 \leq c \leq n \quad
        \fops \in \reduceops
    \end{array}
}{
    \shapetenv \vdash \f{\expre, c} \shapetypeof (a_0, \ldots, a_{c-1}, a_{c+1}, \ldots, a_n)
}{
    \textrm{(S-Rdce)} 
}
\quad

\irulelabel{
    \begin{array}{c}
        \shapetenv \vdash \expre \shapetypeof (a_0, \ldots, a_n) \quad
        0 \leq c \leq n + 1 \quad \shapes = (b_0, \ldots, b_{n+1}) \\
        0 \leq j \leq n + 1 \quad
        b_j = \ite{j = c}{1}{\ite{j < c}{a_j}{a_{j-1}}}
    \end{array}
}{
    \shapetenv \vdash \expand{\expre}{\codetupleone{c}} \shapetypeof \shapes
}{
    \textrm{(S-ExpBase)}
}
\\ \ \\

\irulelabel{
    \begin{array}{c}
        \forall j. 0 < j \leq m \rightarrow c_j > c_{j-1} \\
        \shapetenv \vdash \expand{\expand{\expre}{\codetuple{c_0, \ldots, c_{m-1}}}}{\codetupleone{c_m}}\shapetypeof \shapes
    \end{array}
}{
    \shapetenv \vdash \expand{\expre}{\codetuple{c_0, \ldots, c_m}} \shapetypeof \shapes
}{
    \textrm{(S-ExpInd)}
}

\irulelabel{
    \begin{array}{c}
        \shapetenv \vdash \expand{\expre}{\codetupleone{c}} \shapetypeof (a_0, \ldots, a_n) \quad
        \integraleval{\integrali} = d \\
        0 \leq j \leq n \quad
         b_j = \ite{j = c}{d}{a_j}
    \end{array}
}{
    \shapetenv \vdash \replicate{\expre}{\codetupleone{c}}{\codetupleone{\integrali}}\shapetypeof (b_0, \ldots, b_n)
}{
    \textrm{(S-RepBase)}
}
\\ \ \\

\irulelabel{
    \begin{array}{c}
        \forall j. 0 < j \leq m \rightarrow c_j > c_{j-1} \\
        \shapetenv \vdash \replicate{\replicate{\expre}{\codetuple{c_0, \ldots, c_{m-1}}}{\codetuple{\integrali_0, \ldots, \integrali_{m-1}}}}{\codetupleone{c_m}}{\codetupleone{\integrali_m}} \shapetypeof \shapes
    \end{array}
}{
    \shapetenv \vdash \replicate{\expre}{\codetuple{c_0, \ldots, c_m}}{\codetuple{\integrali_0, \ldots, \integrali_m}} \shapetypeof \shapes
}{
    \textrm{(S-RepInd)}
}

\irulelabel{
    \begin{array}{c}
        \integraleval{\integrali} = c 
    \end{array}
}{
    \shapetenv \vdash \arange{\integrali} \shapetypeof (c,)
}{
    \textrm{(S-Aran)}
}
\\ \ \\

\irulelabel{
    \begin{array}{c}
        \shapetenv \vdash \expre_1 \shapetypeof (a_m, \ldots, a_0) \quad
        \shapetenv \vdash \expre_2 \shapetypeof (b_n, \ldots, b_2, a_0, b_0) \\
        m \geq 1 \quad n \geq 1 \\
        \broadcast((a_m, \ldots, a_2), (b_n, \ldots, b_2)) = (c_l, \ldots, c_2)
    \end{array}
}{
    \shapetenv \vdash \matmul{\expre_1}{\expre_2} \shapetypeof (c_l, \ldots, c_2,a_1,b_0)
}
{
    \textrm{(S-Matmul)}
}

\irulelabel{
    \begin{array}{c}
        \shapetenv \vdash \expre_1 \shapetypeof \shapes \quad \shapetenv \vdash \expre_2\shapetypeof()
    \end{array}
}{
    \shapetenv \vdash \filled{\expre_1}{\expre_2} \shapetypeof \shapes
}{
    \textrm{(S-Fill)}
}
\irulelabel{
    \begin{array}{c}
        0 \leq j \leq n \quad
        \integraleval{\integrali_j} = c_j
    \end{array}
}{
    \shapetenv \vdash \ones{\codetuple{\integrali_0, \ldots, \integrali_n}} \shapetypeof (c_0,\ldots, c_n)
}{
    \textrm{(S-Ones)}
}

\end{array}
\]
\caption{Rules for analyzing shapes of expressions. $\integraleval{\integrali}$ means symbolically evaluating the integral expression $\integrali$ under environment $\shapetenv$.}
\label{fig:shape-expr}
\end{figure*}

\begin{figure*}[t!]
    \scriptsize
\[
\begin{array}{c}

\irulelabel{
    \textbf{Intgral } \integrali
}{
    \masktenv \vdash \integrali \masktypeof \arraynotmasked
}{
    \textrm{(M-Lit)}
}

\irulelabel{
    \begin{array}{c}
        \textbf{Var } x \\ \masktenv(x) = (\shapes, \maskm)
    \end{array}
}{
    \masktenv \vdash x \masktypeof \maskm
}{
    \textrm{(M-Var)}
}

\irulelabel{
    \begin{array}{c}
        \masktenv \vdash \expre \masktypeof \maskm \quad 
        0 \leq k \leq n \\ 
        \masktenv \vdash \expre_k \masktypeof \arraynotmasked
    \end{array}
}{
    \masktenv \vdash \arrindex{\expre}{\expre_0, \ldots, \expre_n} \masktypeof \maskm
}{
    \textrm{(M-Idx)}
}

\irulelabel{
    \begin{array}{c}
        \fops \in \text{Unary } \opops \cup \text{Binary } \opops \\
        \fops = \makemaskedarrayname \lor (\exists k. 1 \leq k \leq n \land \masktenv \vdash \expre_k \masktypeof \arraymasked)
    \end{array}
}{
    \masktenv \vdash \f{\expre_1, \ldots, \expre_n} \masktypeof \arraymasked
}{
    \textrm{(M-Op1)}
}
\\ \ \\

\irulelabel{
    \begin{array}{c}
        \fops \in \text{Unary } \opops \cup \text{Binary } \opops \\
        \fops \neq \makemaskedarrayname \\
        \masktenv \vdash \expre_k \masktypeof \arraynotmasked \quad
        1 \leq k \leq n
    \end{array}
}{
    \masktenv \vdash \f{\expre_1, \ldots, \expre_n} \masktypeof \arraynotmasked
}{
    \textrm{(M-Op2)}
}

\irulelabel{
    \begin{array}{c}
        \masktenv \vdash \expre \masktypeof \maskm
    \end{array}
}{
    \masktenv \vdash \expand{\expre}{\codetuple{\overline{c}}} \masktypeof \maskm
}{
    \textrm{(M-exp)}
}

\irulelabel{
    \begin{array}{c}
        \masktenv \vdash \expre \masktypeof \maskm
    \end{array}
}{
    \masktenv \vdash \replicate{\expre}{\codetuple{\overline{c}}}{\codetuple{\overline{\integrali}}} \masktypeof \maskm
}{
    \textrm{(M-Rep)}
}

\irulelabel{
    \begin{array}{c}
    \end{array}
}{
    \masktenv \vdash \arange{\integrali} \masktypeof \arraynotmasked
}{
    \textrm{(M-Aran)}
}
\\ \ \\

\irulelabel{
    \masktenv \vdash \expre \masktypeof \maskm \quad 
    \fops \in \reduceops
}{
    \masktenv \vdash \f{\expre}{c} \masktypeof \maskm
}{
    \textrm{(M-Rdce)}
}

\irulelabel{
    \begin{array}{c}
        \masktenv \vdash \expre_1 \masktypeof \arraynotmasked \quad 
        \masktenv \vdash \expre_2 \masktypeof \arraynotmasked
    \end{array}
}{
    \masktenv \vdash \matmul{\expre_1}{\expre_2} \masktypeof \arraynotmasked
}{
    \textrm{(M-Matmul)}
}

\irulelabel{
    \begin{array}{c}
        \masktenv \vdash \expre_1 \masktypeof \maskm \quad \masktenv \vdash \expre_2 \masktypeof \arraynotmasked
    \end{array}
}{
    \masktenv \vdash \filled{\expre_1}{\expre_2} \masktypeof \arraynotmasked
}{
    \textrm{(M-Fill)}
}

\irulelabel{
    \begin{array}{c}
    \end{array}
}{
    \masktenv \vdash \ones{\codetuple{\overline{\integrali}}} \masktypeof \arraynotmasked
}{
    \textrm{(M-Ones)}
}
\end{array}
\]
\caption{Mask analysis for expressions. We use $\overline{\text{overlined}}$ symbols to represent comma-spliced sequences in code.}
\label{fig:mask-expr}
\end{figure*}

%% file: sections/type-proof.tex
\section{Proof of Soundness of Type Analysis} \label{sec:type-proof}
We want to prove Theorem~\ref{thm:type-soundness} shown below, which is a more formally restated version of Theorem~\ref{thm:type-soundness-shown}. 
We use $\shapeabs{\valuev}, \maskabs{\maskmu}$ to denote the application of abstraction functions for shapes and maskedness.
The abstraction function for shapes can be simply defined as the $\shape{}$ function in Figure~\ref{fig:expr-semantics1}.
The abstraction function for maskedness can be defined as $\maskabs{\maskmu} = \ite{\maskmu = \dynarraymasked}{\arraymasked}{\arraynotmasked}$.
\begin{theorem}[Soundness of type analysis]\label{thm:type-soundness}
Suppose that we have $\cmdProgram{x_0, \ldots, x_n}{\stmts}{y}$, which is a program, $\cann$ be a type annotation and $\venvs$ an evaluation for $x_0, \ldots, x_n$. 
Assume that $\forall \textbf{ Variable }x. \venvs \vdash x \tovalue \valuev \land \venvs \vdash x \todynmaskness \maskmu \rightarrow \cann \vdash x \shapetypeof \shapes \land \shapes = \shapeabs{\valuev} \land \cann \vdash x \masktypeof \maskm \land \maskm = \maskabs{\maskmu}$.
If $\venvs \vdash \cmdProgram{x_0, \ldots, x_n}{\stmts}{y} \toprogramresult (\valuev, \maskmu)$, then $\cann \vdash \stmts \totypeenv \senv \land \senv \vdash y \shapetypeof \shapes' \land \senv \vdash y \masktypeof \maskm'$ and $\shapeabs{\valuev} = \shapes' \land \maskabs{\maskmu} = \maskm'$.
\end{theorem}

First, we state the following lemmas and prove them.
These lemmas will eventually lead to the proof of Theorem~\ref{thm:type-soundness}.

\begin{lemma}[Soundness of shape analysis for expressions] \label{lem:shape-soundness}
    Given a program state $\venvs$, a type environment $\senv$, and an expression node $\expre$, $(\forall \textbf{ Variable }x. \venvs \vdash x \tovalue \valuev \rightarrow \senv \vdash x \shapetypeof \shapes \land \shapes = \alpha_s(\valuev)) \implies (\venvs \vdash \expre \tovalue \valuev' \rightarrow \senv \vdash \expre \shapetypeof \shapes' \land \shapes' = \alpha_s(\valuev'))$.
\end{lemma}

\begin{proof}
We prove this lemma by structural induction.

Base case (1): $\expre$ is an integral expression.
This case is typed by the rule (S-Lit).
There are two possible integral expressions: integer literals and shape access expressions. 
By the semantics of integer literals (RTV-Lit), integer literals evaluate to their literal value, which are scalars.
By the semantics of shape access expressions (RTV-SAcc), such expressions evaluate to a number in the shape tuples, which are scalars.
So integral expressions always have the scalar type.

Base case (2): $\expre$ is a variable.
$\senv$ stores the correct type of the variable by assumption.

Inductive case (1): $\expre$ is an indexing expression of the form $\arrindex{\expre}{\expre_0, \ldots, \expre_m}$.
The semantics of indexing expressions (RTV-Index) say the value to which $\expre$ evaluates has a runtime shape $(a_n, \ldots, a_0)$. 
This runtime shape is calculated by broadcasting the runtime shapes of the indexers together.
The (S-Idx) rule also types $\expre$ by broadcasting the static shapes of the indexers together.
If we assume all the static shapes of the indexers correctly reflect the runtime shapes of the indexers, then the static shape of $\expre$ correctly reflects the runtime shape of $\expre$.

Inductive case (2): $\expre$ is a routine call, which is in the form of $\f{\expre_1, \ldots, \expre_n}$.
Similar to the last inductive case, the semantics of routine calls (RTV-Op) say the runtime shape of the returned value is the result of broadcasting all the runtime shapes of the operands.
The static shape is calculated by broadcasting the static shapes of all the operands. 
This type inference rule is (S-Op).
So if the static shapes of operands correctly reflect the runtime shapes of operands, the static shape $\expre$ is also correct.

Inductive case (3): $\expre$ is a reduction call of the form $\f{\expre_t, c}$.
The semantics of reduction operators (RTV-Rdce) say the returned value has a runtime shape $(a_0, \ldots, a_{c-1}, a_{c+1}, \ldots, a_n)$ assuming the runtime shape of $\expre_t$'s evaluation has the shape $(a_0, \ldots, a_n)$. 
The static shape of $\expre$ is inferred in the same way based on the static shape of $\expre_t$.
This type inference rule is (S-Rdce).
If we assume the static shape of $\expre_t$ is correctly inferred, the static shape of $\expre$ is correctly inferred.

Inductive case (4): $\expre$ is call to $\expandname$ or $\replicatename$.
First, we discuss the base case, in which only one axis is added by the $\expandname$ call or the $\replicatename$ call.
The (RTV-ExpBase) rule and the (RTV-RepBase) rule show the semantics of the calls in this case.
The runtime shape of the returned value in this case has one more dimension added to the specific axis, and all the axes after the position at which the new dimension is added are shifted to one position to the right.
If the call is $\expandname$, the newly added dimension has a dimensionality of one, and if the call is $\replicatename$, the newly added dimension has a dimensionality of the specified number.
The static shape inference rules are (S-ExpBase) and (S-RepBase).
As the static shape inference is done in the same way on the static shape of the input argument, as long as the static shape of the input argument correctly reflects the runtime shape of the input argument, the static shape correctly reflects the runtime shape of the expression.
For situations where more than one dimension is to be expanded, the returned value is calculated by recursive calls with one less dimension to expand, as shown by rules (RTV-ExpInd) and (RTV-RepInd).
The static shapes are inferred in the same recursive fashion, which is shown by rules (S-ExpInd) and (S-RepInd).
As the base case correctly infers static shapes, any recursive calls should continue to correctly infer the static shapes.

Inductive case (5): $\expre$ is a call to $\matmulname$. 
That is, $\expre$ is of the form $\matmul{\expre_1}{\expre_2}$.
The semantics of $\matmulname$ (RTV-Matmul) say that the returned runtime value has the runtime shape $(c_l, \ldots, c_2, a_1, b_0)$ when the runtime shape of $\expre_1$'s evaluation is $(a_m, \ldots, a_0)$ and that of $\expre_2$ is $(b_n, \ldots, b_2, a_0, b_0)$. 
Here, $(c_l, \ldots, c_2)$ is the result of broadcasting $(a_m, \ldots, a_2)$ and $(b_n, \ldots, b_2)$.
The static shape inference is calculated similarly on the static shapes of $\expre_1$ and $\expre_2$.
So as long as the static shapes of the arguments are correctly reflecting the runtime shapes of the arguments, the static shape inferred for this expression correctly reflects the runtime shape of the expression.

Inductive case (6): $\expre$ is a call to $\arangename$ or $\onesname$.
The semantics of the two calls, (RTV-Aran) and (RTV-Ones), say that the integral expressions used as the arguments to the call are first evaluated, and the returned value has a runtime shape specified by the evaluations of the integral expressions.
The static typing rules (S-Aran) and (S-Ones) symbolically evaluate the integral expressions, and then use the symbols obtained from the symbolic evaluation to represent the static shape of the created array.
For integer literals, the symbolic evaluation always returns the literal value, and for shape access expressions, symbolic evaluations return symbolic dimensionalities on the specified axes of the shapes of the argument expressions.
If we assume the static shape of the argument to the shape access expression is correct, the symbolic evaluation of the shape access expression then correctly reflects its runtime value.
The shape that is inferred based on the symbolic evaluation is also correct then.

Inductive case (7): $\expre$ is a call to $\filledname$.
The semantics of $\filledname$ (RTV-Fill) say the runtime shape of the returned value is the same as the runtime shape of the first argument. 
It also requires the second argument to be a scalar.
The static typing rule for the shape of $\filledname$, (S-Fill), also infers the shape of the call by the shape of its first argument. 
It also checks the static shape of the second argument to ensure its static shape is a scalar.
So if the static shapes of the arguments are correctly inferred, the inferred static shape of $\expre$ correct reflects $\expre$'s runtime shape. 
\end{proof}

\begin{lemma}[Soundness of maskedness analysis for expressions]\label{lem:maskedness-soundness}
    Given a program state $\venvs$, a type environment $\senv$, and an expression node $\expre$, $(\forall \textbf{ Variable }x. \venvs \vdash x \todynmaskness \maskmu \rightarrow \senv \vdash x \masktypeof \maskm \land \maskm = \alpha_m(\maskmu)) \implies (\venvs \vdash \expre \todynmaskness \maskmu' \rightarrow \senv \vdash \expre \masktypeof \maskm' \land \alpha_m(\maskmu') = \maskm')$.
\end{lemma}

\begin{proof}
This lemma may be proven by structural induction.

Base case (1): $\expre$ is an integral expression.
Because integral expressions' runtime maskedness is always $\dynarraynotmasked$ by (RTM-Lit) and the (M-Lit) rule always types them as $\arraynotmasked$, the static type inference rule correctly infers the maskedness of $\expre$.

Base case (2): $\expre$ is an identifier for a variable.
The inferred type is correct by assumption.

Inductive cases: $\expre$ is one of the following types of expression: indexing expression, unary or binary routine call, call to $\expandname$ or $\replicatename$, call to $\matmulname$, call to a reduction operator, call to $\arangename$ or $\onesname$, call to $\filledname$.
Note that the static maskedness inference rules' syntax can be directly mapped to that of the rules for calculating runtime maskedness.
For all these expressions, if the arguments' static maskedness is correctly inferred, the inferred static maskedness of the expression correctly reflects the runtime maskedness of the expression. 
\end{proof}

\begin{lemma}[Soundness of type analysis for statements]\label{lem:type-soundness-stmts}
    Let $\venvs$ be a program state, $\senv$ be a type environment and $\stmts$ be a statement.
$
    (\forall \textbf{ Variable }x. \venvs \vdash x \tovalue \valuev \land \venvs \vdash x \todynmaskness \maskmu \rightarrow \senv \vdash x \shapetypeof \shapes \land \shapes = \shapeabs{\valuev} \land \senv \vdash x \masktypeof \maskm \land \maskabs{\maskmu}) 
    \implies 
    ((\venvs \vdash \stmts \tovenvs \venvs' \land \senv \vdash \stmts \totypeenv \senv') \implies (\forall \textbf{ Variable }y. \venvs' \vdash y \tovalue \valuev' \land \venvs' \vdash y \todynmaskness \maskmu' \rightarrow \senv' \vdash y \shapetypeof \shapes' \land \shapes' = \shapeabs{\valuev'} \land \senv' \vdash y \masktypeof \maskm' \land \maskm' = \maskabs{\maskmu'}))
$.
\end{lemma}

\begin{proof}
This can be proven by structural induction on the statements in the DSL.

Base case (1): $\stmts$ is a skip statement.
As the semantics of $\cmdSkip$ (SEM-Skp) say the evaluation environment is unchanged after the evaluation of this statement, the type inference rule (T-Skp) also does not change the type environment before typing this statement.
So the correctness of the type environment is maintained.

Base case (2): $\stmts$ defines a variable, or in other words, binds a value from the evaluation of an expression to a variable.
In this case, the semantics of the statement (SEM-Bind) say that the statement stores the runtime value and maskedness of the expression's evaluation to the top scope of the evaluation environment.
From Lemma~\ref{lem:shape-soundness} and Lemma~\ref{lem:maskedness-soundness}, we know that with $\senv$ we can correctly infer the static shape and maskedness of the expression on the right-hand side.
The (T-Bind) rule maps the variable on the left-hand side to the static shape and maskedness of the expression on the right-hand side.
Since all variables that can be evaluated under $\venvs$ before the execution of $\stmts$ have the correct types in $\senv$, and the newly added variable also has the correct type in $\senv$, $\senv$ continues to correctly type all the variables that can be evaluated under $\venvs$.

Base case (3): $\stmts$ is an update statement.
The rules stipulating the semantics of update statements, which are (SEM-Upd) and (SEM-MskUpd), use two auxiliary rules (SEM-Auxupd1) and (SEM-Auxupd2), which specify how the runtime evaluation environments are changed.
Note that the actual changes to the runtime evaluation environment are made by the (SEM-Auxupd1) rule, which does not change the runtime maskedness of the variable being updated.
The updated value to be bound to the variable is specified by the (SEM-Upd) rule and the (SEM-MskUpd) rule.
Both rules specify that the updated value has the same runtime shape as the original value.
So the update statements do not change the runtime shape and maskedness of variables.
The (T-Udp) rule also types update statements by directly returning the original type environment.
So the correctness of the type environment is maintained.

Inductive case (1): $\stmts$ is an if-else branch.
The semantics of such statements are given by (SEM-If).
Before the execution of a branch, an empty scope is pushed onto the evaluation environment. 
All the statements in the scope are evaluated under this evaluation environment.
The actual returned evaluation environment after the execution of the branch will have the top scope removed.
Because the (SEM-Bind) rule says all the newly defined variables are stored in the top scope of the evaluation environment, the evaluation environment after the branch is executed cannot evaluate any new variable that cannot be evaluated under $\venvs$.
All the variables that can be evaluated under $\venvs$ can still be evaluated under $\venvs'$, as no statement can explicitly remove variables from an evaluation environment, and no statement can remove a scope from the evaluation environment without adding one first.
Also, because no statement can change the runtime shape and maskedness of variables in an evaluation environment, all the variables that can be evaluated under $\venvs$ still have the same runtime shape and maskedness when evaluated under $\venvs'$.
The (T-If) rule types an if-else branch by first typing the statement in the then-branch and using the updated type environment to type the statement in the else-branch.
As no statement-level type inference rule can remove or update a mapping in the type environment, all the variables that can be typed by $\senv$ will be typed the same under $\senv'$.
Thus $\senv'$ still correctly types all the variables that can be evaluated under $\venvs'$.
We also note that since $\senv$ contains the correct types of all the variables that can be evaluated under the evaluation environment before the execution of the branch body, by the inductive hypothesis of the structural induction, $\senv'$ can also correctly type variables in the branch body.

Inductive case (2): $\stmts$ is a for-loop.
The rules (SEM-For0) and (SEM-Fori) specify the semantics of for-loops.
For-loops are executed in iterations. 
Before the execution of the loop body for each iteration, a new scope containing a mapping from the loop variable to the iteration count is pushed onto the evaluation environment.
The loop body is executed under the evaluation environment, and the top scope in evaluation is removed before the execution of the next iteration.
For the same reasons in the previous inductive case, $\venvs'$, the evaluation environment after the execution of the loop, can evaluate the same set of variables as $\venvs$ and $\venvs'$ cannot evaluate any new variable that cannot be evaluated under $\venvs$.
Also, for the reason discussed in the previous inductive case, all the variables that can be evaluated under $\venvs$ will have the same runtime shape and maskedness when evaluated under $\venvs'$.
Because the mappings in $\senv$ cannot be altered or removed, $\senv'$ still correctly type all the variables that can be evaluated under $\venvs'$.
We also note that before typing the loop body, $\senv$ is updated to map the loop variable to the type of unmasked scalar, which reflects the runtime shape and maskedness of the loop variable, as specified by (SEM-Fori).
So the updated $\senv$ can correctly type all the variables that can be evaluated under the evaluation environment before the execution of the loop body, and by the inductive hypothesis for the structural inductive, $\senv'$ also contains the correct types for the variables in the loop body.

Inductive case (3): $\stmts$ is a sequence of statements.
The semantics of sequences (SEM-Seq) say that in executions of a sequence, the first statement in the sequence is executed first, and then the updated evaluation environment is used to execute the second statement.
The (T-Seq) also first types the first statement in the sequence and then uses the updated type environment to type the second statement. 
By the inductive hypothesis of the structural induction, the updated type environment after typing the first statement in the sequence can correctly type all the variables that can be evaluated under the evaluation environment after the execution of the first statement.
Then the updated type environment after typing the whole sequence can correctly type all the variables that can be evaluated under the evaluation environment after the execution of the whole sequence. 
\end{proof}

Now, we back to the proof of Theorem~\ref{thm:type-soundness}.
\begin{proof}
By Lemma~\ref{lem:type-soundness-stmts}, $\senv$ can correctly type any variable that can be evaluated under the evaluation environment before the return of the program.
By the semantics of programs (SEM-Prog), a program can only return the value of a variable that can be evaluated under the evaluation environment before the return.
Thus the type of the returned variable inferred under $\senv$ must be correct. 
\end{proof}

%% file: sections/rewrite-appendix.tex
\section{Rewrite Rules}\label{sec:rewrite-complete}
Figure~\ref{fig:rewrite-stmt} and Figure~\ref{fig:rewrite-expr} show the complete set of rules for rewriting statements and expressions.
Some of the rules are already included in Figure~\ref{fig:rewrite-stmt-shown} and Figure~\ref{fig:rewrite-expr-shown} in Section~\ref{sec:rewrite}. 
The additional rewrite rules included in Figure~\ref{fig:rewrite-stmt} and Figure~\ref{fig:rewrite-expr} follow the informal descriptions of the rewrite procedure described in Section~\ref{sec:rewrite}. 

The (R-Bind2) rule states that nothing needs to be changed, and the type environment does not need to be updated if the right-hand side of a bind statement is not lifted after rewrite and the variable being bound does not depend on the current loop variable.
The (R-Bind3) rule explains how to do bind-site replication to eliminate pseudo-dependence.
The premise for such rewrites is that the right-hand side of a bind statement depends on the current loop variable but is not lifted after the rewrite.
Depending on whether the statement is in a branch, we need to explicitly replicate the defining expression either by calling $\replicatename$ or calling $\makemaskedarrayname$ and let the broadcasting mechanism handle the replication.
The (R-Updt) rule states how to rewrite update statements that are not rewritable reduce statements.
It first checks that the update statement is not a rewritable reduce statement and then rewrites the left-hand side and the right-hand side of the statement.
Because the right-hand side expression may be implicitly broadcast to the shape of the left-hand side expression, we need to ensure that if the right-hand side is lifted, the axes that were broadcast before the rewrite are still properly broadcast after the rewrite.
Thus, we may need to expand new axes on the right-hand side expression, depending on the shapes of the expressions before and after the rewrite.
This part of the rewrite is similar to the handling of binary operators.
The (R-Brch2) rule handles the cases where the branch condition is not lifted after the rewrite.
As discussed in Section~\ref{subsec:branches}, in this case, we simply rewrite the branch bodies of the then-branch and the else-branch and put them back into the corresponding branches, maintaining the branch structure.
Finally, the (R-Rstmt) rule specifies how to rewrite rewritable reduce statements formally.
We first check if an update statement is indeed a rewritable reduce statement.
Then we rewrite the expression on which the reduction is made.
Depending on whether the statement is in a branch and if the reduced expression is lifted after the rewrite, we handle the statement differently.
If the reduced expression is not lifted after the rewrite, we need to explicitly create a new axis on which the reduction is made.
Depending on whether the statement is in a branch, we may create this new axis by calling $\replicatename$ or calling $\makemaskedarrayname$.
Then we can just replace the reduced expression to call to the corresponding reduction operator on the reduced expression.
Also, as discussed in Section~\ref{subsec:branches}, we need to fill the reduced expression with the identity value of the reduction operator if the statement is in a branch to ensure that the result of the reduction is not a masked value.

The rule (R-Rdce) specifies how to rewrite $\reduceops$ operators.
Because such operators operate on a specified axis, we need to ensure that the reduction is made on the same axis after the rewrite.
So if the reduced expression is lifted, we need to increment the argument specifying the reduced axis by one.
The (R-Idxr) rule and the (Unmask) rule state how indexing expressions are rewritten.
Aligned with the informal descriptions in Section~\ref{sec:rewrite}, the (R-Idxr) rule states that if the indexers are broadcast before the rewrite and some of the broadcast indexers are lifted after the rewrite, we expand these indexers to ensure the broadcast is still properly applied on the right axes.
The (Unmask) rule specifies how to deal with indexers that become masked after the rewrite by extracting the mask from these masked indexers and wrapping the whole indexing expression in a call to $\makemaskedarrayname$.
The conjunction of the extracted masks is used as the mask to the $\makemaskedarrayname$ call.
The (R-Var) rule specifies how to handle two cases where the appearance of a variable needs to be replicated.
The first case is when a variable that is not loop-local escapes the restrictions from the branch condition when the branch is flattened.
The second case is when a loop-local variable that depends on the current loop variable is used in a branch deeper than the branch in which it is defined.
For both cases, we need to wrap the appearance of the variable in a $\makemaskedarrayname$ call to mask them by the mask corresponding to the current branch.
By doing so, we ensure values corresponding to the iterations where the variable's appearance is not evaluated are masked.
The (R-Uop) rule and the (R-Fill) rule are simple as the operators are unary element-wise operators. 
If the operand is lifted, the operators are still applied element-wise on the newly added axis, propagating the newly added axis.
The (R-Shape) rule and the (R-Rep) rule are similar to the (R-Exp) rule, incrementing the arguments specifying the target axes so the operators are applied to the same axes.
The (R-Int) rule and the (R-Init) rule state that the integer literals and the array initialization calls are never changed, as there is no argument that may be lifted by rewrites.
\input{figures/rewrite/rewrite-stmt-complete}
\input{figures/rewrite/rewrite-expr-complete}

%% file: figures/rewrite/rewrite-stmt-complete.tex
\begin{figure*}[]
    \scriptsize
\[
\begin{array}{c}
\irulelabel{
    \senvb, \senva, \{x \mapsto \arange{\integrali}\} \vdash \stmts \stmtrewriteto \stmts', \senva'
}{
    \senvb, \senva, \loopvarrep \vdash \cmdFor{x}{ \codeliteral{0 \ldots} \integrali}{\stmts} \stmtrewriteto \stmts', \senva'
}{
    \textrm{(R-For)}
}

\irulelabel{
    \begin{array}{c}
        \senvb \vdash \expre \shapetypeof \shapes \quad
        \senvba, \loopvarrep \vdash \expre \looprewriteto \expre' \\
        \senva \vdash \expre' \shapetypeof \shapes' \quad
        \senva \vdash \expre' \masktypeof \maskm' \quad 
        \shapes \neq \shapes' \\
        \senva' =  \senva[\name \mapsto (\shapes', \maskm')]
    \end{array}
}{ \senvba, \loopvarrep \vdash \cmdAssign{x}{\expre} \stmtrewriteto \cmdAssign{x}{\expre'}, \senva'}
{
    \textrm{(R-Bind1)}
}

\irulelabel{
    \begin{array}{c}
        \senvb \vdash \expre \shapetypeof \shapes \quad
        \senvba, \loopvarrep \vdash \expre \looprewriteto \expre' \\
        \senva \vdash \expre' \shapetypeof \shapes' \quad
        \shapes = \shapes' \\
        \curloopvar(\loopvarrep) = y \quad 
        y \notin \dependsOn(x)
    \end{array}
}{ \senvba, \loopvarrep \vdash \cmdAssign{x}{\expre} \stmtrewriteto \cmdAssign{x}{\expre}, \senva }
{
    \textrm{(R-Bind2)}
}
\\ \ \\

\irulelabel{
    \begin{array}{c}
        \senvba, \loopvarrep \vdash \stmts_1 \stmtrewriteto \stmts'_1, \senva' \\
        \senvb, \senva', \loopvarrep \vdash \stmts_2 \stmtrewriteto \stmts'_2, \senva'' 
    \end{array}
}{
    \senvba, \loopvarrep \vdash \stmts_1 \codesemicolon \stmts_2 \stmtrewriteto \stmts'_1 \codesemicolon \stmts'_2, \senva''
}{\textrm{(R-Seq)}}

\irulelabel{
    \begin{array}{c}
        \senvb \vdash \expre \shapetypeof \shapes \quad 
        \senvba, \loopvarrep \vdash \expre \looprewriteto \expre' \quad
        \senva \vdash \expre' \shapetypeof \shapes' \quad
        \shapes = \shapes' \quad 
        \senva \vdash \expre' \masktypeof \maskm' \\
        \shapes' = (a_n, \ldots, a_1) \quad
        \curloopvar(\loopvarrep) = y \quad 
        y \in \dependsOn(\name) \quad 
        \curloopvarreplacer(\loopvarrep) = \expre_s \\
        \senva \vdash \expre_s \shapetypeof (c,) \quad
        \senva \vdash \expre_s \masktypeof \maskm_s \quad 
        \expre'' = \ite{\maskm_s = \arraymasked}{\expre''_m}{\replicate{\expre'}{\codeliteral{(0,)}}{\codeliteral{(}c\codeliteral{,)}}} \\
        e''_m = \makemaskedarray{\expre'}{\expand{\logicalnot{\getmaskarray{\expre_s}}}{\codetuple{\codeliteral{1}, \ldots, n}}} \\
        \maskm'' = \ite{\maskm_s = \arraymasked}{\arraymasked}{\maskm'}
    \end{array}
}{ 
    \senvba,\loopvarrep \vdash \cmdAssign{\name}{\expre} \stmtrewriteto \cmdAssign{\name}{\expre''}, \senva[\expre'' \mapsto (c \cons \shapes, \maskm'')]
}{
    \textrm{(R-Bind3)}
}
\\ \ \\

\irulelabel{
    \begin{array}{c}
        \lnot \isrewritablerdce(\codeupdate{\expre_{l}}{\expre_{r}}) \quad 
        \senvba, \loopvarrep \vdash \expre_{l} \looprewriteto \expre'_{l} \quad
        \senvba, \loopvarrep \vdash \expre_{r} \looprewriteto \expre'_{r} \quad
        \senvb \vdash \expre_{l} \shapetypeof \shapes_{l} \quad
        \senvb \vdash \expre_{r} \shapetypeof \shapes_{r} \quad
        \shapes_l = (a_{m+n}, \ldots, a_0) \quad
        \shapes_r = (b_m, \ldots, b_0) \\
        \senva \vdash \expre'_{l} \shapetypeof \shapes'_{l} \quad
        \senva \vdash \expre'_{r} \shapetypeof \shapes'_{r} \quad
        \stmts' = \ite{\shapes_{r} \neq \shapes'_{r} \land n \geq 1}{\codeupdate{\expre'_{l}}{\expand{\expre'_{r}}{\codetuple{\codeliteral{1}, \ldots, n}}}}{\codeupdate{\expre'_{l}}{\expre'_{r}}}
    \end{array}
}{ 
    \senvba, \loopvarrep \vdash \codeupdate{\expre_{l}}{\expre_{r}} \stmtrewriteto \stmts', \senva
}{
    \textrm{(R-Updt)}
}
\\ \ \\

\irulelabel{
    \begin{array}{c}
        \senvba, \loopvarrep \vdash \expre_{c} \looprewriteto \expre'_{c} \quad 
        \senvb \vdash \expre_c \shapetypeof \shapes_c \quad
        \senva \vdash \expre'_c \shapetypeof \shapes'_c \quad 
        \shapes'_c \neq \shapes_c \quad
        \curloopvarreplacer(\loopvarrep) = \expre_s \quad
        \curloopvar(\loopvarrep) = y \\ 
        \senvba, \loopvarrep[y \mapsto \makemaskedarray{\expre_s}{\expre'_{c}}] \vdash \stmts_{t} \stmtrewriteto \stmts'_{t}, \senva' \quad
        \senvb, \senva', \loopvarrep[y \mapsto \makemaskedarray{\expre_s}{\logicalnot{\expre'_{c}}}] \vdash \stmts_{e} \stmtrewriteto \stmts'_{e}, \senva'' \quad 
    \end{array}
}{ 
    \begin{array}{c}
            \senvba, \loopvarrep \vdash \cmdITE{\expre_{c}}{\stmts_{t}}{\stmts_{e}} \stmtrewriteto \stmts'_t \codesemicolon \stmts'_e, \senva'' 
    \end{array}
}{
    \textrm{(R-Brch1)}
}
\\ \ \\

\irulelabel{
    \begin{array}{c}
        \senvba, \loopvarrep \vdash \expre_{c} \looprewriteto \expre'_{c} \quad 
        \senvb \vdash \expre_c \shapetypeof \shapes_c \quad
        \senva \vdash \expre'_c \shapetypeof \shapes'_c \quad 
        \shapes'_c = \shapes_c \quad
        \senvba, \loopvarrep \vdash \stmts_{t} \stmtrewriteto \stmts'_{t}, \senva' \quad 
        \senvb, \senva', \loopvarrep \vdash \stmts_{e} \stmtrewriteto \stmts'_{e}, \senva'' \quad 
    \end{array}
}{ 
    \begin{array}{c}
        \senvba, \loopvarrep \vdash \cmdITE{\expre_{c}}{\stmts_{t}}{\stmts_{e}} \stmtrewriteto \cmdITE{\expre'_c}{\stmts'_{t}}{\stmts'_{e}}, \senva''
    \end{array}
}{
    \textrm{(R-Brch2)}
}
\\ \ \\

\irulelabel{
    \begin{array}{c}
        % \senvba, \loopvarrep \vdash \expre_t \looprewriteto \expre'_t \quad 
        % \expre_t = \expre'_t \\
        \isrewritablerdce(\codeupdate{\expre_t}{\g{\expre_t, \expre_o}}) \quad 
        \senvb \vdash \expre_o \shapetypeof \shapes_o \quad
        \senvba, \loopvarrep \vdash \expre_o \looprewriteto \expre'_o \quad 
        \senva \vdash \expre'_o \shapetypeof \shapes'_o \quad
        \shapes'_o = (a_1, \ldots, a_n) \quad

        % op \in (\codeliteral{+, -, *, /}, \maximum, \minimum) \quad 
        \curloopvarreplacer(\loopvarrep) = \expre_s \quad 
        \senva \vdash \expre_s \masktypeof \maskm \\ 
        \expre''_o = \replicate{\expre'_o}{\codeliteral{(0,)}}{\codeliteral{(}\codeshapeaccess{\expre_s}{\codeliteral{0}} \codeliteral{,)}} \quad
        
        \expre_1 = \f{\expre'_o, \codeliteral{0}} \quad 
        \expre_2 = \f{\expre''_o, \codeliteral{0}} \quad 
        \expre_3 = \f{\filled{\expre'_o}{c_{id}}, \codeliteral{0}} \;
        \expre_4 = \f{\filled{\makemaskedarray{\expre''_o}{\expre_m}}{c_{id}}, \codeliteral{0}} \\
        e_m = \expand{\logicalnot{\getmaskarray{\expre_s}}}{\codetuple{1, \ldots, n}} \quad

        \expre_r = \ite{\maskm = \arraynotmasked}{\ite{\shapes_o \neq \shapes'_o}{\expre_1}{\expre_2}}{\ite{\shapes_o \neq \shapes'_o}{\expre_3}{\expre_4}} \quad 
        \stmts' = \codeupdate{\expre_t}{\g{\expre_t, \expre_r}}\\

        \fops = \{\codeliteral{+} \mapsto \codesum, \codeliteral{-} \mapsto \codesum, \codeliteral{*} \mapsto \codeproduct, \codeliteral{/} \mapsto \codeproduct, \maximum \mapsto \codemax, \minimum \mapsto \codemin\}(\gops) \\
        c_{id} = \{\codeliteral{+} \mapsto \codeliteral{0}, \codeliteral{-} \mapsto \codeliteral{0}, \codeliteral{*} \mapsto \codeliteral{1}, \codeliteral{/} \mapsto \codeliteral{1}, \maximum \mapsto \codeliteral{-\infty}, \minimum \mapsto \codeliteral{\infty}\}(\gops) \\
    \end{array}
}{ 
    \senvba, \loopvarrep \vdash \codeupdate{\expre_t}{\g{\expre_t, \expre_o}} \stmtrewriteto \stmts', \senva 
}{
    \textrm{(R-Rstmt)}
}
\end{array}
\]
\caption{Statement-level rewrite rules. $\curloopvar(\loopvarrep)$ returns the loop variable of the current loop. $\curloopvarreplacer(\loopvarrep)$ returns the expression that the current loop variable maps to in $\loopvarrep$. $\isrewritablerdce(\stmts)$ checks if the top-level operator on the right-hand side is one of $\codeliteral{+, -, *, /}, \maximum, \minimum$, the first operand of the right-hand side is the same as the update target, the variable being updated is not defined in the loop being rewritten, and the left-hand side does not depend on the current loop variable or any restricted loop variable.}
\label{fig:rewrite-stmt}
\end{figure*}

%% file: figures/rewrite/rewrite-expr-complete.tex
\begin{figure*}[]
    \scriptsize
\[
\begin{array}{c}
\irulelabel{ 
    \begin{array}{c}
        \textbf{Var } x \quad 
        \curloopvarreplacer(\loopvarrep) = \expre_s \quad
        \curloopvar(\loopvarrep) = x
    \end{array}
}{
    \senvba,\loopvarrep \vdash x \looprewriteto \expre_s
}{
    \textrm{(R-LVar)}
} 
\irulelabel{
    \begin{array}{c}
        \fops \in \reduceops \quad 
        \senvba, \loopvarrep \vdash \expre \looprewriteto \expre' \quad
        \senvb \vdash \expre \shapetypeof \shapes \\
        \senva \vdash \expre' \shapetypeof \shapes' \quad
        d = c + 1 \quad 
        c' = \ite{\shapes = \shapes'}{c}{d} \\
    \end{array}
}{
    \senvba, \loopvarrep \vdash \f{\expre, c} \looprewriteto \f{\expre',c'}
}{
    \textrm{(R-Rdce)}
}
\\ \ \\

\irulelabel{
    \begin{array}{c}
        \expre = \arrindex{\expre_t}{\expre_m, \ldots, \expre_1} \quad
        \senvb \vdash \expre_t \shapetypeof \shapes \quad
        \senvba, \loopvarrep \vdash \expre_t \looprewriteto \expre'_t \quad 
        \expre''_t = \ite{\notonlhs{\expre}}{\expre'_t}{\expre_t} \quad
        \senva \vdash \expre''_t \shapetypeof \shapes'' \quad
        \senvba,\loopvarrep \vdash \squarebrack{\expre_m, \ldots, \expre_1} \looprewriteto \squarebrack{ \expre'_m, \ldots, \expre'_1 } \\
        \curloopvarreplacer(\loopvarrep) = \expre_s \quad 
        1 \leq j \leq m \quad 
        \senvb \vdash \expre_j \shapetypeof (a_{n_j}^j, \ldots, a_0^j) \quad 
        \auximax(n_1, \ldots, n_m) = n' \\
        \expre' = \ite{\shapes'' = \shapes}{\arrindex{\expre''_t}{\expre'_m, \ldots, \expre'_1}}{\arrindex{\expre''_t}{\expand{\expre_s}{\codetuple{\codeliteral{1}, \ldots, n'}}, \expre'_m, \ldots, \expre'_1}} \quad
        \rmmaskedindexer(\expre', \senva) = \expre'' 
    \end{array}
}{
    \senvba, \loopvarrep \vdash \expre \looprewriteto \expre''
}{ \textrm{(R-Indx)}}
\\ \ \\

\irulelabel{
    \begin{array}{c}
        1 \leq j \leq m \quad 
        \senvb \vdash \expre_j \shapetypeof (a_{n_j}^j, \ldots, a_0^j) \quad 
        \auximax(n_1, \ldots, n_m) = n' \quad
        \senvb, \senva, \loopvarrep \vdash \expre_j \looprewriteto \expre'_j \\
        \senva \vdash \expre'_j \shapetypeof \shapes'_j \quad 
        \expre''_j = \ite{\shapes'_j \neq (a^j_{n_j}, \ldots, a^j_{0}) \land n' \neq n_j}{\expand{\expre'_j}{\codetuple{\codeliteral{1}, \ldots, n'-n_j}}}{\expre'_j}
    \end{array}
}{
    \senvba, \loopvarrep \vdash \squarebrack{ \expre_m, \ldots, \expre_1 }\looprewriteto \squarebrack{ \expre''_m, \ldots, \expre''_1 }
}{
    \textrm{(R-Idxr)}
}
\\ \ \\

\irulelabel{
    \begin{array}{c}
        \overline{\expre^\maskedindexer} = \maskedindexers \quad
        1 \leq j \leq m \quad
        \overline{\expre^\maskedindexer} \subseteq (\expre_m, \ldots, \expre_1) \land \expre_j \in \overline{\expre^\maskedindexer}\leftrightarrow \senva \vdash \expre_j \masktypeof \arraymasked \quad
        \expre'_j = \ite{\expre_j \notin \overline{\expre^\maskedindexer}}{\expre_j}{\filled{\expre_j}{\codeliteral{0}}} \quad
        1 \leq i \leq n \\
        \expre''_i = \logicalnot{\getmaskarray{\expre_i^\maskedindexer}} \quad
        \expre_{mask} = \logicaland{\expre''_1}{\logicaland{\ldots}{\expre''_n}\ldots} \\
        \expre' = \ite{n \geq 1}{\expre}{\makemaskedarray{\arrindex{\expre_t}{\expre'_m, \ldots, \expre'_1}}{\expre_{mask}}}
    \end{array}
}{
    \rmmaskedindexer(\arrindex{\expre_t}{\expre_m, \ldots, \expre_1}, \senva) = \expre'
}{
    \textrm{(Unmask)}
}
\\ \ \\

\irulelabel{
    \begin{array}{c}
        \textbf{Var } y \quad 
        \curloopvarreplacer(\loopvarrep) = \expre_s \quad
        \senvb \vdash y \shapetypeof (a_n, \ldots, a_1) \quad 
        \curloopvar(\loopvarrep) = x \quad 
        x \neq y \\
        \expre = \ite{\homeloopvar{y} \neq x \land \dependsOn(y) \cap \dependsOn(\expre_s) \neq \emptyset}{\expre''}{\expre'} \\ 
        \expre' = \ite{x \in \dependsOn(y) \land \definingscope(y) \neq \curscope{\loopvarrep}}{\expre'''}{y} \\
        \expre'' = \makemaskedarray{y}{\expand{\logicalnot{\getmaskarray{\expre_s}}}{\codetuple{\codeliteral{1}, \ldots, n}}} \\
        \expre''' = \makemaskedarray{y}{\expand{\logicalnot{\getmaskarray{\expre_s}}}{\codetuple{\codeliteral{1}, \ldots, n-1}}} 
    \end{array}
}{
    \senvba, \loopvarrep \vdash y \looprewriteto \expre
}{
    \textrm{(R-Var)}
}

\irulelabel{
    \begin{array}{c}
        \fops \in \text{Unary } \opops \quad
        \senvba, \loopvarrep \vdash \expre \looprewriteto \expre'
    \end{array}
}{
    \senvba, \loopvarrep \vdash \f{\expre} \looprewriteto \f{\expre'}
}{
    \textrm{(R-Uop)}
}
\\ \ \\

\irulelabel{
    \begin{array}{c}
        \fops \in \text{Binary } \opops \quad
        \senvb \vdash \expre_1 \shapetypeof (a_{m+n}, \ldots, a_0) \\
        \senvb \vdash \expre_2 \shapetypeof \shapes_2 \quad
        \shapes_2 = (b_{m}, \ldots, b_0) \\
        \senvba, \loopvarrep \vdash \expre_1 \looprewriteto \expre'_1 \quad
        \senvba, \loopvarrep \vdash \expre_2 \looprewriteto \expre'_2 \quad
        \senva \vdash \expre'_2 \shapetypeof \shapes'_2 \\
        \expre''_2 = \ite{\shapes'_2 \neq \shapes_2 \land n \geq 1}{\expand{\expre'_2}{\codetuple{\codeliteral{1}, \ldots, n}}}{\expre'_2}
    \end{array}
}{
    \senvba, \loopvarrep \vdash \f{\expre_1, \expre_2} \looprewriteto \f{\expre'_1, \expre''_2}
}{
    \textrm{(R-Biop)}
}

\irulelabel{
    \begin{array}{c}
        \expre = \expand{\expre_t}{\codetuple{c_0, \ldots, c_m}} \quad
        \senvb \vdash \expre_t \shapetypeof \shapes \\ 
        \senvba, \loopvarrep \vdash \expre_t \looprewriteto \expre'_t \quad
        \senva \vdash \expre'_t \shapetypeof \shapes' \quad
        0 \leq j \leq m \\
        d_j = \ite{\shapes' \neq \shapes}{c_j+1}{c_j} \quad
        \expre' = \expand{\expre'_t}{\codetuple{d_0, \ldots, d_m}}
    \end{array}
}{ 
    \senvba, \loopvarrep \vdash \expre \looprewriteto \expre'
}{
    \textrm{(R-Exp)}
}
\\ \ \\

\irulelabel{
    \begin{array}{c}
        \expre = \filled{\expre_t}{\expre_f} \quad
        \senvba, \loopvarrep \vdash \expre_t \looprewriteto \expre'_t \\
        \senvba, \loopvarrep \vdash \expre_f \looprewriteto \expre'_f \quad
        \senva \vdash \expre'_f \shapetypeof ()
    \end{array}
}{ 
    \senvba, \loopvarrep \vdash \expre \looprewriteto \filled{\expre'_t}{\expre'_f}
}{
    \textrm{(R-Fill)}
}

\irulelabel{
    \begin{array}{c}
        d = c + 1 \quad
        \senvb \vdash \expre_t \shapetypeof \shapes \\
        \senvba, \loopvarrep \vdash \expre_t \looprewriteto \expre'_t \quad
        \senva \vdash \expre'_t \shapetypeof \shapes' \\
        \expre' = \ite{\shapes \neq \shapes'}{\codeshapeaccess{\expre'_t}{d}}{\codeshapeaccess{\expre'_t}{c}}
    \end{array}
}{ 
    \senvba, \loopvarrep \vdash \codeshapeaccess{\expre_t}{c} \looprewriteto e'
}{
    \textrm{(R-Shape)}
}

\irulelabel{
    \begin{array}{c}
        e = \ones{\codetuple{\integrali_1, \ldots, \integrali_n}} \lor e = \arange{\integrali} \\
    \end{array}
}{ 
    \senvba, \loopvarrep \vdash e \looprewriteto e
}{
    \textrm{(R-Init)}
}
\\ \ \\

\irulelabel{
    \begin{array}{c}
        \textbf{Integer Literal  } c \\
    \end{array}
}{ 
    \senvba, \loopvarrep \vdash c \looprewriteto c
}{
    \textrm{(R-Int)}
}

\irulelabel{
    \begin{array}{c}
        \expre = \replicate{\expre_t}{\codetuple{c_0, \ldots, c_m}}{\codetuple{\overline{\integrali}}} \quad
        \senvb \vdash \expre_t \shapetypeof \shapes \quad
        \senvba, \loopvarrep \vdash \expre_t \looprewriteto \expre'_t \\
        0 \leq j \leq m \quad
        \senva \vdash \expre'_t \shapetypeof \shapes' \quad
        d_j = \ite{\shapes' \neq \shapes}{c_j+1}{c_j} \quad
        \expre' = \replicate{\expre'_t}{\codetuple{d_0, \ldots, d_m}}{\codetuple{\overline{\integrali}}}
    \end{array}
}{ 
    \senvba, \loopvarrep \vdash \expre \looprewriteto \expre'
}{
    \textrm{(R-Rep)}
}
\end{array}
\]
\caption{Rewrite rules for expressions. \textrm{R-Biop} applies to $\matmulname$ if operands are not masked after rewrites. Integer literals, $\onesname$, and $\arangename$ are never changed in rewrites. The rewrite rules for $\replicatename$ and shape access are similar to that of $\expandname$. These rules are omitted here. $\curloopvar(\loopvarrep)$ and $\curloopvarreplacer(\loopvarrep)$ means getting the current loop variable and its substitute from $\loopvarrep$. $\homeloopvar{x}$ returns the loop variable of the loop in which $x$ is defined. $\onlhs{\expre}$ checks if $\expre$ is on the left-hand side of an update. $\definingscope(y) \neq \curscope{\loopvarrep}$ checks that the current branch level we are trying to flatten is not the same branch level defining $y$.}
\label{fig:rewrite-expr}
\end{figure*}

%% file: sections/rewrite-proof.tex
\section{Proof of Correctness of Rewrite Rules}\label{sec:rewrite-proof}
We can state the soundness of loop-level rewrites as a lemma.

\begin{lemma}[Soundness of loop rewrites] \label{lem:loop-rewrite}
Suppose that we have rewritable loop $\cloop = \cmdFor{x}{0 \ldots \integrali}{\stmts}$ in a program $\cprog$ and $\stmts$ does not contain any loop.
Let $\senv$ be the type environment inferred from $\cprog$.
Let $\venvs$ be an evaluation environment.
If $\senv, \senv, \emptyset \vdash \cloop \stmtrewriteto \stmts', \senva'$ and $\venvs \vdash \cloop \tovenvs \venvs'$, then $\venvs \vdash \stmts' \tovenvs \parvenvs$ and $\forall \textbf{ Variable }y. \venvs' \vdash y \tovalue \valuev \rightarrow \parvenvs \vdash y \tovalue \valuev$.
\end{lemma}

We want to prove Lemma~\ref{lem:loop-rewrite}.
The proof has two parts. 
First, we want to prove that for a loop with no loop-carried dependence at all, the assertion on the rewrite result is true. 
Then we prove that the rule for rewriting rewriteable reduce statements is valid while maintaining the correctness of other rewrite rules.

To prove the first part, we cite a theorem proven by \cite{kennedy2001optimizing} stating that ``it is valid to convert a sequential loop to a parallel loop if the loop carries no dependence.'' 
If the loop can be executed in parallel, then clearly each statement in the loop can be executed in parallel.
We only need to prove that the rewrite rules transform each statement into a new statement that is equivalent to executing the original statement at each iteration of the loop in parallel.
In the context of our DSL, we define ``the parallel evaluation of $\expre$'' and ``the parallel execution of $\stmts$'' as follows.

\begin{definition}[Parallel evaluation of an expression]
    Assume that $\cloop = \cmdFor{x}{0 \ldots \integrali}{\stmts}$ is the innermost loop in a program $\cprog$. 
    Let $\expre$ be an expression node in $\stmts$.
    Let $\venvs$ be an evaluation environment such that $\venvs \vdash \cloop \tovenvs \venvs'$ and $\venvs \vdash \integrali \tovalue c$.

    Let $\overline{\integralj} \subseteq \{0, \ldots, c-1\}$ be a set of integers such that during an iteration that $x$ evaluates to $n, n \in \overline{\integralj}$, $\expre$ is evaluated.
    Suppose $n \in \overline{\integralj}$ and let $\venvs_n$ be the updated $\venvs$ used to evaluate $\expre$ during the iteration when $x$ evaluates to $n$ such that $\venvs_n \vdash \expre \tovalue \valuev_n$.

    The parallel evaluation of $\expre$ with respect to $\venvs, \cloop$ is defined differently depending on whether $\expre$ is an identifier for a variable defined outside of $\cloop$.

    If $\expre$ is not an identifier for a variable defined outside of $\cloop$, then the parallel evaluation of $\expre$ with respect to $\venvs, \cloop$ is either $\valuev_n$ for any $n \in \overline{\integralj}$, or the array $(\valuev'_0, \ldots, \valuev'_{c-1})$.
    Here, $\valuev'_k = \valuev_k$ if $k \in \overline{j}$ and $\valuev'_k = \valuev^\masked_k$, an array of the shape of $\expre$ filled with $\masked$, otherwise. 
    If the parallel evaluation of $\expre$ is $\valuev_n$, then $\forall a, b. a, b \in \overline{j} \rightarrow \valuev_a = \valuev_b$. 
    If $\expre$ is an identifer for a variable depending on the loop variable of $\cloop$, the parallel evaluation of $\expre$ must be of the form $(\valuev'_0, \ldots, \valuev'_{c-1})$.

    If $\expre$ is an identifier for a variable defined outside of $\cloop$, then the parallel evaluation of $\expre$ with respect to $\venvs, \cloop$ is the result of merging the updates on the variable in each iteration from the statements before $\expre$.
    That is, suppose in iteration $n$, the statements before $\expre$ in the loop body update $\expre$'s array cells $\overline{k}_n$ to values $\overline{\valuev}_n$.
    Suppose $\venvs \vdash \expre \tovalue \valuev_\expre$.
    The parallel evaluation of $\expre$ with respect to $\venvs, \cloop$ is then $\valuev_\expre[\overline{k}_0 \mapsto \overline{\valuev}_0] \ldots [\overline{k}_{c-1} \mapsto \overline{\valuev}_{c-1}]$.
\end{definition}

\begin{definition}[Parallel execution of a loop]
    Suppse that $\cloop = \cmdFor{x}{0 \ldots \integrali}{\stmts}$ is the innermost loop in a program $\cprog$. 
    Let $\venvs$ be an evaluation environment such that $\venvs \vdash \cloop \tovenvs \venvs'$ and $\venvs \vdash \integrali \tovalue c$.
    
    The parallel execution of $\stmts$ with respect to $\venvs, \cloop$ results in $\parvenvs$, which has the following properties.
    \begin{itemize}
        \item If $\stmts$ is a bind statement of the form $\cmdAssign{y}{e}$, then we have $\peek{\parvenvs}(y) = \valuev$, where $\valuev$ is the parallel evaluation of $\expre$ with respect to $\venvs, \cloop$.
        If $y$ depends on the loop variable of $\cloop$, $\valuev$ must be an array of values to which $\expre$ evaluates in each iteration.
        \item If $\stmts$ is an update statement of the form $\codeupdate{\expre_l}{\expre_r}$, which updates array cells of a variable $y$ selected by $\expre_l$ with values specified by $\expre_r$, then $\parvenvs \vdash y \tovalue \valuev$ such that $\valuev$ has all the array cells selected by $\expre_l$ in each iteration updated to the values that are from the evaluation of $\expre_r$ in each iteration.
        \item If $\stmts$ is a sequence of the form $\stmts_1 \codesemicolon \stmts_2$, then $\parvenvs$ should be the result of parallel execution of $\stmts_2$ with respect to $\venvs^{**}$ and loop $\cmdFor{x}{0 \ldots \integrali}{\stmts_2}$. Here, $\venvs^{**}$ is the result of the parallel execution of $\stmts_1$ with respect to $\venvs, \cmdFor{x}{0 \ldots \integrali}{\stmts_1}$.
        \item If $\stmts$ is a branch of the form $\cmdITE{\expre}{\stmts_t}{\stmts_e}$, then the parallel execution of $\stmts$ with respect to $\venvs, \cloop$ is either
        \begin{itemize}
            \item the parallel execution of $\stmts_t \codesemicolon \stmts_e$ with respect to $\venvs, \cloop$, or 
            \item $\pop{\venvs^{***}}$. Here, the evaluation environment $\venvs^{***}$ is the result of parallel execution of $\stmts_e$ with respect to $\push{\emptyset}{\pop{\venvs^{**}}}$ and $\cmdFor{x}{0 \ldots \integrali}{\stmts_e}$ and $\venvs^{**}$ is the parallel execution result of $\stmts_t$ with respect to $\push{\emptyset}{\venvs}$ and $\cmdFor{x}{0 \ldots \integrali}{\stmts_t}$.
        \end{itemize}
    \end{itemize}
    
\end{definition}

Given the two definitions above, we can more formally restate the theorem by\cite{kennedy2001optimizing} with respect to our DSL.

\begin{theorem}\label{thm:allen}
    Let $\cloop = \cmdFor{x}{0 \ldots \integrali}{\stmts}$ be the innermost loop in a program $\cprog$. 
    Let $\venvs$ be an evaluation environment such that $\venvs \vdash \cloop \tovenvs \venvs'$. 
    $\cloop$ does not have any loop-carried dependence when executed under $\venvs$.
    Then, the parallel execution of $\stmts$ with respect to $\venvs$ and $\cloop$ results in $\parvenvs$ such that for any variable $y$, $\venvs' \vdash y \tovalue \valuev \rightarrow \parvenvs \vdash y \tovalue \valuev$.
    % More precisesly, thie $\parvenvs$ is the updated $\venvs$ after evaluating $\stmts$ but with the frame defining $y$ yet popped out. 
    % The location at the end of the loop body is at a branch scope in which $y$ is defined.
\end{theorem}

Now we just need to prove the following two lemmas.
\begin{lemma}\label{lem:expr-rewrite}
    Let $\cloop = \cmdFor{x}{0 \ldots \integrali}{\stmts}$ be the innermost loop in a program $\cprog$ whose type environment is $\senv$. 
    Let $\expre$ be an expression node in $\stmts$.
    Let $\venvs$ be an evaluation environment such that $\venvs \vdash \cloop \tovenvs \venvs'$ and $\venvs \vdash \integrali \tovalue c$.
    $\cloop$ does not have any loop-carried dependence when executed under $\venvs$.
    Assume that $\senv, \senv, \emptyset \vdash \cloop \stmtrewriteto \stmts'$ entails $\senvba, \loopvarrep \vdash \expre \looprewriteto \expre'$ and that $\parvenvs$ is an evaluation environment that evaluates each the variable appearing in $\expre$ to its parallel evaluation with respect to $\venvs, \cloop$. 
    Then, $\parvenvs \vdash \expre' \tovalue \valuev$ such that $\valuev$ is the parallel evaluation of $\expre$ with respect to $\venvs, \cloop$.
\end{lemma}

\begin{lemma}\label{lem:stmt-rewrite}
    Let $\cloop = \cmdFor{x}{0 \ldots \integrali}{\stmts}$ be the innermost loop in a program $\cprog$ whose type environment is $\senv$. 
    Let $\venvs$ be an evaluation environment such that $\venvs \vdash \cloop \tovenvs \venvs'$ and $\venvs \vdash \integrali \tovalue c$.
    $\cloop$ does not have any loop-carried dependence when executed under $\venvs$.
    If $\senv, \senv, \emptyset \vdash \cloop \stmtrewriteto \stmts', \senv'$, then $\venvs \vdash \stmts' \tovenvs \parvenvs$ such that $\parvenvs$ is the result of the parallel execution of $\stmts$ with respect $\venvs, \cloop$.

    For any variable $y$ that can be evaluated under $\parvenvs$ to a value $\valuev$, $\senv'$ correctly maps $y$ to the type of $\valuev$.
\end{lemma}

Lemma~\ref{lem:expr-rewrite} is proven by structural induction.
\begin{proof}
Base case (1): $\expre$ is an integer literal. 
Then $\expre$ evaluates to the same value always. 
The rule (R-Int) returns the same $\expre$, which always evaluates to the same value.
This value is the parallel evaluation of $\expre$.

Base case (2): $\expre$ is an identifer of a variable.
There are four sub-cases then. 

Base case (2.1): The first one is that $\expre$ is the loop variable of $\cloop$, which is handled by the rule (R-LVar).
In this case $\expre$ is rewritten to the loop variable substitute currently recorded in $\loopvarrep$.
There are only two rules changing the loop variable substitute stored in $\loopvarrep$, which are (R-Brch1) and (R-For).
The rule (R-For) sets the substitute to a call to $\arangename$ which is evaluated to an array of all the values the loop variable can take in each iteration. 
The rule (R-Brch1) sets the substitute to a call to $\makemaskedarrayname$, masking the previous loop variable substitute with the condition of the branch being flattened.
In this case the loop variable substitute will be evaluated to an array of all the values that the loop variable takes in each iteration, but the positions corresponding to iterations not evaluating $\expre$ are evaluated to a $\masked$ value.
Under $\parvenvs$, both possibilities of the loop variable substitute will be evaluated to the parallel evaluation of $\expre$.

Base case (2.2): The second sub-case is that $\expre$ is a variable defined outside of $\cloop$ and the intersection of the loop variables on which $\expre$ depends and the loop variables on which the current loop variable substitute stored in $\loopvarrep$ depends is non-empty. 
That is, the branch of $\expre$ puts restrictions on at least one of the loop variables on which $\expre$ depends. 
The (R-Var) rule handles this case. 
To rewrite $\expre$, it masks $\expre$ by the mask extracted from the current loop variable substitute. 
Because the loop variable substitute is always a 1-D array, its mask is also a 1-D array.
Being a variable defined outside of $\cloop$, $\expre$ does not depend on the loop variable of $\cloop$ and always evaluates to the same value in each iteration of $\cloop$.
The evaluation of $\expre$ under $\parvenvs$ is just this value.
When we expand dimensions of the extracted mask on the axes $(1, \ldots, n)$, where $n$ is the number of dimensions of $\expre$ originally, and use it to mask $\expre$, the broadcasting semantics duplicates the value to which $\expre$ evaluates under $\parvenvs$ on the left-most axis. 
So under $\parvenvs$, the rewrite result evaluates to an array of values that $\expre$ evaluates to in each iteration. 
The locations corresponding to iterations for which $\expre$ is not evaluated are fully masked.
So under $\parvenvs$, $\expre'$ evaluates to the parallel evaluation of $\expre$ with respect to $\venvs, \cloop$.

Base case (2.2): The third sub-case is that $\expre$ is a variable defined in $\cloop$ and depends on the loop variable of $\cloop$. 
However, the occurrence of $\expre$ is in a branch that is deeper than the scope in which it is defined.
This case is also handled by (R-Var).
Because $\expre$ is defined in $\cloop$ and depends on the loop variable of $\cloop$, $\parvenvs$ evaluates $\expre$ to an array of values that $\expre$ evaluates to in each iteration.
To mask elements in this array that correspond to the iterations in which $\expre$ is not evaluated, we do similarly as in the case (2.3). 
In this case, the evaluation of $\expre$ under $\parvenvs$ already has one more dimension on the left compared to its evaluation in each iteration so we just need to expand the dimension of the extracted mask on axes $(1, \ldots, n-1)$. 
Now the rewrite result is an array that $\expre$ evaluates to in each iteration and the elements corresponding to the iterations not evaluating $\expre$ are masked. 
This is the parallel evaluation of $\expre$ with respect to $\venvs, \cloop$.

Base case (2.4): The last sub-case is all the cases not handled by the previous sub-cases. 
Rule (R-Var) handles this sub-case by not making any changes in the rewrite result.
There are three possibilities of what happens in this sub-case.
The first one is that $\expre$ is a variable defined outside of $\cloop$ and is not restricted by the condition of the current branch.
The second possibility is that $\expre$ is defined in $\cloop$ but it does not depend on the loop variable or any restricted loop variable of an outer loop.
The third possibility is that $\expre$ is defined in $\cloop$, depends on the loop variable and the definition of $\expre$ is in the same branch level as $\expre$. 
Because it is in the same branch level as its defining statement, all the entries corresponding to unevaluated iterations must have been correctly masked.
So in all three possibilities, $\expre$ does not need to be changed in the rewrite process and evaluates to its parallel evaluation under $\parvenvs$ by assumption.

Inductive case (1): $\expre$ is a shape access expression of the form $\codeshapeaccess{\expre_t}{c}$.
The (R-Shape) rule handles this case.
Under $\parvenvs$, the result of rewriting $\expre_t$ will evaluate to the parallel evaluation of $\expre_t$, which may be of the same shape or have one more dimension on the left-most axis, by the definition of parallel evaluation. 
If it has the same shape, then the rule does not change anything to rewrite the expression as the parallel evaluation of the expression should read the length of the same axis as before.
If it does not have the same shape, which means one more dimension is added on the left-most axis, the index of the axis whose length we are trying to read should be one greater than the index before the rewrite.
In either case, under $\parvenvs$, the evaluation of the rewritten expression always equals the same value as we are reading in each iteration of the loop.
This is the parallel evaluation of a shape access expression with respect to $\venvs, \cloop$.

Inductive case (2): $\expre$ is an array creating expression in the form of $\ones{\integrali_1, \ldots, \integrali_n}$ or $\arange{\integrali}$.
Note that an integral expression always evaluates to the same value in each iteration of $\cloop$ and the rewrite rules always rewrite them to an expression that evaluates to this value under $\parvenvs$.
The (R-Init) rule rewrites array creating routines without changing anything.
So the rewrite result evaluates to the same value under $\parvenvs$ as $\expre$ does in each iteration. 
This is a parallel evaluation result of $\expre$ with respect to $\venvs, \cloop$.

Inductive case (3): $\expre$ is a unary routine.
The (R-Uop) rule handles this case. 
Because unary operations apply the function to each scalar element in the argument, when the rewritten argument evaluates to the parallel evaluation of the original argument with respect to $\venvs, \cloop$ under $\parvenvs$, applying the function to each scalar element naturally results in the parallel evaluation of $\expre$ with respect to $\venvs, \cloop$.

Inductive case (4): $\expre$ is a binary routine of the form $\f{\expre_1, \expre_2}$, or $\matmul{\expre_1}{\expre_2}$.
The (R-Biop) rule handles this case.
Without loss of generality, we assume the number of axes of $\expre_1$ is no less than that of $\expre_2$.
The rewritten $\expre_1, \expre_2$, or $\expre'_1, \expre'_2$, evaluate to their parallel evaluations with respect to $\venvs, \cloop$ under $\parvenvs$, which may be of the same shapes as their evaluations in each iteration or have one more dimension on the left-most axes, by the definition of parallel evaluation.
To make $\expre'$ evaluate to the parallel evaluation of $\expre$ with respect to $\venvs, \cloop$ under $\parvenvs$, each element-wise function application in $\expre'$ must align with the corresponding element-wise function application in each iteration.
The evaluation of $\expre$ in each iteration involves broadcasting the evaluation of $\expre_2$ in each iteration to the shape of the evaluation of $\expre_1$ in each iteration for element-wise function application.
Thus, if $\expre_2$'s parallel evaluation has one more dimension on the left-most axis, we need to make some changes to ensure the axes at the positions that are broadcast in the evaluation in each iteration are still broadcast to the same numbers corresponding to the axes of $\expre_1$.
The rewrite rule ensures this by expanding new dimensions of one after the left-most axis on $\expre'_2$ so the $\expre'_2$'s new dimension does not take over an axis that should be broadcast.
If the parallel evaluation of $\expre_2$ does not have one more dimension, the semantics of broadcast continue to ensure the same element-wise function is applied and $\expre'_2$ is properly duplicated on the right axes.
Thus, the rewrite result evaluates to the parallel evaluation of $\expre$ under $\parvenvs$.

Inductive case (5): $\expre$ is either a call to $\expandname$ or a call to $\replicatename$.
These two types of expressions are handled by rules (R-Rep) and (R-Exp).
In each iteration, the expression inserts new axes to the first argument of the call.
If the parallel evaluation of the argument is the same value it evaluates to in each iteration, and there is not an extra dimension added to the left-most axis, then the rewrite rules do not need to change the expression. 
Then, $\expre$ still evaluates to the same value under $\parvenvs$ as it is evaluated to in each iteration. 
This value is the parallel evaluation of $\expre$. 
If the parallel evaluation of the argument is an array of the values that it takes in each iteration, a new dimension is added to the left-most axis of the argument.
In this case, the numbers noting the positions to add axes are all incremented by one.
Then each value in the array that the rewritten expression evaluates to under $\parvenvs$ still corresponds to each value that $\expre$ takes in each iteration.
So in either case the rewritten expression evaluates to the parallel evaluation of $\expre$ under $\venvs$.

Inductive case (6): $\expre$ is a call to $\filledname$.
Because the second argument to $\filledname$ is always a scalar, which is stipulated by the semantics of the call, we just need to require that the rewrite result of the second argument is also a scalar. 
Then the reason why the (R-Fill) rule rewrites $\expre$ to an expression that evaluates to the parallel evaluation of $\expre$ follows the same reason why (R-Uop) is correct.

Inductive case (7): $\expre$ is a an indexing expression of the form $\arrindex{\expre_t}{\expre_m, \ldots, \expre_1}$. 
This case is handled by the rules (R-Indx), (R-Idxr) and (Unmask).
At a high-level, because we assume that $\cloop$ does not have any loop-carried dependence when evaluated under $\venvs$ and thus the evaluation of $\expre$ in each iteration does not read values that are written or updated in other iterations, the values read when the base and the indexers are parallelly evaluated are the parallel evaluation of the indexing expression.
First, if $\expre$ is not on the left-hand side of an update statement, the base of the indexing expression is rewritten.
Then all the indexers are rewritten. 
If the base of the indexing expression's parallel evaluation is an array of values that it takes in each iteration, a new indexer is added to select all the values corresponding to iterations in which $\expre$ is evaluated in this array.
This new indexer is simply loop variable substitute stored in $\loopvarrep$.
Similar to binary routines, the indexing operator broadcasts all the indexers to the same shape for element-wise value selection.
So the same trick used to rewrite binary routines is used here to ensure indexers that are broadcast for evaluation in each iteration are still properly broadcast.
Now the evaluation of the $\expre'$ under $\venvs^*$ selects the elements that $\expre$ selects in each iteration.
Because the semantics of the indexing expression require indexers to be unmasked and the parallel evaluation of each indexer may be masked arrays, we need to handle indexers that are masked arrays.
The (Unmask) rule does this.
We fill the masked indexers with 0, which is always a valid index for any meaningful array.
Then, if there is any masked indexer, we wrap the entire indexing expression by a call to $\makemaskedarrayname$, masking the indexing result by the mask extracted from the indexers.
If there is any iteration in which the indexing expression is not evaluated, the corresponding element in the evaluation of the rewrite result is still masked.
Thus, the rewrite result of (R-Indx) evaluates to the parallel evaluation of $\expre$ with respect to $\venvs, \cloop$ under $\parvenvs$.
We do not rewrite the base if $\expre$ is on the left-hand side to avoid the rewrite strategies described in base cases (2.2) and (2.3) duplicating the base variable, which results in an invalid update statement.
We do not need to worry about whether the base variable is escaping the restriction imposed by branches.
Note that because there is no loop-carried dependence, if $\expre$ is on the left-hand side of an update statement, at least one indexer must depend on the loop variable of $\cloop$ and its parallel evaluation is an array of values that this indexer takes in each iteration.
This parallel evaluation of the indexer can pass the mask to the entire indexing expression, so the rewrite result's evaluation under $\parvenvs$ is still $\expre$'s parallel evaluation. 
\end{proof}

Lemma~\ref{lem:stmt-rewrite} is also proven by structural induction on the expressions.
\begin{proof}
Base case (1): $\stmts$ is a variable binding statement in the form of $\cmdAssign{y}{\expre}$.
There are three sub-cases.

Base case (2.1): $\expre$ is rewritten to $\expre'$, which under $\parvenvs$ evaluates to an array of values that $\expre$ evaluates to in each iteration.
This case is handled by the rule (R-Bind1).
In this case, the evaluation of $\expre'$ under $\parvenvs$ is a parallel evaluation of $\expre$ by Lemma~\ref{lem:expr-rewrite}. 
So the execution of $\stmts'$ stores this evaluation of $\expre'$ in the current evaluation environment, which is a parallel execution of $\stmts$.
Because $\expre'$ has the type $(\shapes', \maskm')$ under the current type environment, we just need to update the current type environment to map $y$ to $(\shapes', \maskm')$ to ensure the variable $y$ has the correct type under the updated type environment.

Base case (2.2): $\expre$ is not changed after the rewrite, and the static analysis says $y$ does not depend on the loop variable of $\cloop$.
This case is handled by the rule (R-Bind2).
Because $y$ does not depend on the current loop variable, by the definition of parallel execution, the value bound to $y$ after executing $\stmts'$ does not need to have one more dimension on the left-most axis.
It just needs to be the parallel evaluation of $\expre$, which is true by Lemma~\ref{lem:expr-rewrite}.
So the execution of $\stmts'$ is the parallel execution of $\stmts$.
Because the variable's defining expression is not changed, we do not need to update the typing environment.
The old typing environment still correctly type all the variables.

Base case (2.3): $\expre$ is not changed after the rewrite, and the static analysis says $y$ depends on the loop variable of $\cloop$.
This case is handled by the rule (R-Bind3).
By the definition of parallel execution, $\expre$'s parallel evaluation must have one more dimension on the left-most axis. 
However, $\expre$ is not changed by the rewrite process.
Hence, we need to rewrite to manually duplicate $\expre$ to have one more dimension.
This is achieved through either calling $\replicatename$ on axis 0 to replicate $\expre$ $n$ times or replication by masking.
The premise of the first solution is that the current loop variable substitute stored in $\loopvarrep$ is unmasked, meaning $\stmts$ is not in a branch being flattened. 
In this case we just need to replicate it on the left-most axis.
Similarly, the premise of the second solution means $\stmts$ is in a branch being flattened. 
In this case, calling $\makemaskedarrayname$ with the expanded mask extracted from the current loop variable substitute will duplicate the value that $\expre'$ evaluates to by broadcasting.
All the elements in the left-most axis, which is created by broadcasting, will either be the same value that $\expre$ evaluates to in iterations that it is evaluated or be a fully masked array.
So the evaluation of this rewritten expression under $\parvenvs$ is the parallel evaluation of $\expre$.
This statement, $\stmts'$, stores this evaluation into the updated evaluation environment, which means the execution of $\stmts'$ is a parallel execution of $\stmts$.
The shape that the rewritten result evaluates to under $\parvenvs$ is just the length of the loop variable substitute prepended to the old shape.
The maskedness is either the old maskedness if we do not use the extracted mask to replicate or $\arraymasked$ otherwise.
Storing the updated shape and maskedness to the environment ensures that the typing environment is updated to type $y$ correctly.

Base case (2): $\stmts$ is an update statement of the form $\codeupdate{\expre_l}{\expre_r}$. 
This case is handled by the rule (R-Updt).
First, the expression on the left-hand side of the statement and the expression on the right-hand side of the statement are rewritten. 
So by Lemma~\ref{lem:expr-rewrite}, the rewritten left-hand side should evaluate to the parallel evaluation of the original left-hand side under $\parvenvs$.
Because the semantics of updates (SEM-Upd) and masked updates (SEM-MskUpd) state that elements corresponding to the evaluation of the left-hand side are selected for update, the elements in the array bound to the variable being updated in each iteration are all selected at once.
The right-hand side is also rewritten, which will evaluate to the parallel evaluation of $\expre_r$ under $\parvenvs$.
This parallel evaluation may or may not have one more dimension on the left-most axis. 
To ensure that the broadcasting semantics can continue to correctly duplicate dimensions so each scalar array cell can be updated to the value it is updated to in each iteration, we need to expand new dimensions on the rewritten right-hand side.
The expanded dimensions are on axes $(1, \ldots, n)$, which corresponds to the axes being broadcast in each iteration.
$n$ is the difference in numbers of axes of $\expre_l$ and $\expre_r$ before the rewrite. 
After the dimension expansion, the broadcast semantics ensure that if there is any broadcasting used to duplicate the values for update, the same duplications are applied to the values so they can be used to update the same locations as in each iteration.
Since the update statement cannot change the type of a variable, the typing environment does not need to be updated, and it still correctly stores all the variables' types.

Inductive case (1): $\stmts$ is a sequence of the form $\stmts_1 \codesemicolon \stmts_2$.
In this case, $\stmts$ is rewritten to the sequence of the rewrite results of $\stmts_1$ and $\stmts_2$.
The semantics of sequence statements (SEM-Seq) say the semantics of the rewrite result matches the definition of parallel execution of sequences.
Also, if after rewriting $\stmts_1$ we get an updated type environment that can correctly type all the variables defined no later than $\stmts_1$, the rewrite process for $\stmts_2$ can be correctly guided by this type environment and we get an updated type environment that can correctly type all the variables defined no later than $\stmts_2$.
This can be concluded under the assumption that each statement level rewrite correctly updates the type environment.

Inductive case (2): $\stmts$ is a branch, and it is in the form of $\cmdITE{\expre_c}{\stmts_t}{\stmts_e}$.
First, depending on whether $\expre_c$ has one more dimension on the left-most axis after the rewrite, branch statements are handled by either (R-Brch1) or (R-Brch2).

Inductive case (2.1):
If $\expre_c$ does not have one more dimension on the left-most axis, which is handled by (R-Brch2), the branch is not flattened, and we just need to rewrite the statements in both branches.
The rewritten statements are still kept in the branches where they used to reside.
When the rewrite result is executed, new scopes are pushed into the evaluation environment before executing a branch and popped after executing the branch.
The rewritten statements in the branches are executed under the evaluation environment with the newly added scope, which corresponds to the parallel execution of the statements in the branches under this evaluation environment. 
This execution meets the second case of the definition for parallel execution of branches.
Because the type environment is not scoped and the updated type environment after rewriting the then-branch is used to rewrite the else-branch, the rewrite process continues to correctly update the types of variables defined in both branches.

Inductive case (2.2):
If $\expre_c$ has one more axis after the rewrite, the branches are flattened. 
The rewritten statements in the then-branch and the else-branch are concatenated into a sequence. 
Note that before rewriting the statements in the branches, the loop variable substitute stored in $\loopvarrep$ is updated to be the previous loop variable substitute masked by the expanded condition expression.
The previous loop variable substitute should have all the elements corresponding to the iterations during which the outer branch is not executed masked. 
By the semantics of the branch statements, the condition expression should be a scalar, and after the rewrite, it should be a one-dimensional array of the values it evaluates to in each iteration.
The non-zero values in the array correspond to the iterations that the then-branch is executed, the zeroes in the array correspond to the iterations that the else-branch is executed, and the $\masked$ values in the array correspond to the iterations that neither branch is executed.
Masking the previous loop variable substitute with this flattened condition will mask out all elements corresponding to the iterations that the then-branch is not executed, as by the semantics of $\makemaskedarrayname$ and binary $\opops$, if an element whose corresponding value in the mask argument is zero or $\masked$, the corresponding element in the returned array is $\masked$.
The statements in the else-branch are handled similarly.
When this sequence of rewritten statements is executed, the stack of the evaluation environment is not changed, and all the variables defined in this sequence stay in the evaluation environment after the sequence is evaluated.
This is the same as first using the current evaluation environment to execute the statement in the then-branch in parallel, then using the updated evaluation environment to execute the else-branch in parallel. 
This behavior of the execution meets the first-case of the definition of the parallel execution of a branch, and all the variables are correctly typed in the updated typing environment after the rewrite for the same reason as in Inductive case (2.1).

Inductive case (3): The current target for the rewrite process is $\cloop$ itself.
This case is trivially true. 
Note that when we rewrite a loop, we set the loop variable substitute stored in $\loopvarrep$ to an $\arangename$ call that generates all the values that the loop variable takes in each iteration. 
Because the scope immediately under the loop is always executed in each iteration of the loop, the loop variable substitute stored in $\loopvarrep$ when rewriting this scope does not have a value masked.
So the loop variable substitute is correctly evaluated to the parallel evaluation of the loop variable under the current scope.
Because we are assuming the statement rewrite process correctly returns an updated type environment, the type environment returned from the rewrite process is correct by assumption. 
\end{proof}

Finally we need to prove that if we relax the assumption from $\cloop$ being a loop without any loop-carried dependence to $\cloop$ being a rewritable loop, the rewrite procedure can still correctly rewrite the loop.
The proof goal for this part can be summarized into an enhanced Lemma~\ref{lem:stmt-rewrite}.

\begin{lemma}\label{lem:stmt-rewrite-rewritable-loop}
    Let $\cloop = \cmdFor{x}{0 \ldots \integrali}{\stmts}$ be the innermost loop in a program $\cprog$ whose type environment is $\senv$. 
    Let $\venvs$ be an evaluation environment such that $\venvs \vdash \cloop \tovenvs \venvs'$ and $\venvs \vdash \integrali \tovalue c$.
    $\cloop$ \textbf{is a rewritable loop} when executed under $\venvs$.
    If $\senv, \senv, \emptyset \vdash \cloop \stmtrewriteto \stmts', \senv'$, then $\venvs \vdash \stmts' \tovenvs \parvenvs$ such that $\forall \textbf{ Variable }y. \venvs' \vdash y \tovalue \valuev \rightarrow \parvenvs \vdash y \tovalue \valuev$.
\end{lemma}

\begin{proof}
By the definition of rewritable loops, $\cloop$ is either free of loop-carried dependence or has one rewriteable reduce statement.

The first case is covered by Lemma~\ref{lem:stmt-rewrite}. 
In this case $\parvenvs$ is the result of the parallel execution of $\stmts$ with respect to $\venvs, \cloop$. In this case, $\forall \textbf{ Variable }y. \venvs' \vdash y \tovalue \valuev \rightarrow \parvenvs \vdash y \tovalue \valuev$ , by Theorem~\ref{thm:allen}.

Now we discuss the case where the loop has one rewritable reduce statement.
By the definition of rewriteable reduce statements, if this rewriteable reduce statement is removed from $\cloop$, the resulting loop, $\cloop'$, is completely free of loop-carried dependence. 
Because no other statement reads from the array cells updated by the rewritable reduce statement, the evaluation environment that is the result of the parallel execution of the body of $\cloop'$ should evaluate all the variables that $\venvs'$ can evaluate to the same value as their evaluation under $\venvs'$, except the variable updated by the rewritable reduce statement.
This follows the proof of the first case.
Now we add back this rewritable reduce statement.
The rule for rewriting such statements is (R-Rstmt).
This rule first rewrites the second argument of the right-hand side, which is the argument different from the left-hand side.
Because by definition the rewritable reduce statement does not read array cells that are updated by other statements in other iterations, 
the rewritten second argument evaluates to its parallel evaluation when all the variables making up the second argument evaluate to their parallel evaluation, by Lemma~\ref{lem:expr-rewrite}.
The (R-Rstmt) rule rewrites the iterative reduce to one single declarative call that does the same reduce operation on the left-most axis.
If the second argument does not have one more dimension on the left-most axis after the rewrite, we need to manually duplicate it by either calling $\replicatename$ or using replication by masking. 
The masked elements in the array to be reduced are filled with the ``identity value'' of the reduce operation, so in case that the entire array being reduced is masked, the reduction result is not a masked value. 
If the entire array being reduced is masked, the reduction is never performed in the iterative execution, and the array cells being updated should retain their initial value.
Applying the operation to the array cells with the ``identity value'' being the second argument keeps the array cells unchanged.
Executing this rewritten statement updates the array cells being updated to the values they should hold at the end of iterative execution.
As no other statement in the loop updates these array cells, the updated values remain unchanged after the parallel execution of all other statements. 
Thus, $\venvs'$ and $\parvenvs$ evaluate the variable updated by this statement to the same value.
$\senv$ remains correct after rewriting the rewritable reduce statement as the rewrite of the rewritable reduce statement does not change the $\senv$. 
\end{proof}

Lemma~\ref{lem:loop-rewrite} can then be proved.
\begin{proof}
    Lemma~\ref{lem:stmt-rewrite-rewritable-loop} implies Lemma~\ref{lem:loop-rewrite}.
\end{proof}

From Lemma~\ref{lem:loop-rewrite}, we can prove the following theorem.

\begin{theorem}[Soundness of rewriting the innermost loop] \label{thm:soundness-rewrite-prog}
Let $\cprog$ be a rewritable program, $\cloop$ be its innermost loop, $\senv$ be its type environment. If $\textsc{Rewrite}(\cloop, \senv, \senv, \emptyset)$ returns a statement $\stmts$, the program after rewrite is $\cprog' = \cprog[\stmts/\cloop]$. For any evaluation environment $\venvs$ such that $\venvs \vdash \cprog \toprogramresult (\valuev, \maskmu)$, it holds that $\venvs \vdash \cprog' \toprogramresult (\valuev', \maskmu') \land \valuev = \valuev' \land \maskmu = \maskmu'$.
\end{theorem}

\begin{proof}
Because each unique variable name can only be defined in one statement in the program and no statement can move variables across scopes in the evaluation environment, neither the original loop nor $\stmts$ can change the scope of variables that are already in the evaluation environment.
From Lemma~\ref{lem:loop-rewrite} we know all the variables that can be evaluated after the original loop was executed can be evaluated to the same values after $\stmts$ is executed. 
We also know that these variables remain in the scopes that they are defined after the execution of the original loop or the execution of $\stmts$.
Lastly, by the semantics of statements in the DSL, the evaluation environment after evaluating the original loop or $\stmts$ must have the same number of scopes as before.
So as far as all the variables that are defined before the loop are concerned, the results of the execution of the original loop and the execution of $\stmts$ are the same. 
The remaining executions in the program will consequently return the same value.
\end{proof}

Finally, we can prove Theorem~\ref{thm:soundness-vectorize-shown}, which is restated below.

\begin{theorem}[Correctness of the \textsc{Vectorize} Routine] \label{thm:soundness-vectorize}
Let $\cprog$ be a rewritable program. 
The types of the input arguments to $\cprog$ are correctly annotated by $\cann$.
Assume that \textsc{Vectorize}($\cprog, \cann$) returns $\cprog'$.
Then for all evaluation environment $\venvs$ such that $\venvs \vdash \cprog \tovalue (\valuev, \maskmu)$, $\venvs \vdash \cprog' \tovalue (\valuev', \maskmu') \land \valuev = \valuev' \land \maskmu = \maskmu'$.
\end{theorem}

\begin{proof} 
The \textsc{AnalyzeShape} routine correctly returns a type environment for all the variables in the program, which follows Theorem~\ref{thm:type-soundness-shown}.
Suppose we call the program at the $n$th iteration of the while-loop $\cprog_n$. 
By Theorem~\ref{thm:soundness-rewrite-prog}, for any evaluation environment $\venvs$ such that $\venvs \vdash \cprog \tovalue \valuev \implies \venvs \vdash \cprog_n \tovalue \valuev$ is clearly a loop invariant, which can be shown by a simple induction.
As the routine directly returns the $\cprog$ after the while-loop, the property obviously holds on the returned $\cprog$. 
\end{proof}